\documentclass{svproc}
\usepackage{url}

\usepackage{graphicx}
\usepackage{amsmath}
\usepackage{amssymb}
\usepackage{booktabs}
\usepackage{subcaption} 
\usepackage{threeparttable}
\usepackage{booktabs}
\usepackage{siunitx}
\usepackage{multirow}
\usepackage{mathtools}
\usepackage{hyperref}

\begin{document}
\mainmatter              
\title{Understanding NTK Variance in Implicit Neural Representations}
\titlerunning{NTK Variance in INRs}  
%
\author{Chengguang Ou \and Yixin Zhuang}
\authorrunning{Chengguang Ou et al.} 
%
\tocauthor{Chengguang Ou and Yixin Zhuang}
\institute{Fuzhou University
\\
\email{231020043@fzu.edu.cn, yixin.zhuang@gmail.com}
}

\maketitle              

\begin{abstract}
Implicit Neural Representations (INRs) often converge slowly and struggle to recover high-frequency details due to spectral bias. While prior work links this behavior to the Neural Tangent Kernel (NTK), how specific architectural choices affect NTK conditioning remains unclear. We show that many INR mechanisms can be understood through their impact on a small set of pairwise similarity factors and scaling terms that jointly determine NTK eigenvalue variance. For standard coordinate MLPs, limited input–feature interactions induce large eigenvalue dispersion and poor conditioning. We derive closed-form variance decompositions for common INR components and show that positional encoding reshapes input similarity, spherical normalization reduces variance via layerwise scaling, and Hadamard modulation introduces additional similarity factors strictly below one, yielding multiplicative variance reduction. This unified view explains how diverse INR architectures mitigate spectral bias by improving NTK conditioning. Experiments across multiple tasks confirm the predicted variance reductions and demonstrate faster, more stable convergence with improved reconstruction quality.
\end{abstract}

\section{Introduction}
\label{sec:intro}

\begin{figure}[t]
\centering
\includegraphics[width=0.99\linewidth]{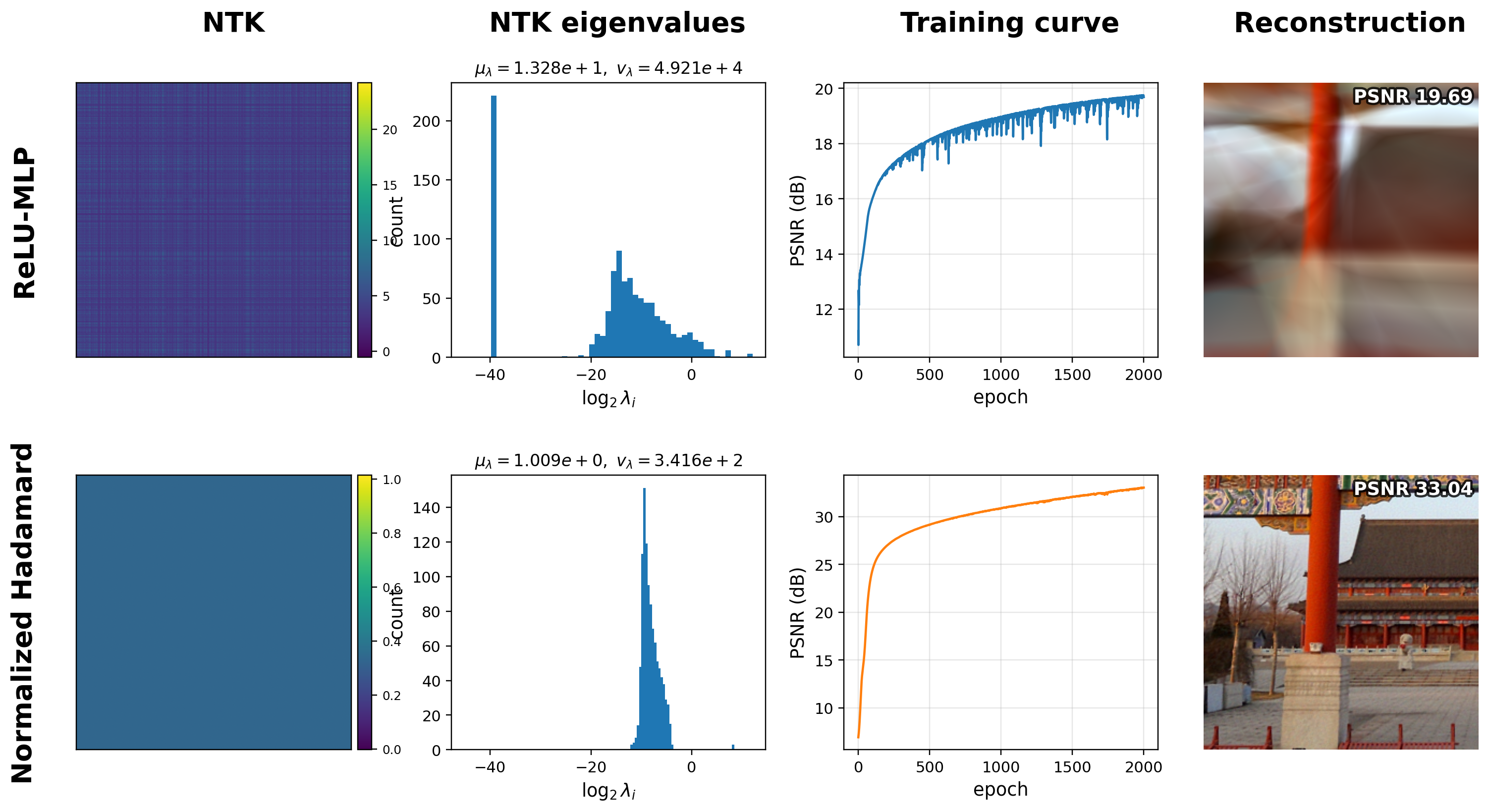}
\caption{Architectural operations reshape NTK conditioning and thus INR convergence. Panels (left to right) show NTK heatmaps, eigenvalue spectra, training curves, and reconstructions. \emph{Top}: A standard ReLU MLP exhibits strong hidden-feature alignment, visible as bright off-diagonal NTK bands, leading to concentrated similarity mass, heavy-tailed eigenvalues, and slow high-frequency convergence. \emph{Bottom}: Normalization and coordinate-dependent Hadamard modulation reduce feature alignment via reshaped similarity and bounded scaling, producing a more uniform NTK, a flatter spectrum, and more stable optimization with improved reconstruction fidelity.}
\label{fig:teaser}
\end{figure}

Implicit Neural Representations (INRs) parameterize continuous signals such as images and shapes using coordinate-based multilayer perceptrons (MLPs)~\cite{ref3,ref4,ref5,ref6}. Their resolution-agnostic formulation has enabled strong results in reconstruction, rendering, and signal compression~\cite{ref7,ref8,ref9,ref10,ref11,ref12}. However, INRs often converge slowly and struggle to recover high-frequency details due to spectral bias, a phenomenon linked to the conditioning of the Neural Tangent Kernel (NTK)~\cite{ref4,ref15,ref16,ref17,ref19}, whose spectrum governs gradient-descent dynamics in over-parameterized networks~\cite{ref20,ref14}.

While empirical studies show that INR architectures differ substantially in NTK conditioning, a clear architectural explanation remains missing. In particular, how design choices such as positional encoding, normalization, or modulation modify the NTK to improve convergence is not well understood. Existing approaches—including SIREN and Fourier feature networks—improve high-frequency learning but are typically motivated by bandwidth or initialization heuristics rather than a unified kernel-level view.

We provide such a view by showing that many INR operations act through a small set of pairwise similarity factors and scaling coefficients that jointly determine NTK eigenvalue variance. In standard coordinate MLPs, similarities between encoded inputs and correlations among hidden features induce large eigenvalue dispersion and poor conditioning.

Our analysis reveals that architectural components modify these variance-controlling factors in complementary ways. Positional encoding reshapes input similarity geometry, normalization introduces contractive layerwise scaling without altering similarity patterns, and coordinate-dependent Hadamard modulation adds new similarity terms strictly bounded in $(0,1)$, yielding multiplicative variance reduction and a flatter NTK spectrum. This perspective explains why diverse INR architectures mitigate spectral bias and suggests a simple design principle: reshaping similarity geometry or injecting bounded modulation systematically improves NTK conditioning.

We validate our theory on CT reconstruction~\cite{ref61} and image super-resolution~\cite{ref63}, showing that normalization and modulation consistently reduce empirical NTK variance, accelerate convergence, and improve reconstruction fidelity.

Our main contributions are summarized as follows:
\begin{itemize}
\item We derive an NTK variance decomposition for INRs in terms of pairwise similarities and scaling factors.
\item We prove that normalization contracts NTK variance and that Hadamard modulation introduces bounded similarity terms yielding multiplicative variance reduction.
\item We empirically demonstrate more stable optimization, improved reconstruction fidelity, and reduced spectral bias across multiple INR tasks.
\end{itemize}

\section{Related Work}

\subsection{Implicit Neural Representations}

Implicit Neural Representations model signals as continuous coordinate-to-attribute functions and are widely used in reconstruction, neural rendering, and 3D modeling~\cite{ref8,ref22,ref23,ref24,ref25}. Despite their expressivity, coordinate-based MLPs suffer from spectral bias, motivating numerous architectural modifications.

One prominent direction enhances input representations. Fourier Feature Networks~\cite{ref23} and deterministic sinusoidal encodings as in NeRF~\cite{ref5} reshape input geometry to make high-frequency structure more accessible and are now standard in many INR designs. Other approaches modify intermediate representations: SIREN~\cite{ref22} employs sinusoidal activations to preserve gradients, multi-scale models such as BACON~\cite{ref24} decompose signals by frequency, and modulation-based architectures (e.g., MFN~\cite{refMFN}, FINN~\cite{ref25}) apply coordinate-dependent feature reweighting. Additional extensions include residual connections, coordinate warping, and hybrid representations.

Although these methods consistently reduce spectral bias and improve convergence, their effectiveness is typically justified by architectural intuition rather than a unified analytical explanation of their optimization behavior.

\subsection{NTK-Based Analyses of INRs}

The Neural Tangent Kernel framework~\cite{ref15,ref16,ref17,ref32,ref34,ref35} provides a principled view of gradient descent in wide networks and has been used to study spectral bias in INRs. Prior work links slow high-frequency learning to rapidly decaying NTK eigenvalue spectra~\cite{ref13,ref36,ref41} and shows that architectural choices—such as ReLU, sinusoidal, or kernelized designs—affect kernel smoothness and eigenvalue dispersion~\cite{ref37,ref38,ref39,ref40}.

Existing analyses, however, typically consider individual components in isolation. Positional encodings reshape pairwise input similarity~\cite{ref23,ref5}, normalization layers stabilize gradients~\cite{ref51,ref52}, and modulation-based models empirically accelerate high-frequency convergence, yet their precise and comparative effects on NTK eigenvalue statistics remain underexplored.

Overall, prior work offers fragmented insights into how specific mechanisms influence spectral bias. Our work unifies these perspectives by showing that positional encoding, normalization, and modulation all act through a common pathway—modifying pairwise similarity and scaling factors that control NTK eigenvalue variance—and thereby directly shape INR convergence dynamics.

\section{Preliminaries}
\label{sec:baseline-ntk}

We introduce the baseline coordinate MLP and its NTK, and show that NTK eigenvalue variance—governing spectral bias—admits a compact decomposition in terms of input similarity, hidden-feature similarity, and an energy scale. This decomposition motivates later architectural design principles.

\subsection{Baseline Coordinate MLP and NTK}

We consider a two-layer overparameterized ReLU network of width $m$ with output scale $a>0$,
\begin{equation}\label{eq:bl-two-layer}
f(\mathbf W;\mathbf x)
=\frac{a}{\sqrt m}\sum_{r=1}^m \sigma(\mathbf w_r^\top\mathbf x),
\end{equation}
where $\mathbf x\in\mathbb R^{d_0}$ and $\mathbf w_r(0)\sim\mathcal N(\mathbf 0,\kappa^2\mathbf I)$.  
Given samples $\{(\mathbf x_i,y_i)\}_{i=1}^n$, we minimize the quadratic loss
$\Phi(\mathbf W)=\tfrac12\|\mathbf y-\mathbf u\|_2^2$. Under standard NTK assumptions~\cite{ref15,ref32,ref35}, early-time dynamics satisfy
\[
\dot{\mathbf u}(t)
=-\mathbf H(\mathbf W(0))(\mathbf u(t)-\mathbf y),
\]
where $\mathbf H$ is the NTK at initialization with entries
$
H_{ij}
=\sum_{r=1}^m
\left\langle
\frac{\partial f(\mathbf w;\mathbf x_i)}{\partial \mathbf w_r},
\frac{\partial f(\mathbf w;\mathbf x_j)}{\partial \mathbf w_r}
\right\rangle.
$
Differentiating~\eqref{eq:bl-two-layer} with respect to the first-layer weights $\mathbf w_r$ for a data point $\mathbf x_i$ yields
$
\frac{\partial f(\mathbf w;\mathbf x_i)}{\partial \mathbf w_r}
=\frac{a}{\sqrt m}\,\mathbb I_{r,i}\,\mathbf x_i,
\quad
\mathbb I_{r,i}=\mathbb I\{\mathbf w_r^\top\mathbf x_i\ge 0\},
$
which implies the NTK entry
$
H_{ij}
=\frac{a^2}{m}\,\rho_{ij}\sum_{r=1}^m \mathbb I_{r,i}\mathbb I_{r,j},
\quad
\rho_{ij}=\mathbf x_i^\top\mathbf x_j.
$ 
Let $\mathbf s_i^{\mathrm{bl}}=(\mathbb I_{r,i})_{r=1}^m$, giving
\begin{equation}\label{eq:bl-NTK}
H_{ij}
=\frac{a^2}{m}\,\rho_{ij}\,
\langle \mathbf s_i^{\mathrm{bl}},\mathbf s_j^{\mathrm{bl}}\rangle .
\end{equation}

\subsection{Spectral Statistics and Bias}

Let $\{\lambda_i,\mathbf v_i\}$ be eigenpairs of $\mathbf H$.  
The error decomposes as~\cite{ref32}
\[
\|\mathbf u(t)-\mathbf y\|_2^2
\approx \sqrt{\sum_{i=1}^{n}(1-\eta\lambda_i)^{2k}\,(\mathbf{v}_i^\top\mathbf{y})^2},
\]
so convergence is controlled by the NTK spectrum. Large eigenvalue dispersion causes spectral bias.

We measure dispersion using
\begin{equation}\label{eq:def-mu-v}
\mu_\lambda=\tfrac1n\operatorname{Tr}(\mathbf H),
\qquad
v_\lambda=\tfrac1n\operatorname{Tr}(\mathbf H^2)-\mu_\lambda^2.
\end{equation}

\begin{proposition}[Baseline spectral statistics]\label{prop:bl-spectrum}
For the NTK in~\eqref{eq:bl-NTK},
\begin{align}
\mu_\lambda
&=\frac{a^2}{nm}\sum_i \rho_{ii}\|\mathbf s_i^{\mathrm{bl}}\|_2^2,
\label{eq:bl-mu-q}
\end{align}
$$
\frac1n\operatorname{Tr}(\mathbf H^2)
=\frac{a^4}{nm^2}\sum_{i,j}
(\rho_{ii}\rho_{jj}\tau^{\mathrm{bl}}_{x,ij})
(\|\mathbf s_i^{\mathrm{bl}}\|_2^4\|\mathbf s_j^{\mathrm{bl}}\|_2^4
(\tau^{\mathrm{bl}}_{s,ij})^2),
$$
where
\[
\tau^{\mathrm{bl}}_{x,ij}
=\frac{\rho_{ij}^2}{\rho_{ii}\rho_{jj}},
\qquad
\tau^{\mathrm{bl}}_{s,ij}
=\frac{\|\mathbf s_i^{\mathrm{bl}}\odot\mathbf s_j^{\mathrm{bl}}\|_2^2}
{\|\mathbf s_i^{\mathrm{bl}}\|_2^2\|\mathbf s_j^{\mathrm{bl}}\|_2^2}.
\]
\end{proposition}

\begin{theorem}[Baseline variance proxy]\label{thm:baseline-MLP-varianve}
Approximating $\rho_{ii}\!\approx\!R_x^2$ and
$\|\mathbf s_i^{\mathrm{bl}}\|_2^2\!\approx\!\overline S_{\mathrm{bl}}$ yields
\begin{equation}\label{eq:bl-v-compact}
v_\lambda
\approx
\frac{a^4 R_x^4 \overline S_{\mathrm{bl}}^{\,2}}{n m^2}
\sum_{i\ne j}\tau^{\mathrm{bl}}_{x,ij}\tau^{\mathrm{bl}}_{s,ij},
\end{equation}
where
$\overline S_{\mathrm{bl}}=\tfrac1n\sum_i\|\mathbf s_i^{\mathrm{bl}}\|_2^2
\approx m/2$ at Gaussian initialization.
\end{theorem}
The proofs of Proposition~\eqref{prop:bl-spectrum} and Theorem~\eqref{thm:baseline-MLP-varianve} are provided in the appendix.

Eq.~\eqref{eq:bl-v-compact} shows that NTK variance—and spectral bias—is governed by
(i) input similarity $\tau_x$,
(ii) hidden-feature similarity $\tau_s$, and
(iii) the energy scale $\overline S$.
Reducing any factor improves conditioning.
This motivates three architectural levers explored later:
reshaping input geometry (positional encoding),
controlling feature energy (normalization),
and introducing bounded modulation pathways that further suppress similarity.

\section{Architectural Effects on NTK Variance}
\label{sec:hada-arch-ntk}

We analyze how common INR components reshape the NTK, focusing on their effect on eigenvalue variance. We first unify positional encoding, spherical normalization, and Hadamard modulation in a multi-layer architecture. To obtain analytic insight, we then extract a two-layer normalized Hadamard model with a closed-form NTK. This enables explicit spectral statistics that isolate how input reshaping, normalization, and multiplicative gating improve conditioning and mitigate spectral bias.

\subsection{Architecture Components}

We extend the baseline two-layer model~\eqref{eq:bl-two-layer} to a practical
multi-layer INR by integrating three common components: positional encoding,
spherical normalization, and Hadamard modulation.  
First, inputs are mapped through random Fourier feature positional encoding~\cite{ref4}:
$
  \tilde{\mathbf{x}}
  = \gamma(\mathbf{x})
  = \sqrt{\tfrac{2}{d}}
    \big[\,\cos(2\pi \mathbf{B}\mathbf{x})\ ;\ \sin(2\pi \mathbf{B}\mathbf{x})\,\big],
$
where $\mathbf B\!\sim\!\mathcal N(0,\varsigma^2)^{(d/2)\times d_0}$.
Each hidden layer applies a linear map, ReLU activation, spherical normalization,
and Hadamard modulation:
$$
\mathbf y_\ell
= \frac{\sigma(\mathbf W_\ell \mathbf y_{\ell-1})}
       {\|\sigma(\mathbf W_\ell \mathbf y_{\ell-1})\|_2}
  \odot \mathbf p_\ell,
\qquad \ell=1,\dots,L,
$$
with $\mathbf y_0=\tilde{\mathbf x}$. A final linear readout produces
$f(\mathbf x)=\mathbf a^\top \mathbf y_L$.

This architecture enforces controlled similarity and stabilized energy across
layers. Since a closed-form NTK for the full depth-$L$ model is intractable, we
extract an analytically tractable two-layer variant that isolates these effects.

\paragraph{Two-layer normalized Hadamard model.}
We define
\begin{equation}\label{eq:hada-two-layer-final}
f(\mathbf W,\mathbf c;\tilde{\mathbf x})
= \frac{1}{\sqrt m}\sum_{r=1}^m
c_r\,\frac{\sigma(\mathbf w_r^\top \tilde{\mathbf x})}{\sqrt{S(\tilde{\mathbf x})}},
\qquad
S(\tilde{\mathbf x})
= \sum_{r=1}^m \sigma(\mathbf w_r^\top \tilde{\mathbf x})^2,
\end{equation}
where $\mathbf w_r$ denote the first-layer weights and $\mathbf c$ are fixed modulation coefficients, with $c_r = a_r p_r$ and $a_r \in {\pm a}$ drawn at initialization.

\paragraph{Top-$K$ spherical normalization (TopK-SP).}
As a practical refinement, we consider an energy-aware variant that normalizes
only the top-$K$ hidden coordinates. Let $\mathbf 1_{K_i}$ mask the $K$ largest
entries of $\mathbf y_i$ in magnitude. We define
\begin{equation}\label{eq:topk-sp-def-main}
\mathbf s_i^{\mathrm{tk}}
=
\frac{\mathbf y_i\odot \mathbf 1_{K_i}}
     {\|\mathbf y_i\odot \mathbf 1_{K_i}\|_2}.
\end{equation}
TopK-SP preserves dominant similarity structure while further reducing the
effective energy term in the NTK variance. Throughout the analysis, standard SP
serves as the theoretical model, with TopK-SP viewed as a practical enhancement.

\subsection{NTK Training Dynamics}

We analyze optimization in the two-layer normalized Hadamard model using the NTK framework from the Preliminaries. The encoded input $\tilde{\mathbf x} = \gamma(\mathbf x)$ and the modulation vector $\mathbf c$ are fixed, and only the first-layer weights $\mathbf W = (\mathbf w_r)_{r=1}^m$ are trained.

Differentiating~\eqref{eq:hada-two-layer-final} with respect to $\mathbf w_r$
gives
\begin{equation}\label{eq:hada-grad-two-layer-final}
\frac{\partial f(\tilde{\mathbf x})}{\partial \mathbf w_r}
=
\frac{c_r}{\sqrt m}\,
\frac{\mathbb I_r\,\beta_r}{\sqrt{S(\tilde{\mathbf x})}}\,
\tilde{\mathbf x},
\end{equation}
where $\mathbb I_r=\mathbb I\{\mathbf w_r^\top\tilde{\mathbf x}\ge0\}$ and
$\beta_r=1-\sigma(\mathbf w_r^\top\tilde{\mathbf x})^2/S(\tilde{\mathbf x})$
captures the normalization correction.

Substituting into the NTK definition yields
\begin{equation}\label{eq:hada-ntk-two-layer-final}
H_{ij}
=
\frac{\tilde{\mathbf x}_i^\top\tilde{\mathbf x}_j}{m}
\sum_{r=1}^m
c_r^2\,
\frac{\mathbb I_{r,i}\beta_{r,i}}{\sqrt{S_i}}\,
\frac{\mathbb I_{r,j}\beta_{r,j}}{\sqrt{S_j}},
\end{equation}
which explicitly reveals the roles of ReLU gating, normalization, and Hadamard
modulation in shaping pairwise similarity.

\paragraph{Assumptions.}
We adopt standard NTK conditions:
\begin{itemize}
\item[\textbf{A1}] \emph{Bounded data:}\label{assump:A-data} $\|\tilde{\mathbf x}_i\|_2\le R_x$,
$|y_i|\le C$.
\item[\textbf{A2}] \label{assump:A-init} \emph{Random initialization:}
$\mathbf w_r(0)\sim\mathcal N(0,\kappa^2\mathbf I_d)$ i.i.d.,
$a_r\sim\mathrm{Unif}\{\pm a\}$, fixed $\mathbf c$.
\item[\textbf{A3}] \label{assump:A-pd} \emph{Non-degenerate NTK:}
$\mathbf H^\infty=\lim_{m\to\infty}\mathbb E[\mathbf H(\mathbf W(0))]$ exists and
$\lambda_{\min}(\mathbf H^\infty)=\lambda_0>0$.
\end{itemize}

We now state the convergence guarantees of our structure in the NTK regime, based on the results in \cite{ref35,ref32,ref67,ref16}. Complete proofs, including quantitative stability bounds for the NTK, are deferred to the appendix.

\begin{theorem}[Exponential decay under gradient flow]\label{thm:ct}
Under \textbf{A1}--\textbf{A3}, if
$\|\mathbf H(\mathbf W(0))-\mathbf H^\infty\|_2\le\tfrac12\lambda_0$, then
\[
\|\mathbf u(t)-\mathbf y\|_2^2
\le
e^{-\lambda_0 t}\,\|\mathbf u(0)-\mathbf y\|_2^2,
\qquad t\ge0.
\]
\end{theorem}

\begin{proof}
Finite-width concentration and Weyl’s inequality give
$\lambda_{\min}(\mathbf H(\mathbf W(0)))\ge\tfrac12\lambda_0$. The error
$\mathbf e(t)=\mathbf u(t)-\mathbf y$ then satisfies
$\dot{\mathbf e}(t)=-\mathbf H(\mathbf W(0))\mathbf e(t)$, yielding exponential
decay by a standard energy estimate.
\end{proof}

\begin{theorem}[Linear-rate convergence of gradient descent]\label{thm:gd}
Under \textbf{A1}--\textbf{A3}, assume
$\|\mathbf H(\mathbf W(0))-\mathbf H^\infty\|_2\le\tfrac12\lambda_0$ and choose
$0<\eta\le\|\mathbf H(\mathbf W(0))\|_2^{-1}$. Writing
$\mathbf H(\mathbf W(0))=\sum_{i=1}^n\lambda_i\mathbf v_i\mathbf v_i^\top$, we have
\[
\|\mathbf u(k)-\mathbf y\|_2
=
\sqrt{\sum_{i=1}^n(1-\eta\lambda_i)^{2k}
(\mathbf v_i^\top(\mathbf u(0)-\mathbf y))^2}
\ +\ \epsilon(k),
\]
where $\sup_{k\le T}\epsilon(k)=O(m^{-1/2})$ for any polynomially bounded $T$.
\end{theorem}

\begin{proof}
Freezing the NTK gives the linear recursion
$\mathbf e(k+1)=(\mathbf I-\eta\mathbf H(\mathbf W(0)))\mathbf e(k)$, whose spectral
solution yields the stated formula. NTK stability bounds control kernel drift,
giving $\epsilon(k)=O(m^{-1/2})$. For small initialization,
$\mathbf u(0)\approx\mathbf0$, recovering the standard NTK spectral form
\cite{ref32}.
\end{proof}

These results show that training is governed by the eigenvalues of the
(initial) NTK. Enlarging $\lambda_{\min}$ accelerates the slowest mode, while
reducing eigenvalue spread mitigates spectral bias. We next analyze the
eigenspectrum of the normalized Hadamard NTK and derive a variance decomposition
that isolates how similarity and normalization control conditioning.

\subsection{NTK Spectral Statistics}
\label{sec:spectral}

Having established that the two-layer normalized Hadamard model remains in the
standard NTK setting and inherits the same convergence guaranties as the
baseline ReLU-MLP, we now turn to the structure of its NTK at initialization.
Our goal is to characterize how the architectural components introduced earlier
(PE, spherical/TopK normalization, and Hadamard modulation) shape the
eigenvalue distribution of the NTK matrix and thereby influence the spectral bias. We now derive a compact variance decomposition tailored to the normalized  Hadamard
architecture. Full derivations appear in the appendix.

\paragraph{NTK variance decomposition.}  
From the NTK expression~\eqref{eq:hada-ntk-two-layer-final}, define
\[
\mathbf{s}_i = \Big(\frac{\mathbb I_{r,i}\beta_{r,i}}{\sqrt{S_i}}\Big)_{r=1}^m,\quad
\mathbf{p}_i=(p_r)_{r=1}^m,\quad
\mathbf{t}_i=\mathbf{s}_i\odot\mathbf{p}_i,
\]
so that $H_{ij}=\frac{a^2}{m}\rho_{ij}\langle \mathbf t_i, \mathbf t_j\rangle$, and
\[
\tau_{x,ij}=\frac{\rho_{ij}^2}{\rho_{ii}\rho_{jj}},\quad
\tau_{s,ij}=\frac{\|\mathbf{s}_i\odot \mathbf{s}_j\|_2^2}{\|\mathbf{s}_i\|_2^2\|\mathbf{s}_j\|_2^2},\quad
\tau_{p,ij}=\frac{\|\mathbf{p}_i\odot \mathbf{p}_j\|_2^2}{\|\mathbf{p}_i\|_2^2\|\mathbf{p}_j\|_2^2},\quad
\tau_{q,ij}=(\cos\angle(\mathbf{s}_i\odot \mathbf{s}_j,\mathbf{p}_i\odot \mathbf{p}_j))^2.
\]
These quantities describe (i) the input similarity $(\tau_{x,ij})$, (ii) the hidden
similarity shaped by normalization $(\tau_{s,ij})$, (iii) the modulation similarity
$(\tau_{p,ij})$, and (iv) the alignment between the two similarity structures
$(\tau_{q,ij})$.Although $\tau_{x,ij},\tau_{s,ij},\tau_{p,ij},\tau_{q,ij}$ arise from different normalizations (squared-cosine terms or normalized energy-overlap terms), we uniformly refer to them as similarities.

\begin{proposition}
\label{prop:ntk-mu-exact}
Under \textbf{A1}--\textbf{A3} and in the large-width limit $m\to\infty$, the NTK mean eigenvalue and second spectral moment satisfy,
\begin{align}
\mu_{\lambda}
&= \frac{1}{n}\sum_{i=1}^n H_{ii}
 = \frac{a^2}{nm}\sum_{i=1}^n \rho_{ii}\,\|\mathbf{s}_i\|_2^2 \,\|\mathbf{p}_i\|_2^2
   \sqrt{\tau_{s,ii}\,\tau_{p,ii}}\;\kappa_{ii},
\label{eq:mu-exact-thm}
\\[0.1em]
\frac{1}{n}\operatorname{Tr}(\mathbf{H}\mathbf{H}^\top)
&= \frac{a^4}{nm^2}\sum_{i=1}^n\sum_{j=1}^n
\bigl(\rho_{ii}\rho_{jj}\tau_{x,ij}\bigr)
\bigl(\|\mathbf{s}_i\|_2^2\|\mathbf{s}_j\|_2^2\tau_{s,ij}\bigr)
\bigl(\|\mathbf{p}_i\|_2^2\|\mathbf{p}_j\|_2^2\tau_{p,ij}\bigr)
\tau_{q,ij}.
\label{eq:q-exact-thm}
\end{align}
\end{proposition}

\begin{proof}
Substitute $H_{ij}=\tfrac{a^2}{m}\rho_{ij}\langle\mathbf t_i,\mathbf t_j\rangle$,
expand the inner product into norms and a cosine, use
$\rho_{ij}^2=\rho_{ii}\rho_{jj}\tau_{x,ij}$, and sum over $(i,j)$.
Setting $i=j$ yields~\eqref{eq:mu-exact-thm}.
\end{proof}

\begin{theorem}[variance proxy of of the normalized Hadamard NTK]
\label{thm:ntk-v-exact}
Assume that the diagonal quantities are approximately constant across the samples:
$\rho_{ii}\approx R_x^2$,
$\|\mathbf s_i\|_2^2\approx\overline S$,
$\|\mathbf p_i\|_2^2\approx\overline P$,
and
$\sqrt{\tau_{s,ii}\tau_{p,ii}}\kappa_{ii}\approx 1$.
Then, by Proposition~\ref{prop:ntk-mu-exact}, it follows that
\begin{equation}\label{eq:v-compact-four}
\begin{aligned}
v_\lambda
\;\approx\;&\;
\underbrace{\frac{a^4\,R_x^4\,\overline S^{\,2}\,\overline P^{\,2}}{n\,m^2}}_{\text{scale}}
\sum_{i\neq j}
\underbrace{\tau_{x,ij}}_{\text{input}} \cdot
\underbrace{\tau_{s,ij}}_{\text{hidden}} \cdot
\underbrace{\tau_{p,ij}}_{\text{modulation}} \cdot
\underbrace{\tau_{q,ij}}_{\text{coupling}}.
\end{aligned}
\end{equation}
\end{theorem}
This four-factor decomposition shows how each component shapes the NTK spectrum. The modulation vector $\mathbf{p}_i$ is $\tanh$-scaled so that $\|\mathbf{p}_i\|_2^2 \approx \overline{P}$, ensuring controlled, non-amplifying contributions to NTK scaling.

\paragraph{Effects of architectural operations.}
We now describe how the three operations---positional encoding, normalization, and Hadamard
modulation---affect individual factors in~\eqref{eq:v-compact-four}.Full derivations appear in the appendix.

\begin{lemma}[Positional encoding reduces input similarity]\label{lem:pe}
Let $\{\mathbf x_i\}_{i=1}^n\subset[0,1]^2$ be a uniform grid, and let
$\tilde{\mathbf x}_i=\gamma(\mathbf x_i)\in\mathbb R^d$ be random Fourier
feature encodings with bandwidth $\varsigma>0$ .Then there exist constants $d_0\in\mathbb N$ and $\varsigma_0>0$(depending only on the grid) such that for all feature dimensions
$d\ge d_0$ and bandwidths $\varsigma\ge\varsigma_0$,
\[
\mathbb E_{\mathbf B}\!\left[\frac{1}{n(n-1)}\sum_{i\neq j}
\tau_{x,ij}\right]
\;\le\;
\frac{1}{n(n-1)}\sum_{i\neq j}\tau_{x,ij}^{\mathrm{bl}},
\]
with the reduction increasing as $\varsigma$ grows.
\end{lemma}

\begin{corollary}[PE decreases NTK variance] \label{cor:pe}
By formula~\ref{eq:v-compact-four},Lemma~\ref{lem:pe},
the smaller $\sum_{i\neq j}\tau_{x,ij}$ yields the smaller $v_\lambda$; therefore,
positional encoding systematically mitigates spectral bias.
\end{corollary}

\begin{corollary}[Effect of SP / TopK normalization on the energy factor]\label{cor:sn}
In the baseline ReLU-MLP without normalization, the variance formula \eqref{eq:bl-v-compact} contains an energy factor $\overline S_{\mathrm{bl}}^{\,2}$, reflecting the typical number and magnitude of active hidden channels per sample. Under spherical normalization (SP), the hidden vector is rescaled to unit $\ell_2$ norm, so $\|\mathbf s_i\|_2^2 \equiv 1$ and hence $\overline S^{\,2} \approx 1 \ll \overline S_{\mathrm{bl}}^{\,2}$ whenever many neurons are active. Thus $v_\lambda$ is reduced compared to the baseline essentially by a factor $\overline S^{\,2}/\overline S_{\mathrm{bl}}^{\,2}$.

TopK-SP further masks to the largest-$K$ coordinates before $\ell_2$-normalization, keeping $\tau_{s,ij}$ of the same order as in SP but yielding a smaller effective energy factor in~\eqref{eq:v-compact-four}. The energy-weighted hidden similarity $\mathcal M_{ij}=\tau_{s,ij}\,\overline S^{\,2}$ therefore contracts in expectation, and under the same data and modulation one has
\[
v_\lambda^{\mathrm{(TopK)}} \;\lesssim\; v_\lambda^{\mathrm{(SP)}} \;\ll\;
v_\lambda^{\mathrm{(bl)}}.
\]
\end{corollary}

\begin{corollary}[Hadamard modulation reduces modulation similarities]
\label{cor:mod}
In the baseline model, $\tau_{p,ij}=\tau_{q,ij}\equiv1$.  
A nontrivial modulation pattern (e.g., sinusoidal or structured masks) typically
yields $\tau_{p,ij}\tau_{q,ij}<1$ for $i\neq j$, thereby reducing
$v_\lambda$ and promoting a more balanced contraction across spectral modes.
\end{corollary}

Although the normalized Hadamard architecture obeys the same NTK training dynamics and convergence guaranties as the baseline, the structure of its NTK spectrum is qualitatively different.  
The four-factor decomposition~\eqref{eq:v-compact-four} makes this explicit:
PE reshapes input similarities, normalization contracts the energy-weighted hidden similarity, and Hadamard modulation introduces two additional similarity factors that are identically $1$ in the baseline but strictly smaller off the diagonal in our model.  Together, these effects substantially reduce NTK eigenvalue variance and hence systematically alleviate spectral bias.

\section{Experiments}
\label{sec:exp}

This section empirically validates our NTK variance analysis and demonstrates the benefits of the normalized Hadamard architecture. We first confirm that NTK similarity and scaling factors match theoretical predictions, then benchmark the model on standard INR tasks, including CT reconstruction and image super-resolution.

\subsection{Validating NTK Variance Reduction}
\label{subsec:ntk-validation}

We evaluate the NTK mean $\mu_\lambda$, variance $v_\lambda$, the scale variance $\tilde v_\lambda$, and the four similarity masses $\sum \tau_x$, $\sum \tau_s$, $\sum \tau_p$, $\sum \tau_q$ at initialization (epoch=0) on 2D images. Optimization quality is measured by PSNR at epoch=1000 using 200 sample points. We denote variants as positional encoding (\texttt{rff\_pe\_enc}), spherical normalization (\texttt{Norm}), Hadamard modulation (\texttt{Hada}), and their combinations.

Table~\ref{tab:ablation_ntk_part} summarizes NTK statistics and reconstruction quality.

\begin{itemize}
\item Positional encoding reduces input similarity. PSNR rises from 22.77\,dB (\texttt{base}) to 33.22\,dB (\texttt{rff\_pe\_enc}) and $v_\lambda$ drops orders of magnitude, with $\sum \tau_x$ nearly two orders lower (Corollary~\ref{cor:pe}).

\item Spherical normalization stabilizes the NTK spectrum. Adding $\ell_2$-normalization reduces $v_\lambda$, $\tilde v_\lambda$, and the hidden energy $\overline S$, mitigating spectral bias (Corollary~\ref{cor:sn}).

\item Hadamard modulation enhances expressivity with low dispersion. \texttt{Hada} combines modulation with normalization, achieving PSNR 43.28\,dB while keeping $v_\lambda$, $\tilde v_\lambda$, $\sum \tau_p$, and $\sum \tau_q$ low (Corollary~\ref{cor:mod}).
\end{itemize}

\begin{table*}[t]
  \centering
  \scriptsize
  \begin{threeparttable}
  \caption{The reconstruction quality and NTK statistics. PSNR at epoch 1000; all NTK statistics at epoch 0.}
  \label{tab:ablation_ntk_part}
  \setlength{\tabcolsep}{4pt}
  \begin{tabular}{lccccccccc}
    \toprule
    Model & PSNR$\uparrow$ & $\mu_{\lambda}$ & $v_{\lambda}$ & $\tilde v_{\lambda}$ & $\overline S$ & $\sum\tau_x$ & $\sum\tau_s$ & $\sum\tau_p$ & $\sum\tau_q$ \\
    \midrule
    \texttt{base}               & 22.77 &  4.53e-2  &  2.92e-1  &  9.41e0   &  1.94e1  &  2.97e4  &  3.72e4  &  4.00e4  &  6.65e3  \\
    \texttt{base\_Norm}         & 23.64 &  2.33e-3  &  7.80e-4  &  5.08e-3  &  4.61e-1  &  2.99e4  &  3.82e4  &  4.00e4  &  6.47e3  \\
    \texttt{rff\_pe\_enc}       & 33.22 &  2.27e-2  &  2.91e-4  &  1.57e2   &  7.75e0  &  4.24e2  &  1.73e4  &  4.00e4  &  6.40e3  \\
    \texttt{rff\_pe\_enc\_Norm} & 38.01 &  2.93e-3  &  6.20e-6  &  7.78e-1  &  4.45e-1  &  4.22e2  &  2.45e4  &  4.00e4  &  7.38e3  \\
    \texttt{Hada-NoNorm}        & 41.52 &  2.96e-4  &  1.72e-8  &  5.71e-3  &  4.98e-1  &  4.26e2  &  8.74e3  &  6.06e2  &  2.20e2  \\
    \texttt{Hada}               & 43.28 &  5.96e-4  &  4.92e-8  &  7.02e-4  &  1.43e-1  &  4.28e2  &  3.12e4  &  6.06e2  &  2.51e2  \\
    \bottomrule
  \end{tabular}
  \end{threeparttable}
\end{table*}

\paragraph{TopK-SP reduces hidden similarity.} Table~\ref{tab:ntk_init_vs_psnr} shows that TopK-SP, applied to different architectures such as \texttt{rff\_pe\_enc} or \texttt{Hada}, consistently lowers $v_\lambda$ and $\sum \tau_s$, resulting in higher PSNR. This behavior aligns with the variance proxy $v_\lambda \propto \overline S^2 \sum_{i\neq j} \tau_x \tau_s \tau_p \tau_q$, since TopK-SP simultaneously reduces both the hidden-similarity mass $\sum\tau_{s}$ and the energy scale $\overline S^2$.

Overall, these results confirm that the proposed architectural components systematically control the NTK spectrum, reducing eigenvalue variance and mitigating spectral bias. This provides a principled foundation for improved optimization and reconstruction performance in downstream INR tasks.

\begin{table}[t]
  \centering
  \footnotesize
  \setlength{\tabcolsep}{4pt}
\caption{NTK statistics at initialization (epoch 0) and final PSNR (epoch 1000) for TopK-SP and standard variants.}
  \label{tab:ntk_init_vs_psnr}
  \begin{tabular}{lcccccc}
    \toprule
    Model & PSNR & $\mu_\lambda$ & $v_\lambda$ & $\tilde v_\lambda$ & $\overline S$ & $\sum\tau_s$ \\
    \midrule
    rff\_pe\_enc\_Norm & 38.01 & 2.93e-3 & 6.24e-6 & 7.94e-1 & 4.47e-1 & 2.46e4 \\
    rff\_pe\_enc\_TopK & 41.27 & 2.93e-3 & 1.81e-6 & 4.35e-2 & 3.82e-1 & 6.48e3 \\
    Hada               & 43.27 & 5.88e-4 & 5.47e-8 & 9.10e-4 & 1.43e-1 & 3.12e4 \\
    Hada(TopK)         & 45.32 & 5.92e-4 & 5.12e-8 & 4.53e-4 & 1.42e-1 & 2.33e4 \\
    \bottomrule
  \end{tabular}
\end{table}

\begin{figure}[t]
  \centering
  \begin{subfigure}[t]{0.32\linewidth}
    \centering
    {\scriptsize \quad}\\[-2pt]
    \includegraphics[width=\linewidth]{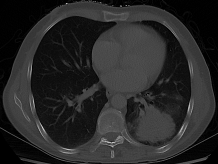}
    \caption{Original}
  \end{subfigure}\hfill
  \begin{subfigure}[t]{0.32\linewidth}
    \centering
    \textbf{32.7 dB} {\scriptsize($\pm$0.177)}\\[-2pt]
    \includegraphics[width=\linewidth]{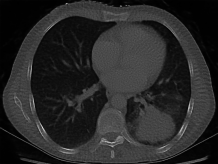}
    \caption{Normalized Hadamard}
  \end{subfigure}\hfill
  \begin{subfigure}[t]{0.32\linewidth}
    \centering
    \textbf{31.5 dB} {\scriptsize($\pm$0.291)}\\[-2pt]
    \includegraphics[width=\linewidth]{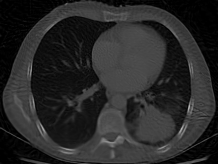}
    \caption{BW-ReLU}
  \end{subfigure}

  \vspace{0.6em}

  \begin{subfigure}[t]{0.32\linewidth}
    \centering
    \textbf{30.47 dB} {\scriptsize($\pm$0.142)}\\[-2pt]
    \includegraphics[width=\linewidth]{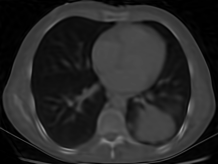}
    \caption{WIRE}
  \end{subfigure}\hfill
  \begin{subfigure}[t]{0.32\linewidth}
    \centering
    \textbf{31.16 dB} {\scriptsize($\pm$0.191)}\\[-2pt]
    \includegraphics[width=\linewidth]{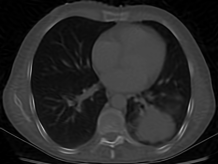}
    \caption{SIREN}
  \end{subfigure}\hfill
  \begin{subfigure}[t]{0.32\linewidth}
    \centering
    \textbf{26.2 dB} {\scriptsize($\pm$0.233)}\\[-2pt]
    \includegraphics[width=\linewidth]{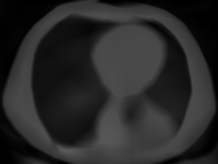}
    \caption{ReLU}
  \end{subfigure}
\vspace{-6pt}
  \caption{CT reconstruction comparisons. Mean $\pm$ standard error over five runs.}

  \label{fig:CT_Reconstruction}
\end{figure}

\subsection{Reconstruction Performance}

We evaluate the proposed architecture on two standard INR tasks, CT reconstruction and $4\times$ image super-resolution \cite{ref30}, comparing against BW-ReLU~\cite{ref30}, SIREN~\cite{ref3}, WIRE~\cite{ref60}, and ReLU networks with positional encoding.

\subsubsection{CT Reconstruction.}
Following \cite{ref61}, we reconstruct $326\times 435$ chest CT slices from 100 equidistant measurements. Fig.~\ref{fig:CT_Reconstruction} shows normalized Hadamard model achieves the highest PSNR, with adaptive spectral modulation recovering high-frequency structures while suppressing streaking artifacts.

\subsubsection{Super-Resolution.}
We perform $4\times$ super-resolution on DIV2K~\cite{ref63}. Fig.~\ref{fig:superres_grid} shows normalized Hadamard model yields superior edge and texture quality. Table~\ref{tab:sr_psnr} reports PSNR across parameter budgets, showing monotonic improvement with increasing width $m$, consistent with reduced NTK variance.

\begin{figure}[t]
  \centering
  \begin{subfigure}[t]{0.31\linewidth}
    \centering
    {\scriptsize \quad}\\[-2pt]
    \includegraphics[width=\linewidth]{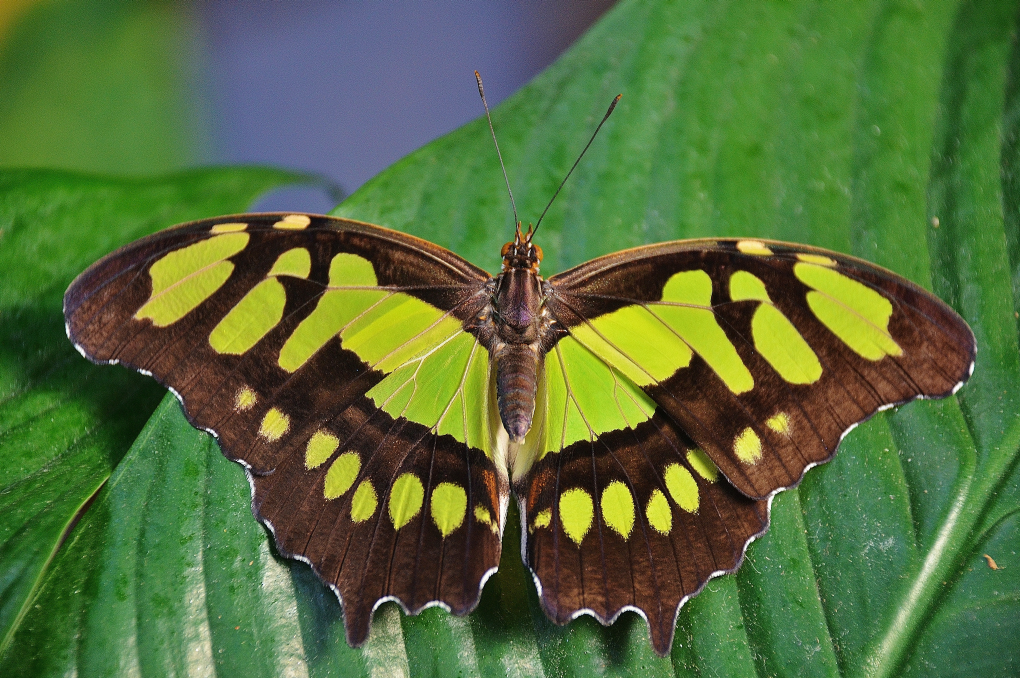}
    \caption{Original}
  \end{subfigure}\hfill
  \begin{subfigure}[t]{0.31\linewidth}
    \centering
    \textbf{27.68 dB} {\scriptsize($\pm$0.018)}\\[-2pt]
    \includegraphics[width=\linewidth]{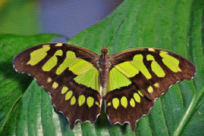}
    \caption{Normalized Hadamard}
  \end{subfigure}\hfill
  \begin{subfigure}[t]{0.31\linewidth}
    \centering
    \textbf{27.10 dB} {\scriptsize($\pm$0.027)}\\[-2pt]
    \includegraphics[width=\linewidth]{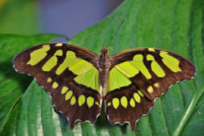}
    \caption{BW-ReLU}
  \end{subfigure}

  \vspace{0.6em}

  \begin{subfigure}[t]{0.31\linewidth}
    \centering
    \textbf{27.05 dB} {\scriptsize($\pm$0.053)}\\[-2pt]
    \includegraphics[width=\linewidth]{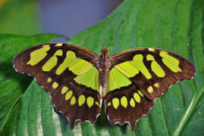}
    \caption{WIRE}
  \end{subfigure}\hfill
  \begin{subfigure}[t]{0.31\linewidth}
    \centering
    \textbf{26.30 dB} {\scriptsize($\pm$0.223)}\\[-2pt]
    \includegraphics[width=\linewidth]{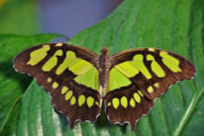}
    \caption{SIREN}
  \end{subfigure}\hfill
  \begin{subfigure}[t]{0.31\linewidth}
    \centering
    \textbf{19.30 dB} {\scriptsize($\pm$0.516)}\\[-2pt]
    \includegraphics[width=\linewidth]{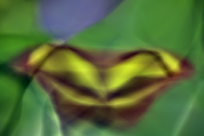}
    \caption{ReLU}
  \end{subfigure}

\vspace{-6pt}
  \caption{Super-resolution ($4\times$) results. Mean $\pm$ standard error over five runs.}
  \label{fig:superres_grid}
\end{figure}

\begin{table}[t]
  \centering
  \footnotesize
    \setlength{\tabcolsep}{4pt}
  \caption{Super-resolution PSNR (dB) for matched parameter budgets.}
  \label{tab:sr_psnr}
  \begin{tabular}{lccc}
    \toprule
    Method & 66k & 133k & 244k \\
    \midrule
    Normalized Hadamard & \textbf{26.75} & \textbf{27.68} & \textbf{27.92} \\
    BW-ReLU       & 26.44 & 27.10 & 27.30 \\
    WIRE          & \textbf{26.75} & 27.05 & 27.36 \\
    SIREN         & 25.40 & 26.30 & 26.67 \\
    ReLU          & 18.92 & 19.33 & 19.74 \\
    \bottomrule
  \end{tabular}
\end{table}

In addition to single-example results, we evaluate robustness across datasets. We sample 16 images from DIV2K and TEXT, resize to $512\times512$, and train for 4000 epochs. Table~\ref{tab:stability_dataset_first} shows that the normalized Hadamard model consistently achieves the highest mean PSNR and SSIM, while ReLU exhibits the lowest performance and highest variability, reflecting weaker spectral expressivity.

\begin{table}[t]
  \centering
  \footnotesize
    \setlength{\tabcolsep}{4pt}
\caption{Comparison of methods for reconstruction across datasets, with mean $\pm$ std of PSNR and SSIM over 16 images from DIV2K and TEXT.}
  \label{tab:stability_dataset_first}
  \begin{tabular}{llcc}
    \toprule
    Dataset & Model & PSNR (dB) & SSIM \\
    \midrule
    \multirow{5}{*}{DIV2K}
      & Normalized Hadamard & $\mathbf{36.12}\pm3.91$ & $\mathbf{0.9531}\pm0.0187$ \\
      & BSpline-W     & $30.17\pm3.82$ & $0.8923\pm0.0420$ \\
      & WIRE          & $32.78\pm3.49$ & $0.9281\pm0.0266$ \\
      & SIREN         & $28.78\pm3.55$ & $0.8964\pm0.0392$ \\
      & ReLU          & $21.68\pm2.82$ & $0.8528\pm0.0512$ \\
    \addlinespace[1pt]
    \midrule
    \multirow{5}{*}{TEXT}
      & Normalized Hadamard & $\mathbf{51.71}\pm3.37$ & $\mathbf{0.9996}\pm0.0002$ \\
      & BSpline-W     & $35.05\pm2.85$ & $0.9835\pm0.0091$ \\
      & WIRE          & $37.73\pm2.63$ & $0.9835\pm0.0087$ \\
      & SIREN         & $32.08\pm3.02$ & $0.9483\pm0.0210$ \\
      & ReLU          & $22.08\pm3.03$ & $0.8931\pm0.0508$ \\
    \bottomrule
  \end{tabular}
\end{table}

\section{Conclusion}
In this work, we analyzed the variance of neural tangent kernel eigenvalues in implicit neural representations and studied how architectural components---including positional encoding, spherical normalization, and multiplicative modulation---shape spectral bias and optimization dynamics. Through this lens, we examined the normalized Hadamard architecture and showed that it effectively reduces NTK variance while preserving expressive capacity. Experiments on CT reconstruction and image super-resolution tasks confirm that the variance-reduction effects predicted by our theory translate into improved reconstruction quality. Overall, our study provides a principled framework for understanding and designing implicit neural networks with controlled NTK spectra, enabling more stable and accurate high-fidelity function representation.


\clearpage
\renewcommand{\thesection}{\Alph{section}}
\setcounter{section}{0}
\section*{Appendix}

This appendix provides complete derivations, extended proofs, and additional analytical details supporting the theoretical results presented in the main paper. Each section focuses on one architectural component or theoretical claim and develops the corresponding NTK dynamics, similarity structures, and spectral statistics in full mathematical detail.

\begin{itemize}
    \item \textbf{Appendix~\ref{sec:appA}} gives full derivations for a baseline two-layer ReLU-MLP, including its NTK expression, eigenvalue mean/variance formulas, and the role of input and hidden-feature similarity.
    \item \textbf{Appendix~\ref{sec:appB}} provides a formal analysis of positional encoding, showing how Fourier or sinusoidal mappings transform input-similarity structure without introducing new similarity factors.
    \item \textbf{Appendix~\ref{sec:appC}} derives the NTK for Hadamard-modulated ReLU networks, including the full decomposition into four similarity and scaling factors, and characterizes how bounded modulation suppresses variance.
    \item \textbf{Appendix~\ref{sec:appD}} analyzes Top-$K$ sparse spherical normalization, proving how normalization introduces contractive, energy-weighted scaling and how this modifies hidden-feature similarity.
    \item \textbf{Appendix~\ref{sec:appE}} presents extended convergence and stability results, including gradient flow dynamics under variance-suppressed NTKs and proofs of improved conditioning.
\end{itemize}

\section{Baseline ReLU-MLP: NTK Dynamics and Spectral Statistics}
\label{sec:appA}

This appendix makes the spectral-statistics analysis of Section~\ref{sec:baseline-ntk} fully explicit for the baseline two-layer ReLU-MLP. We first derive the standard NTK gradient-flow dynamics for quadratic loss in a model-agnostic way\cite{ref16,ref32,ref35}, then specialize to the baseline architecture~\eqref{eq:bl-two-layer} to obtain the exact NTK matrix at initialization. Building on this kernel expression, we compute the mean and variance.

\subsection{Baseline Model, Loss, and Standing Assumptions}

\paragraph{Baseline predictor and loss.}
We consider the two-layer over-parameterized ReLU network of Section~\ref{sec:baseline-ntk}. For an input $\mathbf x\in\mathbb R^{d_0}$, the predictor is
\begin{equation}\label{eq:A-bl-model}
f(\mathbf W;\mathbf x)
=\frac{1}{\sqrt m}\sum_{r=1}^m a\,\phi_r(\mathbf x),
\qquad
\phi_r(\mathbf x)=\sigma(\mathbf w_r^\top\mathbf x),
\tag{A.1}
\end{equation}
where $m$ is the hidden width, $a>0$ is a fixed readout scale, and
$\mathbf w_r(0)\sim\mathcal N(\mathbf 0,\kappa^2\mathbf I_{d_0})$ are initialized i.i.d.\ Gaussian.

Given training data $\{(\mathbf x_i,y_i)\}_{i=1}^n$ with $\|\mathbf x_i\|_2\le R_x$ and $|y_i|\le C$, we define the predictions and labels
\[
u_i \triangleq f(\mathbf W;\mathbf x_i),
\qquad
\mathbf u\triangleq(u_1,\ldots,u_n)^\top,
\qquad
\mathbf y\triangleq(y_1,\ldots,y_n)^\top,
\]
and train using the quadratic loss
\begin{equation}\label{eq:A-loss}
\Phi(\mathbf W)=\frac12\|\mathbf y-\mathbf u\|_2^2.
\tag{A.2}
\end{equation}

\paragraph{Standing assumptions.}
Throughout this appendix we work under the standard NTK conditions of Section~\ref{sec:baseline-ntk}: bounded inputs and outputs, i.i.d.\ Gaussian initialization, and a non-degenerate infinite-width kernel $\mathbf H^\infty$ that is positive definite on the training set. In the over-parameterized regime and for sufficiently small learning rate, the NTK at initialization governs early-time training.

\subsection{NTK Gradient-Flow Dynamics}

We first derive the NTK dynamics for quadratic loss in a model-agnostic way.

\paragraph{Gradient flow in parameter space.}
The continuous-time gradient flow associated with
\eqref{eq:A-loss} is
\begin{equation}\label{eq:A-gf-params}
\frac{\mathrm d\mathbf W(t)}{\mathrm dt}
=-\nabla_{\mathbf W}\Phi(\mathbf W(t)).
\tag{A.3}
\end{equation}
Write $u_i(t)=f(\mathbf W(t);\mathbf x_i)$ and
$\mathbf u(t)=(u_1(t),\ldots,u_n(t))^\top$. For each neuron $r$, define the per-sample gradient
\[
\mathbf g_{r,i}(t)
\triangleq \frac{\partial u_i(t)}{\partial\mathbf w_r(t)}
\in\mathbb R^{d_0}.
\]
By the chain rule and the definition of $\Phi$,
\begin{equation}\label{eq:A-grad-loss}
\frac{\partial\Phi(\mathbf W)}{\partial\mathbf w_r}
=\sum_{i=1}^n \frac{\partial\Phi}{\partial u_i}
\frac{\partial u_i}{\partial\mathbf w_r}
=\sum_{i=1}^n (u_i-y_i)\,\mathbf g_{r,i}.
\tag{A.4}
\end{equation}
Substituting~\eqref{eq:A-grad-loss} into~\eqref{eq:A-gf-params} and
using $\frac{\mathrm d\mathbf w_r}{\mathrm dt}
=-\frac{\partial\Phi}{\partial\mathbf w_r}$ gives
\begin{equation}\label{eq:A-dw-dt}
\frac{\mathrm d\mathbf w_r(t)}{\mathrm dt}
=-\sum_{j=1}^n (u_j(t)-y_j)\,\mathbf g_{r,j}(t).
\tag{A.5}
\end{equation}

\paragraph{NTK as a Gram matrix in parameter space.}
Differentiating $u_i(t)$ with respect to time and applying the chain rule,
\begin{equation}\label{eq:A-du-dt-chain}
\frac{\mathrm d u_i(t)}{\mathrm dt}
=\sum_{r=1}^m
\Big\langle
\frac{\partial u_i(t)}{\partial\mathbf w_r(t)},
\frac{\mathrm d\mathbf w_r(t)}{\mathrm dt}
\Big\rangle
=\sum_{r=1}^m
\big\langle \mathbf g_{r,i}(t),\tfrac{\mathrm d\mathbf w_r(t)}{\mathrm dt}\big\rangle.
\tag{A.6}
\end{equation}
Substituting~\eqref{eq:A-dw-dt} into~\eqref{eq:A-du-dt-chain} yields
\begin{align}
\frac{\mathrm d u_i(t)}{\mathrm dt}
&=-\sum_{r=1}^m\sum_{j=1}^n
(u_j(t)-y_j)
\big\langle\mathbf g_{r,i}(t),\mathbf g_{r,j}(t)\big\rangle
\nonumber\\
&=-\sum_{j=1}^n H_{ij}(t)\,(u_j(t)-y_j),
\label{eq:A-du-dt-NTK}
\tag{A.7}
\end{align}
where the time-dependent neural tangent kernel (NTK) is
\begin{equation}\label{eq:A-NTK-def}
H_{ij}(t)
\triangleq
\sum_{r=1}^m
\big\langle\mathbf g_{r,i}(t),\mathbf g_{r,j}(t)\big\rangle.
\tag{A.8}
\end{equation}
In vector form,
\begin{equation}\label{eq:A-du-dt-vector}
\dot{\mathbf u}(t)
=-\mathbf H(t)\big(\mathbf u(t)-\mathbf y\big).
\tag{A.9}
\end{equation}

Under the NTK assumptions and for large hidden width $m$, the kernel remains close to its initialization along the training trajectory,
$\mathbf H(t)\approx\mathbf H(\mathbf W(0))\triangleq\mathbf H$. Replacing
$\mathbf H(t)$ by $\mathbf H$ in~\eqref{eq:A-du-dt-vector} gives the
linear ODE
\begin{equation}\label{eq:A-NTK-ODE-const}
\dot{\mathbf u}(t)
=-\mathbf H\big(\mathbf u(t)-\mathbf y\big),
\tag{A.10}
\end{equation}
which is the standard NTK regime used in the main text.

Let $\mathbf H=\mathbf V\boldsymbol\Lambda\mathbf V^\top$ be the eigendecomposition with
$\boldsymbol\Lambda=\mathrm{diag}(\lambda_1,\ldots,\lambda_n)$ and $\mathbf V=[\mathbf v_1,\ldots,\mathbf v_n]$ orthonormal. Then
\[
\mathbf u(t)-\mathbf y
=\exp(-\mathbf H t)\big(\mathbf u(0)-\mathbf y\big)
=\sum_{i=1}^n \mathrm e^{-\lambda_i t}\,
\big(\mathbf v_i^\top(\mathbf u(0)-\mathbf y)\big)\mathbf v_i,
\]
and the training error decomposes as
\begin{equation}\label{eq:A-training-error-spectral}
\|\mathbf u(t)-\mathbf y\|_2^2
=\sum_{i=1}^n \mathrm e^{-2\lambda_i t}
\big(\mathbf v_i^\top(\mathbf u(0)-\mathbf y)\big)^2.
\tag{A.11}
\end{equation}
Modes corresponding to larger eigenvalues $\lambda_i$ decay faster. A highly dispersed spectrum (large gaps between the $\lambda_i$) implies that some modes converge much more slowly than others, which is the spectral-bias pattern described in Section~\ref{sec:baseline-ntk}.

\subsection{Baseline NTK: Exact Expression}
\label{sec:A-baseline-NTK}

We now specialize the general definition~\eqref{eq:A-NTK-def} to the
baseline network~\eqref{eq:A-bl-model} and make explicit how input geometry and hidden gate overlap enter the NTK.

For $i,j\in[n]$ we write
\begin{equation}\label{eq:A-rho-def}
\rho_{ij}\triangleq \mathbf x_i^\top\mathbf x_j.
\tag{A.12}
\end{equation}
For each hidden neuron $r$ and sample $i$, the ReLU gate is
\begin{equation}\label{eq:A-gate-def}
\mathbb I_{r,i}
\triangleq \mathbb I\{\mathbf w_r^\top\mathbf x_i\ge 0\}.
\tag{A.13}
\end{equation}
Collecting these gates across channels yields the baseline hidden feature vector
\begin{equation}\label{eq:A-s-bl-def}
\mathbf s_i^{\mathrm{bl}}
\triangleq (\mathbb I_{r,i})_{r=1}^m\in\{0,1\}^m.
\tag{A.14}
\end{equation}
Its squared $\ell_2$ norm counts active channels:
\begin{equation}\label{eq:A-Si-bl}
\|\mathbf s_i^{\mathrm{bl}}\|_2^2
=\sum_{r=1}^m \mathbb I_{r,i}.
\tag{A.15}
\end{equation}
We also define the average hidden energy scale
\begin{equation}\label{eq:A-Sbar-bl}
\overline S_{\mathrm{bl}}
\triangleq \frac{1}{n}\sum_{i=1}^n \|\mathbf s_i^{\mathrm{bl}}\|_2^2.
\tag{A.16}
\end{equation}
Under i.i.d.\ Gaussian initialization, each gate is active with probability $1/2$ at a random input, so typically
$\overline S_{\mathrm{bl}}\approx m/2$.

For sample $i$,
\[
u_i
=f(\mathbf W;\mathbf x_i)
=\frac{1}{\sqrt m}\sum_{r=1}^m a\,\sigma(\mathbf w_r^\top\mathbf x_i).
\]
Differentiating with respect to $\mathbf w_r$ and applying the chain rule gives
\begin{equation}\label{eq:A-grad-raw}
\frac{\partial u_i}{\partial\mathbf w_r}
=\frac{1}{\sqrt m}\,a\,\sigma'(\mathbf w_r^\top\mathbf x_i)\,\mathbf x_i.
\tag{A.17}
\end{equation}
For ReLU, $\sigma(z)=\max\{z,0\}$ and
$\sigma'(z)=\mathbb I\{z\ge 0\}$, so by~\eqref{eq:A-gate-def}
\begin{equation}\label{eq:A-grad-I}
\frac{\partial u_i}{\partial\mathbf w_r}
=\frac{a}{\sqrt m}\,\mathbb I_{r,i}\,\mathbf x_i.
\tag{A.18}
\end{equation}

By definition,
\[
H_{ij}
=\sum_{r=1}^m
\Big\langle
\frac{\partial u_i}{\partial\mathbf w_r},
\frac{\partial u_j}{\partial\mathbf w_r}
\Big\rangle.
\]
Substituting~\eqref{eq:A-grad-I} yields
\begin{align}
H_{ij}
&=\sum_{r=1}^m
\Big\langle
\frac{a}{\sqrt m}\,\mathbb I_{r,i}\,\mathbf x_i,\,
\frac{a}{\sqrt m}\,\mathbb I_{r,j}\,\mathbf x_j
\Big\rangle
\nonumber\\[2pt]
&=\frac{a^2}{m}
\sum_{r=1}^m\mathbb I_{r,i}\mathbb I_{r,j}\,
\langle\mathbf x_i,\mathbf x_j\rangle
\nonumber\\[2pt]
&=\frac{a^2}{m}\,\rho_{ij}\sum_{r=1}^m\mathbb I_{r,i}\mathbb I_{r,j}.
\label{eq:A-H-raw}
\tag{A.19}
\end{align}
The sum $\sum_r\mathbb I_{r,i}\mathbb I_{r,j}$ counts the number of hidden channels that are simultaneously active on samples $i$ and $j$. Using the gate vector $\mathbf s_i^{\mathrm{bl}}$ from \eqref{eq:A-s-bl-def}, we may write
\[
\sum_{r=1}^m \mathbb I_{r,i}\mathbb I_{r,j}
=\sum_{r=1}^m(\mathbb I_{r,i}\mathbb I_{r,j})^2
=\|\mathbf s_i^{\mathrm{bl}}\odot\mathbf s_j^{\mathrm{bl}}\|_2^2,
\]
where $\odot$ denotes the entrywise product. Thus
\begin{equation}\label{eq:A-H-compact-pre}
\begin{aligned}
H_{ij}
&= \frac{a^2}{m}\,\rho_{ij}\sum_{r=1}^m \mathbb I_{r,i}\mathbb I_{r,j}
= \frac{a^2}{m}\,\rho_{ij}\sum_{r=1}^m s_{i,r}^{\mathrm{bl}} s_{j,r}^{\mathrm{bl}}\\
&= \frac{a^2}{m}\,\rho_{ij}\,\langle \mathbf s_i^{\mathrm{bl}},\mathbf s_j^{\mathrm{bl}}\rangle
= \frac{a^2}{m}\,\rho_{ij}\,\|\mathbf s_i^{\mathrm{bl}}\odot\mathbf s_j^{\mathrm{bl}}\|_2^2,
\end{aligned}
\tag{A.20}
\end{equation}

This expression makes explicit how input geometry and gate overlap jointly determine the baseline NTK.

\subsection{Spectral Statistics and Variance Proxy}
\label{sec:A-Spectral-Statistics-and-Variance-Proxy}
We now derive the spectral statistics used in Section~\ref{sec:baseline-ntk}.
Recall that $\mathbf H\in\mathbb R^{n\times n}$ is the (finite-width) NTK Gram
matrix with entries $H_{ij}$ given in~\eqref{eq:A-H-compact-pre}. As a Gram
matrix, $\mathbf H$ is real symmetric and positive semidefinite, and therefore
admits an eigendecomposition
\[
\mathbf H = \mathbf V\boldsymbol{\Lambda}\mathbf V^\top,
\qquad
\boldsymbol{\Lambda}=\operatorname{diag}(\lambda_1,\ldots,\lambda_n),
\]
where $\{\lambda_i\}_{i=1}^n$ are the eigenvalues of $\mathbf H$, counted with
multiplicity, and $\mathbf V=[\mathbf v_1,\ldots,\mathbf v_n]$ is an orthonormal
basis of eigenvectors. In particular,$\operatorname{Tr}(\mathbf H)=\sum_{i=1}^n\lambda_i$ and
$\operatorname{Tr}(\mathbf H^2)=\sum_{i=1}^n\lambda_i^2$.
Following Eq.~\eqref{eq:def-mu-v} in the main text, we define
\begin{equation}\label{eq:A-mu-def}
\mu_\lambda
\triangleq \frac{1}{n}\operatorname{Tr}(\mathbf H)
=\frac{1}{n}\sum_{i=1}^n\lambda_i,
\tag{A.21}
\end{equation}
\begin{equation}\label{eq:A-trH2-def}
\frac{1}{n}\operatorname{Tr}(\mathbf H^2)
=\frac{1}{n}\sum_{i=1}^n\sum_{j=1}^n H_{ij}^2
=\frac{1}{n}\sum_{i=1}^n\lambda_i^2,
\tag{A.22}
\end{equation}
\begin{equation}\label{eq:A-v-def}
v_\lambda
\triangleq \frac{1}{n}\operatorname{Tr}(\mathbf H^2)-\mu_\lambda^2.
\tag{A.23}
\end{equation}

To separate geometry and hidden overlap, we introduce the input squared-cosine similarity
\begin{equation}\label{eq:A-tau-x-def}
\tau^{\mathrm{bl}}_{x,ij}
\triangleq \frac{\rho_{ij}^2}{\rho_{ii}\rho_{jj}}
\in[0,1],
\tag{A.24}
\end{equation}
and the baseline hidden feature similarity
\begin{equation}\label{eq:A-tau-s-bl-def}
\tau^{\mathrm{bl}}_{s,ij}
\triangleq
\frac{\|\mathbf s_i^{\mathrm{bl}}\odot\mathbf s_j^{\mathrm{bl}}\|_2^2}
     {\|\mathbf s_i^{\mathrm{bl}}\|_2^2\,\|\mathbf s_j^{\mathrm{bl}}\|_2^2}
\in[0,1].
\tag{A.25}
\end{equation}
For the binary gate vectors $\mathbf s_i^{\mathrm{bl}}\in\{0,1\}^m$,
$\|\mathbf s_i^{\mathrm{bl}}\odot\mathbf s_j^{\mathrm{bl}}\|_2^2$ counts jointly active channels, while $\|\mathbf s_i^{\mathrm{bl}}\|_2^2$ counts active channels on sample $i$. Combining \eqref{eq:A-s-bl-def} and \eqref{eq:A-tau-s-bl-def},
\begin{equation}\label{eq:A-s-overlap-factor}
\|\mathbf s_i^{\mathrm{bl}}\odot\mathbf s_j^{\mathrm{bl}}\|_2^2
=\tau^{\mathrm{bl}}_{s,ij}\,
\|\mathbf s_i^{\mathrm{bl}}\|_2^2\,\|\mathbf s_j^{\mathrm{bl}}\|_2^2.
\tag{A.26}
\end{equation}
Substituting~\eqref{eq:A-s-overlap-factor} into
\eqref{eq:A-H-compact-pre} gives the factorized NTK
\begin{equation}\label{eq:A-H-compact}
H_{ij}
=\frac{a^2}{m}\,\rho_{ij}\,
\tau^{\mathrm{bl}}_{s,ij}\,
\|\mathbf s_i^{\mathrm{bl}}\|_2^2\,\|\mathbf s_j^{\mathrm{bl}}\|_2^2,
\tag{A.27}
\end{equation}
which separates input similarity, hidden similarity, and hidden energy.

Next, we first prove Proposition ~\ref{prop:bl-spectrum} in the main text.For the baseline NTK~\eqref{eq:A-H-compact},
\begin{equation}\label{eq:A-mu-exact}
\mu_\lambda
=\frac{1}{n}\sum_{i=1}^n H_{ii}
=\frac{a^2}{nm}\sum_{i=1}^n
\rho_{ii}\,\|\mathbf s_i^{\mathrm{bl}}\|_2^2.
\tag{A.28}
\end{equation}

\begin{proof}
Setting $i=j$ in~\eqref{eq:A-H-compact-pre} yields
\[
H_{ii}
=\frac{a^2}{m}\,\rho_{ii}\,
\|\mathbf s_i^{\mathrm{bl}}\odot\mathbf s_i^{\mathrm{bl}}\|_2^2.
\]
Since $\mathbf s_i^{\mathrm{bl}}\in\{0,1\}^m$, one has
$\mathbf s_i^{\mathrm{bl}}\odot\mathbf s_i^{\mathrm{bl}}
=\mathbf s_i^{\mathrm{bl}}$, so
\[
H_{ii}
=\frac{a^2}{m}\,\rho_{ii}\,\|\mathbf s_i^{\mathrm{bl}}\|_2^2.
\]
Averaging over $i$ and using~\eqref{eq:A-mu-def} gives
\eqref{eq:A-mu-exact}.
\end{proof}

We now turn to $\frac{1}{n}\operatorname{Tr}(\mathbf H\mathbf H^\top)$.
Squaring~\eqref{eq:A-H-compact} gives
\begin{equation}\label{eq:A-Hij-square}
H_{ij}^2
=\Big(\frac{a^2}{m}\Big)^2
\rho_{ij}^2\,
\big(\tau^{\mathrm{bl}}_{s,ij}\big)^2\,
\|\mathbf s_i^{\mathrm{bl}}\|_2^4\,\|\mathbf s_j^{\mathrm{bl}}\|_2^4.
\tag{A.29}
\end{equation}
Using the input similarity~\eqref{eq:A-tau-x-def},
\begin{equation}\label{eq:A-rho-factor}
\rho_{ij}^2
=\rho_{ii}\rho_{jj}\,\tau^{\mathrm{bl}}_{x,ij}.
\tag{A.30}
\end{equation}
Combining~\eqref{eq:A-Hij-square} and~\eqref{eq:A-rho-factor}, summing over $(i,j)$ and dividing by $n$ yields the exact identity
\begin{align}
\frac{1}{n}\operatorname{Tr}(\mathbf H\mathbf H^\top)
&=\frac{1}{n}\sum_{i=1}^n\sum_{j=1}^n H_{ij}^2
\nonumber\\
&=\frac{a^4}{nm^2}
\sum_{i=1}^n\sum_{j=1}^n
\big(\rho_{ii}\rho_{jj}\,\tau^{\mathrm{bl}}_{x,ij}\big)\,
\big(\tau^{\mathrm{bl}}_{s,ij}\big)^2\,
\|\mathbf s_i^{\mathrm{bl}}\|_2^4\,\|\mathbf s_j^{\mathrm{bl}}\|_2^4.
\label{eq:A-trH2-exact}
\tag{A.31}
\end{align}

We now impose a mild regularity assumption on the per-sample energy terms in order to factor out their contribution and isolate the dependence on $\tau^{\mathrm{bl}}_{x,ij}$ and $\tau^{\mathrm{bl}}_{s,ij}$. Specifically, we assume that the input energies and gate energies are approximately constant across samples:
\begin{equation}\label{eq:A-approx-energies}
\rho_{ii}\approx R_x^2,
\qquad
\|\mathbf s_i^{\mathrm{bl}}\|_2^2\approx \overline S_{\mathrm{bl}}
\quad\text{for all }i\in[n].
\tag{A.32}
\end{equation}

Under~\eqref{eq:A-approx-energies},
\[
\rho_{ii}\rho_{jj}\approx R_x^4,
\qquad
\|\mathbf s_i^{\mathrm{bl}}\|_2^4\,\|\mathbf s_j^{\mathrm{bl}}\|_2^4
\approx \overline S_{\mathrm{bl}}^{\,4},
\]
and the exact second moment \eqref{eq:A-trH2-exact} simplifies to
\begin{equation}\label{eq:A-trH2-approx}
\frac{1}{n}\operatorname{Tr}(\mathbf H\mathbf H^\top)
\approx
\frac{a^4 R_x^4\,\overline S_{\mathrm{bl}}^{\,4}}{n\,m^2}
\sum_{i=1}^n\sum_{j=1}^n
\tau^{\mathrm{bl}}_{x,ij}\,\big(\tau^{\mathrm{bl}}_{s,ij}\big)^2.
\tag{A.33}
\end{equation}
Similarly, combining~\eqref{eq:A-mu-exact} with~\eqref{eq:A-approx-energies} yields
\begin{equation}\label{eq:A-mu-approx}
\mu_\lambda
\approx
\frac{a^2 R_x^2\,\overline S_{\mathrm{bl}}}{m},
\qquad
\mu_\lambda^2
\approx
\frac{a^4 R_x^4\,\overline S_{\mathrm{bl}}^{\,2}}{m^2}.
\tag{A.34}
\end{equation}
Substituting~\eqref{eq:A-trH2-approx} and~\eqref{eq:A-mu-approx} into the variance definition~\eqref{eq:A-v-def} then gives
\begin{align}
v_\lambda
&=\frac{1}{n}\operatorname{Tr}(\mathbf H\mathbf H^\top)-\mu_\lambda^2
\nonumber\\
&\approx
\frac{a^4 R_x^4\,\overline S_{\mathrm{bl}}^{\,4}}{n\,m^2}
\sum_{i,j}\tau^{\mathrm{bl}}_{x,ij}\,\big(\tau^{\mathrm{bl}}_{s,ij}\big)^2
\;-\;
\frac{a^4 R_x^4\,\overline S_{\mathrm{bl}}^{\,2}}{m^2}.
\label{eq:A-v-approx-raw}
\tag{A.35}
\end{align}

The first term in~\eqref{eq:A-v-approx-raw} aggregates both diagonal $(i=j)$ and off-diagonal $(i\neq j)$ contributions. Using $\tau^{\mathrm{bl}}_{x,ii}=\tau^{\mathrm{bl}}_{s,ii}=1$ and the approximation~\eqref{eq:A-approx-energies}, the purely diagonal part is of order
\[
\frac{a^4 R_x^4\,\overline S_{\mathrm{bl}}^{\,4}}{n\,m^2}\sum_{i=1}^n 1
\;\asymp\;
\frac{a^4 R_x^4\,\overline S_{\mathrm{bl}}^{\,4}}{m^2},
\]
which is the same order as $\mu_\lambda^2$ in~\eqref{eq:A-mu-approx}. Hence the difference between the diagonal contribution and $\mu_\lambda^2$ can be absorbed into an $O(1)$ multiplicative constant in front of the variance.

For the remaining off-diagonal term, we further treat the squared hidden similarity as
\[
\big(\tau^{\mathrm{bl}}_{s,ij}\big)^2
=\tau^{\mathrm{bl}}_{s,ij}\cdot\tau^{\mathrm{bl}}_{s,ij}
\approx \tau^{\mathrm{bl}}_{s,ij}\,\overline{\tau}_s,
\]
where $\overline{\tau}_s\in[0,1]$ denotes a dataset-dependent typical value of $\tau^{\mathrm{bl}}_{s,ij}$, which we again fold into the overall scalar prefactor. This yields the compact variance proxy highlighted in Section~\ref{sec:baseline-ntk}:
\begin{equation}\label{eq:A-v-compact}
v_\lambda
\;\approx\;
\frac{a^4 R_x^4\,\overline S_{\mathrm{bl}}^{\,2}}{n\,m^2}
\sum_{i\ne j}\tau^{\mathrm{bl}}_{x,ij}\,\tau^{\mathrm{bl}}_{s,ij},
\tag{A.36}
\end{equation}
up to dataset-dependent constants of order one. This expression matches the scaling form stated in Proposition~\ref{prop:bl-spectrum} and is used as the baseline reference in the main text.

The expression~\eqref{eq:A-v-compact} should be read as a caling form for the eigenvalue variance of the baseline NTK. Its role is to expose the dependence of $v_\lambda$ on the input similarities $\tau^{\mathrm{bl}}_{x,ij}$ and hidden similarities $\tau^{\mathrm{bl}}_{s,ij}$, rather than to keep track of all dataset-dependent constants. Two approximations are implicit. First, regarding diagonal versus off-diagonal contributions: in the over-parameterized regime ($m\gg n$), the baseline NTK entries satisfy $H_{ii}$ and $H_{ij}$ (for $i\neq j$) having different typical magnitudes once we plug in $\|\mathbf s_i^{\mathrm{bl}}\|_2^2\propto m$ and use that only a fraction of channels are jointly active for $i\neq j$. When forming
$\tfrac1n\mathrm{Tr}(\mathbf H^2)$, both the diagonal terms $H_{ii}^2$ and the off-diagonal terms $H_{ij}^2$ contribute at the same $m$–dependent scale after we factor out the common energy prefactor in~\eqref{eq:A-trH2-exact}. However, there are only $n$ diagonal terms versus $n(n-1)$ off-diagonal ones. It is therefore natural to treat the diagonal part together with $\mu_\lambda^2$ as a dataset-dependent constant and to single out the off-diagonal mass $\sum_{i\neq j}\tau^{\mathrm{bl}}_{x,ij}\,\tau^{\mathrm{bl}}_{s,ij}$ as the dominant structured contribution that controls how dispersed the eigenvalues are. This is precisely the quantity that appears in the compressed form~\eqref{eq:A-v-compact}. Second, in going from~\eqref{eq:A-trH2-exact} to\eqref{eq:A-v-compact} we replace $\big(\tau^{\mathrm{bl}}_{s,ij}\big)^2$ by $\tau^{\mathrm{bl}}_{s,ij}$ times a typical factor $\overline{\tau}_s$, which is then absorbed into the overall constant. This mean-field linearization does not change the qualitative dependence of $v_\lambda$ on $\tau^{\mathrm{bl}}_{s,ij}$: pairs with larger hidden overlap still increase the variance more than pairs with smaller overlap.

The expression~\eqref{eq:A-v-compact} shows that, once the overall scale
\(
\frac{a^4 R_x^4\,\overline S_{\mathrm{bl}}^{\,2}}{n\,m^2}
\)
is fixed by initialization, the dominant contribution to the eigenvalue variance comes from the off-diagonal similarity mass
\[
\sum_{i\ne j}\tau^{\mathrm{bl}}_{x,ij}\,\tau^{\mathrm{bl}}_{s,ij}.
\]
Large input similarities $\tau^{\mathrm{bl}}_{x,ij}$, large hidden similarities $\tau^{\mathrm{bl}}_{s,ij}$, or a large hidden energy scale $\overline S_{\mathrm{bl}}$ all increase $v_\lambda$, leading to a more dispersed NTK spectrum and hence more imbalanced decay rates $\mathrm e^{-\lambda_i t}$ in~\eqref{eq:A-training-error-spectral}. This baseline analysis underpins the design objective in Section~\ref{sec:hada-arch-ntk}, where positional encoding, normalization, and Hadamard modulation are introduced precisely to reduce $v_\lambda$ while keeping $\mu_\lambda$ sufficiently large.

\section{Positional Encoding Shrinks Off-Diagonal Input Similarities}
\label{sec:appB}

In this appendix we provide a detailed proof of Lemma~\ref{lem:pe}.Using the same notation for input squared-cosine similarities as in Section~\ref{sec:baseline-ntk}, we analyze random Fourier feature (RFF) positional encoding on a uniform grid in $[0,1]^2$ and show that, for sufficiently large feature dimension and bandwidth, the average off-diagonal squared-cosine similarity of the encoded inputs is strictly smaller than its raw-coordinate counterpart. We also show that this average is strictly decreasing in the bandwidth $\varsigma$.

\subsection{Random Fourier Features and Second Moment}

We adopt the RFF positional encoding
\begin{equation}\label{eq:B-RFF}
\begin{gathered}
\gamma(\mathbf x)
=\sqrt{\frac{2}{d}}\;\Big[
\cos(2\pi\mathbf b_1^\top \mathbf x),\ \sin(2\pi\mathbf b_1^\top \mathbf x),\ \ldots,\\ 
\cos(2\pi\mathbf b_{d/2}^\top \mathbf x),\ \sin(2\pi\mathbf b_{d/2}^\top \mathbf x)
\Big]^\top,
\end{gathered}
\tag{B.1}
\end{equation}
where the frequencies
$\mathbf b_k\sim \mathcal N(\mathbf 0,\varsigma^2\mathbf I_2)$ are i.i.d.,
$\varsigma>0$ is the bandwidth, and $d$ is the encoded feature dimension.
For a pair of inputs $\mathbf x_i,\mathbf x_j\in[0,1]^2$, we write
\[
\tilde{\mathbf x}_i\triangleq\gamma(\mathbf x_i),
\qquad
\Delta_{ij}\triangleq \mathbf x_i-\mathbf x_j.
\]

To match the main text, we recall the squared-cosine similarities. For raw coordinates,
\begin{equation}\label{eq:B-tau-raw}
\tau_{x,ij}^{\mathrm{bl}}
\triangleq
\frac{\big(\mathbf x_i^\top\mathbf x_j\big)^2}
     {\|\mathbf x_i\|_2^2\,\|\mathbf x_j\|_2^2},
\tag{B.2}
\end{equation}
while for encoded inputs $\tilde{\mathbf x}_i=\gamma(\mathbf x_i)$ we
define
\begin{equation}\label{eq:B-tau-enc}
\tau_{x,ij}
\triangleq
\frac{\big(\tilde{\mathbf x}_i^\top\tilde{\mathbf x}_j\big)^2}
     {\|\tilde{\mathbf x}_i\|_2^2\,\|\tilde{\mathbf x}_j\|_2^2}.
\tag{B.3}
\end{equation}
These are exactly the quantities appearing in Lemma~\ref{lem:pe}.

For a single frequency $\mathbf b\sim\mathcal N(\mathbf 0,\varsigma^2\mathbf I_2)$
and a fixed displacement $\Delta\in\mathbb R^2$, set
\[
Z \;\triangleq\; \mathbf b^\top\Delta.
\]
Since $\mathbf b$ is Gaussian and the map $\mathbf b\mapsto\mathbf b^\top\Delta$ is linear, $Z$ is one-dimensional Gaussian with
\[
Z\sim\mathcal N\big(0,\varsigma^2\|\Delta\|_2^2\big).
\]
We first compute the first trigonometric moment
\[
\kappa(\Delta)
\;\triangleq\;
\mathbb E_{\mathbf b}\big[\cos(2\pi\mathbf b^\top\Delta)\big]
=\mathbb E_Z\big[\cos(2\pi Z)\big].
\]
For a zero-mean Gaussian $Z\sim\mathcal N(0,s^2)$ and any $t\in\mathbb R$, the characteristic function gives
\[
\mathbb E\big[e^{\mathrm i t Z}\big]
=\exp\!\Big(-\tfrac12 t^2 s^2\Big).
\]
Taking real parts and using $\cos(tZ)=\Re(e^{\mathrm i tZ})$ yields
\[
\mathbb E[\cos(tZ)]
=\exp\!\Big(-\tfrac12 t^2 s^2\Big).
\]
Here $t=2\pi$ and $s^2=\varsigma^2\|\Delta\|_2^2$, hence
\begin{equation}\label{eq:B-kappa-def}
\kappa(\Delta)
=\mathbb E_{\mathbf b}\big[\cos(2\pi\mathbf b^\top\Delta)\big]
=\exp\!\big(-2\pi^2\varsigma^2\|\Delta\|_2^2\big).
\tag{B.4}
\end{equation}

For the second trigonometric moment, we use the identity
$\cos^2 x = \tfrac12(1+\cos 2x)$ and obtain
\[
\mathbb E_{\mathbf b}\big[\cos^2(2\pi\mathbf b^\top\Delta)\big]
=\tfrac12\Big(1+\mathbb E_Z[\cos(4\pi Z)]\Big).
\]
Applying the same Gaussian formula with $t=4\pi$ and
$s^2=\varsigma^2\|\Delta\|_2^2$ gives
\[
\mathbb E_Z[\cos(4\pi Z)]
=\exp\!\Big(-\tfrac12 (4\pi)^2\varsigma^2\|\Delta\|_2^2\Big)
=\exp\!\big(-8\pi^2\varsigma^2\|\Delta\|_2^2\big),
\]
so that
\begin{equation}\label{eq:B-cos2-exact}
\mathbb E_{\mathbf b}\big[\cos^2(2\pi\mathbf b^\top\Delta)\big]
=\frac{1}{2}\Big(1 + \exp\!\big(-8\pi^2\varsigma^2\|\Delta\|_2^2\big)\Big).
\tag{B.5}
\end{equation}

The next proposition computes the second moment of the encoded inner product.

\begin{proposition}[Second moment of the RFF inner product]\label{prop:B-second-moment}
Let $\tilde{\mathbf x}_i=\gamma(\mathbf x_i)$ be as in
\eqref{eq:B-RFF}. Then, for any $i,j$,
\begin{equation}\label{eq:B-Ex2-exact}
\mathbb E_{\mathbf B}\!\Big[(\tilde{\mathbf x}_i^\top \tilde{\mathbf x}_j)^2\Big]
= \frac{2}{d}\,\mathbb E_{\mathbf b}\big[\cos^2(2\pi \mathbf b^\top \Delta_{ij})\big]
+ \Big(1-\frac{2}{d}\Big)\,\kappa(\Delta_{ij})^2,
\tag{B.6}
\end{equation}
where $\kappa(\cdot)$ is given by~\eqref{eq:B-kappa-def}. Using
\eqref{eq:B-cos2-exact}, this simplifies to
\begin{equation}\label{eq:B-Ex2-closed}
\begin{aligned}
\mathbb E_{\mathbf B}\!\Big[(\tilde{\mathbf x}_i^\top \tilde{\mathbf x}_j)^2\Big]
&= \frac{1}{d}\Big(1 + e^{-8\pi^2\varsigma^2\|\Delta_{ij}\|_2^2}\Big)
\\ &\quad +\Big(1-\frac{2}{d}\Big)\,e^{-4\pi^2\varsigma^2\|\Delta_{ij}\|_2^2}.
\end{aligned}
\tag{B.7}
\end{equation}
\end{proposition}

\begin{proof}
From \eqref{eq:B-RFF},
\[
\tilde{\mathbf x}_i^\top \tilde{\mathbf x}_j
= \frac{2}{d}\sum_{k=1}^{d/2}\cos\!\big(2\pi \mathbf b_k^\top \Delta_{ij}\big).
\]
Let $c_k=\cos(2\pi \mathbf b_k^\top \Delta_{ij})$. Then
\[
(\tilde{\mathbf x}_i^\top \tilde{\mathbf x}_j)^2
=\Big(\frac{2}{d}\sum_{k=1}^{d/2}c_k\Big)^2
=\frac{4}{d^2}\Big(\sum_{k=1}^{d/2} c_k^2 + 2\!\!\sum_{1\le k<\ell\le d/2}\! c_k c_\ell\Big).
\]
Taking expectation and using i.i.d.\ of $\{c_k\}$,
\[
\mathbb E\big[(\tilde{\mathbf x}_i^\top \tilde{\mathbf x}_j)^2\big]
=\frac{4}{d^2}\Big(\frac{d}{2}\, \mathbb E[c_1^2]
+ \frac{d}{2}\Big(\frac{d}{2}-1\Big)\, \big(\mathbb E[c_1]\big)^2\Big).
\]
By definition,
$\mathbb E[c_1]=\kappa(\Delta_{ij})
=\exp(-2\pi^2\varsigma^2\|\Delta_{ij}\|_2^2)$, and
\[
\begin{aligned}
\mathbb E[c_1^2]
&=\mathbb E\!\Big[\frac{1+\cos\!\big(4\pi \mathbf b^\top \Delta_{ij}\big)}{2}\Big]
\\ &=\frac{1}{2}\Big(1+\exp(-8\pi^2\varsigma^2\|\Delta_{ij}\|_2^2)\Big),
\end{aligned}
\]
which gives \eqref{eq:B-Ex2-exact} and \eqref{eq:B-Ex2-closed} after simplification.
\end{proof}

\subsection{Relating Inner Products to Encoded Similarities}

We now relate the second moment
$\mathbb E_{\mathbf B}[(\tilde{\mathbf x}_i^\top \tilde{\mathbf x}_j)^2]$
to the encoded squared-cosine similarity
$\tau_{x,ij} $.

A key property of the RFF encoding \eqref{eq:B-RFF} is that the squared norm of each encoded vector is deterministic and equal to 1. Indeed, for any input $\mathbf x_i$,

\begin{equation}\label{eq:B-norm-deterministic}
\|\tilde{\mathbf x}_i\|_2^2 
= \frac{2}{d}\sum_{k=1}^{d/2}\Big[\cos^2(2\pi\mathbf b_k^\top\mathbf x_i) + \sin^2(2\pi\mathbf b_k^\top\mathbf x_i)\Big]
= \frac{2}{d}\cdot\frac{d}{2}\cdot 1 = 1,
\tag{B.8}
\end{equation}

where we used the trigonometric identity $\cos^2\theta + \sin^2\theta = 1$. Consequently, the squared-cosine similarity \eqref{eq:B-tau-enc} simplifies exactly to the square of the encoded inner product:

\begin{equation}\label{eq:B-tau-enc-exact}
\tau_{x,ij} = (\tilde{\mathbf x}_i^\top\tilde{\mathbf x}_j)^2.
\tag{B.9}
\end{equation}

Taking expectations, we obtain the exact relation
$$
\mathbb E_{\mathbf B}\big[\tau_{x,ij}\big] = \mathbb E_{\mathbf B}\big[(\tilde{\mathbf x}_i^\top\tilde{\mathbf x}_j)^2\big].
$$

\subsection{Grid Average and Proof of Lemma~\ref{lem:pe}}

Let $\{\mathbf x_i\}_{i=1}^n\subset[0,1]^2$ be a uniform grid (for example, an $m\times m$ grid with $n=m^2$ points, enumerated by $i$). We define the off-diagonal averages
\begin{equation}\label{eq:B-avg-enc-def}
\overline{\tau}_{x} 
\triangleq
\frac{1}{n(n-1)}\sum_{i\ne j}
\tau_{x,ij} ,
\tag{B.10}
\end{equation}
\begin{equation}\label{eq:B-avg-raw-def}
\overline{\tau}_{x}^{\mathrm{bl}}
\triangleq
\frac{1}{n(n-1)}\sum_{i\ne j}
\tau_{x,ij}^{\mathrm{bl}}.
\tag{B.11}
\end{equation}
On any non-degenerate grid in $[0,1]^2$ we have
$\overline{\tau}_{x}^{\mathrm{bl}}>0$.

Averaging \eqref{eq:B-tau-enc-exact} over all $i\ne j$ and using Proposition~\ref{prop:B-second-moment} yields the following.

\begin{lemma}[Averaged off-diagonal expectation]\label{lem:B-avg}
For all $d \ge 2$ (even) and any uniform grid $\{\mathbf x_i\}_{i=1}^n\subset[0,1]^2$,
\begin{equation}\label{eq:B-avg-bound-exact}
\begin{aligned}
\mathbb E_{\mathbf B}\big[\overline{\tau}_{x} \big]
&= \frac{1}{n(n-1)}\sum_{i\ne j}
\mathbb E_{\mathbf B}\big[\tau_{x,ij} \big] \\
&= \frac{1}{n(n-1)}\sum_{i\ne j}
\left[
\frac{1}{d}\Big(1 + e^{-8\pi^2\varsigma^2\|\Delta_{ij}\|_2^2}\Big)
+ \Big(1-\frac{2}{d}\Big)e^{-4\pi^2\varsigma^2\|\Delta_{ij}\|_2^2}
\right].
\end{aligned}
\tag{B.12}
\end{equation}
For each fixed $d$ and grid, the right-hand side is a continuous and strictly decreasing function of the bandwidth $\varsigma>0$.
\end{lemma}

\begin{proof}
By definition of $\overline{\tau}_x$ \eqref{eq:B-avg-enc-def} and the exact relation \eqref{eq:B-tau-enc-exact},
$$
\mathbb E_{\mathbf B}\big[\overline{\tau}_x\big] = \frac{1}{n(n-1)}\sum_{i\ne j}\mathbb E_{\mathbf B}\big[(\tilde{\mathbf x}_i^\top\tilde{\mathbf x}_j)^2\big].
$$
Substituting the closed-form expression \eqref{eq:B-Ex2-closed} from Proposition~\ref{prop:B-second-moment} yields \eqref{eq:B-avg-bound-exact}.
For a fixed grid, $\Delta_{ij}\neq 0$ for all $i\ne j$. The map
$$
\varsigma\ \mapsto\ \frac{1}{d}\Big(1 + e^{-8\pi^2\varsigma^2\|\Delta_{ij}\|_2^2}\Big) + \Big(1-\frac{2}{d}\Big)e^{-4\pi^2\varsigma^2\|\Delta_{ij}\|_2^2}
$$
is continuous and strictly decreasing in $\varsigma>0$ (since exponential terms decay with $\varsigma$). A finite average of such maps preserves continuity and strict monotonicity.
\end{proof}

Since $\|\Delta_{ij}\|_2\ge\delta>0$ for all $i\ne j$ on a fixed grid, the exponential terms in \eqref{eq:B-avg-bound-exact} converge to zero as $\varsigma\to\infty$, and we obtain
\begin{equation}\label{eq:B-avg-limit-exact}
\lim_{\varsigma\to\infty}
\mathbb E_{\mathbf B}\big[\overline{\tau}_{x} \big]
= \frac{1}{d}.
\tag{B.13}
\end{equation}
Thus, for fixed $d$, the average encoded off-diagonal similarity can be made of order $O(1/d)$ by choosing $\varsigma$ sufficiently large.

We are now ready to prove Lemma~\ref{lem:pe}.

\begin{proof}[Proof of Lemma~\ref{lem:pe}]
By definition,
\[
\mathbb E_{\mathbf B}\!\left[\frac{1}{n(n-1)}\sum_{i\neq j}
\tau_{x,ij} \right]
=
\mathbb E_{\mathbf B}\big[\overline{\tau}_{x} \big].
\]
The target inequality in Lemma~\ref{lem:pe} can therefore be written as
\begin{equation}\label{eq:B-lemma-goal}
\mathbb E_{\mathbf B}
\big[\overline{\tau}_{x} \big]
\;\le\;
\overline{\tau}_{x}^{\mathrm{bl}},
\tag{B.14}
\end{equation}
where $\overline{\tau}_{x}^{\mathrm{bl}}$ is given by
\eqref{eq:B-avg-raw-def} and satisfies
$\overline{\tau}_{x}^{\mathrm{bl}}>0$ for any non-degenerate grid.

Fix
\[
\epsilon\in\Big(0,\frac{1}{2}\overline{\tau}_{x}^{\mathrm{bl}}\Big).
\]
By Lemma~\ref{lem:B-avg} and the limit
\eqref{eq:B-avg-limit-exact}, for any fixed $d$ there exists
$\varsigma_0(d)>0$ such that for all $\varsigma\ge\varsigma_0(d)$,
\[
\mathbb E_{\mathbf B}
\big[\overline{\tau}_{x} \big]
\le \frac{1}{d} + \epsilon.
\]
Next, choose $d_0\in\mathbb N$ large enough so that for all $d\ge d_0$
we have
\[
\frac{1}{d}\le \epsilon.
\]
For any $d\ge d_0$, take $\varsigma_0\triangleq\varsigma_0(d)$ as
above; then for all $\varsigma\ge\varsigma_0$,
\[
\mathbb E_{\mathbf B}
\big[\overline{\tau}_{x} \big]
\le \frac{1}{d} + \epsilon
\le 2\epsilon
< \overline{\tau}_{x}^{\mathrm{bl}},
\]
where the last strict inequality follows from the choice of
$\epsilon<\overline{\tau}_{x}^{\mathrm{bl}}/2$.
This establishes \eqref{eq:B-lemma-goal} for all
$d\ge d_0$ and $\varsigma\ge\varsigma_0$, proving
Lemma~\ref{lem:pe}.
Finally, Lemma~\ref{lem:B-avg} shows that
$\mathbb E_{\mathbf B}
[\overline{\tau}_{x} ]$ is strictly decreasing in
$\varsigma$ for fixed $d$, so the reduction effect strengthens as
$\varsigma$ grows.
\end{proof}

\subsection{Interpretation for Spectral Statistics}

In the main text and Appendix~A, the spectral bias of the baseline NTK is quantified by the   variance of eigenvalues $v_\lambda$, which depends on the off-diagonal input similarities $\tau_{x,ij}$. Lemma~\ref{lem:pe} shows that, on a uniform grid, random Fourier feature encoding with sufficiently large feature dimension and bandwidth strictly reduces the average off-diagonal input similarity mass in expectation. Therefore, in any spectral-variance expression where the input factor appears multiplicatively with hidden similarity terms, replacing raw coordinates by $\gamma(\mathbf x)$ directly lowers the contribution of the input similarity factor and thus helps mitigate spectral bias, provided that the hidden similarities remain of comparable order.

\section{Hadamard-ReLU: NTK Dynamics and Spectral Statistics}
\label{sec:appC}

This appendix provides the detailed NTK analysis for the two-layer normalized Hadamard model introduced in Section~\ref{sec:hada-arch-ntk}.  We start from the precise model and training setup, derive the gradient and NTK
entries step by step, and then obtain the exact eigenvalue statistics and the four-factor variance formula used to quantify spectral bias in the main text.

\subsection{Setup, Model, and NTK Representation}

We first fix notation for the training data, define the two-layer normalized
Hadamard model, and compute its NTK matrix at initialization.  

Let $\{(\mathbf x_i,y_i)\}_{i=1}^n$ be the original training pairs and
\[
\tilde{\mathbf x}_i \;=\; \gamma(\mathbf x_i)\in\mathbb R^d,
\qquad i=1,\ldots,n,
\]
be the encoded inputs obtained by the positional encoding
$\gamma(\cdot)$.
We denote the encoded inner products by
\begin{equation}\label{eq:C-rho-def}
\rho_{ij}
\;\triangleq\;
\tilde{\mathbf x}_i^\top\tilde{\mathbf x}_j,
\qquad i,j\in[n].
\tag{C.1}
\end{equation}

Given first-layer weights $\mathbf W=(\mathbf w_r)_{r=1}^m$ and fixed
second-layer coefficients $\mathbf c=(c_r)_{r=1}^m$, the prediction on
$\tilde{\mathbf x}_i$ is written as
\[
u_i
\;\triangleq\;
f(\mathbf W,\mathbf c;\tilde{\mathbf x}_i),
\qquad
\mathbf u
=
(u_1,\ldots,u_n)^\top\in\mathbb R^n.
\]
We train $\mathbf W$ by minimizing the squared loss
\begin{equation}\label{eq:C-loss}
\Phi(\mathbf W)
=
\frac12\|\mathbf y-\mathbf u\|_2^2,
\qquad
\mathbf y=(y_1,\ldots,y_n)^\top.
\tag{C.2}
\end{equation}
Throughout, the encoded inputs $(\tilde{\mathbf x}_i)$ and the modulation
coefficients $\mathbf c$ are treated as fixed, while only $\mathbf W$ is
updated.

\paragraph{Two-layer normalized Hadamard model.}
The two-layer normalized Hadamard model extends the baseline two-layer ReLU network by (i) replacing the raw input by the encoded input $\tilde{\mathbf x}$, and (ii) applying a global spherical normalization to the hidden activations before the Hadamard-modulated readout.  For a generic encoded input $\tilde{\mathbf x}\in\mathbb R^d$, the predictor is
\begin{equation}\label{eq:C-hada-model}
\begin{gathered}
f(\mathbf W,\mathbf c;\tilde{\mathbf x})
=
\frac{1}{\sqrt m}\sum_{r=1}^m c_r\,
\frac{\sigma(\mathbf w_r^\top\tilde{\mathbf x})}{\sqrt{S(\tilde{\mathbf x})}},
\\
S(\tilde{\mathbf x})
=
\sum_{r=1}^m \sigma(\mathbf w_r^\top\tilde{\mathbf x})^2,
\end{gathered}
\tag{C.3}
\end{equation}
where:
\begin{itemize}
  \item $\sigma(\cdot)$ is the ReLU activation,
  \item $S(\tilde{\mathbf x})$ is the pre-normalization hidden energy,
  \item $\mathbf c=(c_r)_{r=1}^m$ encodes the Hadamard modulation, with
  $c_r=a_r p_r$ where $a_r\in\{\pm a\}$ is a fixed readout scale and
  $p_r$ is the $r$-th modulation coefficient of $\mathbf p$.
\end{itemize}
For the $i$-th training example, we abbreviate
\begin{equation}\label{eq:C-Si-def}
S_i
\;\triangleq\;
S(\tilde{\mathbf x}_i)
=
\sum_{r=1}^m \sigma(\mathbf w_r^\top\tilde{\mathbf x}_i)^2.
\tag{C.4}
\end{equation}
The definition~\eqref{eq:C-hada-model} is exactly the two-layer normalized
Hadamard model in~\eqref{eq:hada-two-layer-final} of the main text.

To make the gradient computation explicit, we introduce the pre-activation
\[
g_{r,i}
\;\triangleq\;
\mathbf w_r^\top\tilde{\mathbf x}_i,
\qquad r\in[m],\ i\in[n],
\]
and the corresponding ReLU gate
\[
\mathbb I_{r,i}
\;\triangleq\;
\mathbb I\{g_{r,i}\ge 0\},
\]
so that $\sigma(g_{r,i})=\mathbb I_{r,i}\,g_{r,i}$ and
$\partial\sigma(g_{r,i})/\partial g_{r,i}=\mathbb I_{r,i}$.  
In addition, we define the normalization-induced correction factor
\begin{equation}\label{eq:C-beta-def}
\beta_{r,i}
\;\triangleq\;
1-\frac{\sigma(g_{r,i})^2}{S_i},
\tag{C.5}
\end{equation}
which will appear naturally when differentiating through the spherical
normalization.

\paragraph{Gradient flow in parameter space.}
Fix a neuron index $r\in[m]$ and a sample index $i\in[n]$.  
By definition~\eqref{eq:C-hada-model} and~\eqref{eq:C-Si-def}, the prediction
on $\tilde{\mathbf x}_i$ is
\begin{equation}\label{eq:C-ui-def}
u_i
=
f(\mathbf W,\mathbf c;\tilde{\mathbf x}_i)
=
\frac{1}{\sqrt m}\sum_{\ell=1}^m
c_\ell\,
\frac{\sigma(g_{\ell,i})}{\sqrt{S_i}},
\tag{C.6}
\end{equation}
where $S_i$ depends on all pre-activations $(g_{q,i})_{q=1}^m$.  
Only the term with $\ell=r$ depends on $\mathbf w_r$, hence
\begin{equation}\label{eq:C-grad-step1}
\frac{\partial u_i}{\partial\mathbf w_r}
=
\frac{c_r}{\sqrt m}\,
\frac{\partial}{\partial\mathbf w_r}
\bigg(
\frac{\sigma(g_{r,i})}{\sqrt{S_i}}
\bigg).
\tag{C.7}
\end{equation}

To differentiate the scalar quantity
\[
\varphi_i(\mathbf w_r)
\;\triangleq\;
\frac{\sigma(g_{r,i})}{\sqrt{S_i}},
\]
we view $\varphi_i$ as a function of the two scalar arguments
$(g_{r,i},S_i)$:
\[
\varphi_i(g_{r,i},S_i)
=
\sigma(g_{r,i})\,S_i^{-1/2}.
\]
By the multivariate chain rule,
\begin{equation}\label{eq:C-dphi-dwr}
\frac{\partial\varphi_i}{\partial\mathbf w_r}
=
\frac{\partial\varphi_i}{\partial g_{r,i}}\,
\frac{\partial g_{r,i}}{\partial\mathbf w_r}
+
\frac{\partial\varphi_i}{\partial S_i}\,
\frac{\partial S_i}{\partial\mathbf w_r}.
\tag{C.8}
\end{equation}
We now compute the four factors in~\eqref{eq:C-dphi-dwr} one by one.

\medskip
\noindent\emph{(1) Derivative with respect to $g_{r,i}$.}
Treating $S_i$ as a constant,
\[
\frac{\partial\varphi_i}{\partial g_{r,i}}
=
\sigma'(g_{r,i})\,S_i^{-1/2}
=
\mathbb I_{r,i}\,S_i^{-1/2}.
\]

\medskip
\noindent\emph{(2) Derivative of $g_{r,i}$ with respect to $\mathbf w_r$.}
By definition $g_{r,i}=\mathbf w_r^\top\tilde{\mathbf x}_i$, so
\[
\frac{\partial g_{r,i}}{\partial\mathbf w_r}
=
\tilde{\mathbf x}_i.
\]

\medskip
\noindent\emph{(3) Derivative with respect to $S_i$.}
For fixed $g_{r,i}$,
\[
\frac{\partial\varphi_i}{\partial S_i}
=
\sigma(g_{r,i})\,\frac{\partial}{\partial S_i}(S_i^{-1/2})
=
-\frac{1}{2}\sigma(g_{r,i})\,S_i^{-3/2}.
\]

\medskip
\noindent\emph{(4) Derivative of $S_i$ with respect to $\mathbf w_r$.}
Recall that
\[
S_i
=
\sum_{q=1}^m \sigma(g_{q,i})^2.
\]
Only the term with $q=r$ depends on $\mathbf w_r$, so
\begin{align*}
\frac{\partial S_i}{\partial\mathbf w_r}
&=
\frac{\partial}{\partial\mathbf w_r}\big(\sigma(g_{r,i})^2\big)\\[2pt]
&=
2\,\sigma(g_{r,i})\,\sigma'(g_{r,i})\,
\frac{\partial g_{r,i}}{\partial\mathbf w_r}\\[2pt]
&=
2\,\sigma(g_{r,i})\,\mathbb I_{r,i}\,\tilde{\mathbf x}_i.
\end{align*}

\medskip
\noindent\emph{(5) Combine all pieces.}
Substituting the four derivatives into~\eqref{eq:C-dphi-dwr} gives
\begin{align*}
\frac{\partial\varphi_i}{\partial\mathbf w_r}
&=
\big(\mathbb I_{r,i}\,S_i^{-1/2}\big)\,\tilde{\mathbf x}_i
 +
\big(-\tfrac{1}{2}\sigma(g_{r,i})\,S_i^{-3/2}\big)\,
\big(2\,\sigma(g_{r,i})\,\mathbb I_{r,i}\,\tilde{\mathbf x}_i\big)\\[2pt]
&=
\mathbb I_{r,i}\,S_i^{-1/2}\,\tilde{\mathbf x}_i
-
\mathbb I_{r,i}\,\sigma(g_{r,i})^2\,S_i^{-3/2}\,\tilde{\mathbf x}_i\\[2pt]
&=
\mathbb I_{r,i}\Big(S_i^{-1/2}-\sigma(g_{r,i})^2 S_i^{-3/2}\Big)\tilde{\mathbf x}_i.
\end{align*}
Factor out $S_i^{-1/2}$ and use the definition~\eqref{eq:C-beta-def}:
\[
S_i^{-1/2}-\sigma(g_{r,i})^2 S_i^{-3/2}
=
\frac{1}{\sqrt{S_i}}
\Big(1-\frac{\sigma(g_{r,i})^2}{S_i}\Big)
=
\frac{\beta_{r,i}}{\sqrt{S_i}}.
\]
Therefore
\begin{equation}\label{eq:C-dphi-final}
\frac{\partial\varphi_i}{\partial\mathbf w_r}
=
\frac{\mathbb I_{r,i}\,\beta_{r,i}}{\sqrt{S_i}}\,
\tilde{\mathbf x}_i.
\tag{C.9}
\end{equation}
Substituting~\eqref{eq:C-dphi-final} into~\eqref{eq:C-grad-step1}, we obtain
the per-neuron gradient
\begin{equation}\label{eq:C-grad-final}
\frac{\partial u_i}{\partial\mathbf w_r}
=
\frac{c_r}{\sqrt m}\,
\frac{\mathbb I_{r,i}\,\beta_{r,i}}{\sqrt{S_i}}\,
\tilde{\mathbf x}_i.
\tag{C.10}
\end{equation}
This matches the gradient formula stated in~\eqref{eq:hada-grad-two-layer-final}.

\paragraph{NTK as a Gram matrix in parameter space.}
By definition, the NTK matrix $\mathbf H\in\mathbb R^{n\times n}$ at
initialization is the Gram matrix of the gradients
$\{\partial u_i/\partial\mathbf w_r\}$:
\begin{equation}\label{eq:C-Hij-def}
H_{ij}
\;\triangleq\;
\sum_{r=1}^m
\Big\langle
\frac{\partial u_i}{\partial\mathbf w_r},
\frac{\partial u_j}{\partial\mathbf w_r}
\Big\rangle.
\tag{C.11}
\end{equation}
Substituting~\eqref{eq:C-grad-final},
\begin{align}
H_{ij}
&=
\sum_{r=1}^m
\Big\langle
\frac{c_r}{\sqrt m}\,
\frac{\mathbb I_{r,i}\,\beta_{r,i}}{\sqrt{S_i}}\,
\tilde{\mathbf x}_i,\,
\frac{c_r}{\sqrt m}\,
\frac{\mathbb I_{r,j}\,\beta_{r,j}}{\sqrt{S_j}}\,
\tilde{\mathbf x}_j
\Big\rangle
\nonumber\\[2pt]
&=
\frac{1}{m}\sum_{r=1}^m
c_r^2\,
\frac{\mathbb I_{r,i}\,\beta_{r,i}}{\sqrt{S_i}}\,
\frac{\mathbb I_{r,j}\,\beta_{r,j}}{\sqrt{S_j}}\,
\langle\tilde{\mathbf x}_i,\tilde{\mathbf x}_j\rangle
\nonumber\\[2pt]
&=
\frac{\rho_{ij}}{m}\sum_{r=1}^m
c_r^2\,
\frac{\mathbb I_{r,i}\,\beta_{r,i}}{\sqrt{S_i}}\,
\frac{\mathbb I_{r,j}\,\beta_{r,j}}{\sqrt{S_j}}.
\label{eq:C-Hij-precompact}
\tag{C.12}
\end{align}

In the normalized Hadamard model, we write $c_r=a_r p_r$ with
$a_r\in\{\pm a\}$ and modulation pattern $(P_r)_{r=1}^m$.  
Then $c_r^2=a^2 p_r^2$ and the $r$-th term in the sum can be grouped as a
product of a “hidden factor” and a “modulation factor”.  
For later use, we introduce the hidden vectors, modulation vectors, and their
Hadamard product
\begin{equation}\label{eq:C-s-P-q-def}
\begin{gathered}
\mathbf s_i
\;\triangleq\;
\Big(\frac{\mathbb I_{r,i}\,\beta_{r,i}}{\sqrt{S_i}}\Big)_{r=1}^m\in\mathbb R^m,
\\
\mathbf p_i
\;\triangleq\;
(p_r)_{r=1}^m\in\mathbb R^m,
\\
\mathbf t_i
\;\triangleq\;
\mathbf s_i\odot\mathbf p_i,
\end{gathered}
\tag{C.13}
\end{equation}
so that
\[
\sum_{r=1}^m
c_r^2\,
\frac{\mathbb I_{r,i}\,\beta_{r,i}}{\sqrt{S_i}}\,
\frac{\mathbb I_{r,j}\,\beta_{r,j}}{\sqrt{S_j}}
=
a^2\sum_{r=1}^m s_{i,r}P_r\,s_{j,r}P_r
=
a^2\langle\mathbf t_i,\mathbf t_j\rangle.
\]
Substituting into~\eqref{eq:C-Hij-precompact} yields the compact NTK representation
\begin{equation}\label{eq:C-NTK-compact}
H_{ij}
=
\frac{a^2}{m}\,\rho_{ij}\,\langle\mathbf t_i,\mathbf t_j\rangle,
\end{equation}
which is exactly the form used in Section~\ref{sec:spectral}.  
In the remainder of this appendix, we analyze the eigenvalue statistics of
$\mathbf H$ in~\eqref{eq:C-NTK-compact} and show how its variance
$v_\lambda$ decomposes into four interpretable similarity factors that
reflect the roles of input reshaping, normalization, and Hadamard modulation.

\subsection{Exact Eigenvalue Statistics in Similarity Form}
\label{sec:C-Exact-Eigenvalue-Statistics}
We now express the eigenvalue statistics of $\mathbf H$ in terms of
normalized similarity quantities.  
Let $\lambda_1,\ldots,\lambda_n$ be the eigenvalues of $\mathbf H$.
Following~\eqref{eq:def-mu-v}, we define
\begin{equation}\label{eq:C-mu-v-def}
\begin{aligned}
\mu_\lambda
\;\triangleq\;
\frac{1}{n}\operatorname{Tr}(\mathbf H)
=
\frac{1}{n}\sum_{i=1}^n\lambda_i,
\\
\frac{1}{n}\operatorname{Tr}(\mathbf H^2)
=
\frac{1}{n}\sum_{i=1}^n\lambda_i^2,
\\
v_\lambda
\;\triangleq\;
\frac{1}{n}\operatorname{Tr}(\mathbf H^2)-\mu_\lambda^2.
\end{aligned}
\end{equation}
The variance $v_\lambda$ quantifies the spread of the NTK spectrum and
thus the degree of spectral bias, since the training error decomposes
into modes that decay at rates governed by $\{\lambda_i\}$.

\paragraph{Normalized similarity factors.}
We introduce four families of normalized similarities.  
For inputs, define
\begin{equation}\label{eq:C-taux-def}
\tau_{x,ij}
\;\triangleq\;
\frac{\rho_{ij}^2}{\rho_{ii}\rho_{jj}}\in[0,1],
\tag{C.14}
\end{equation}
which measures the squared cosine between $\tilde{\mathbf x}_i$ and
$\tilde{\mathbf x}_j$.

For the normalized hidden vectors $\mathbf s_i$ and modulation vectors
$\mathbf p_i$, define
\begin{align}
\tau_{s,ij}
&\triangleq
\frac{\|\mathbf s_i\odot\mathbf s_j\|_2^2}
     {\|\mathbf s_i\|_2^2\,\|\mathbf s_j\|_2^2}\in[0,1],
\label{eq:C-taus-def}\\[3pt]
\tau_{p,ij}
&\triangleq
\frac{\|\mathbf p_i\odot\mathbf p_j\|_2^2}
     {\|\mathbf p_i\|_2^2\,\|\mathbf p_j\|_2^2}\in[0,1].
\label{eq:C-tauP-def}
\end{align}
Finally, to capture the alignment between the hidden and modulation similarities, define
\begin{equation}\label{eq:C-tauQ-def}
\begin{aligned}
\kappa_{ij}
\;\triangleq\;
\cos\angle\!\big(\mathbf s_i\odot\mathbf s_j,\ \mathbf p_i\odot\mathbf p_j\big)
\in[-1,1],
\\
\tau_{q,ij}
\;\triangleq\;
\kappa_{ij}^2\in[0,1].
\end{aligned}
\tag{C.15}
\end{equation}
Under Assumptions~\ref{assump:A-data}--\ref{assump:A-pd}, these quantities are almost surely well defined at initialization.

\paragraph{Mean eigenvalue.}
By definition of the trace,
\begin{equation}\label{eq:C-mu-def}
\mu_\lambda
=
\frac{1}{n}\operatorname{Tr}(\mathbf H)
=
\frac{1}{n}\sum_{i=1}^n H_{ii}.
\tag{C.16}
\end{equation}
Using~\eqref{eq:C-NTK-compact} with $i=j$,
\[
H_{ii}
=
\frac{a^2}{m}\,\rho_{ii}\,\langle\mathbf t_i,\mathbf t_i\rangle
=
\frac{a^2}{m}\,\rho_{ii}\,\|\mathbf t_i\|_2^2.
\]
Thus
\begin{equation}\label{eq:C-mu-basic}
\mu_\lambda
=
\frac{a^2}{nm}\sum_{i=1}^n
\rho_{ii}\,\|\mathbf t_i\|_2^2.
\tag{C.17}
\end{equation}

To connect $\|\mathbf t_i\|_2^2$ with the similarity factors, introduce
\[
\mathbf z_i
\;\triangleq\;
\mathbf s_i\odot\mathbf s_i,
\qquad
\mathbf v_i
\;\triangleq\;
\mathbf p_i\odot\mathbf p_i.
\]
Then
\begin{equation}\label{eq:C-ui-vi-inner}
\langle\mathbf z_i,\mathbf v_i\rangle
=
\sum_{r=1}^m s_{i,r}^2 p_{i,r}^2
=
\|\mathbf s_i\odot\mathbf p_i\|_2^2
=
\|\mathbf t_i\|_2^2.
\tag{C.18}
\end{equation}
By the cosine representation,
\begin{equation}\label{eq:C-ui-vi-cos}
\langle\mathbf z_i,\mathbf v_i\rangle
=
\|\mathbf z_i\|_2\,\|\mathbf v_i\|_2\,\kappa_{ii},
\tag{C.19}
\end{equation}
where $\kappa_{ii}=\cos\angle(\mathbf z_i,\mathbf v_i)$.

On the other hand, from~\eqref{eq:C-taus-def} and~\eqref{eq:C-tauP-def},
\[
\tau_{s,ii}
=
\frac{\|\mathbf s_i\odot\mathbf s_i\|_2^2}{\|\mathbf s_i\|_2^4}
=
\frac{\|\mathbf z_i\|_2^2}{\|\mathbf s_i\|_2^4},
\qquad
\tau_{p,ii}
=
\frac{\|\mathbf p_i\odot\mathbf p_i\|_2^2}{\|\mathbf p_i\|_2^4}
=
\frac{\|\mathbf v_i\|_2^2}{\|\mathbf p_i\|_2^4}.
\]
Taking square roots yields
\begin{equation}\label{eq:C-ui-vi-norms}
\begin{aligned}
\|\mathbf z_i\|_2
=
\|\mathbf s_i\|_2^2\,\sqrt{\tau_{s,ii}},
\\
\|\mathbf v_i\|_2
=
\|\mathbf p_i\|_2^2\,\sqrt{\tau_{p,ii}}.
\end{aligned}
\tag{C.20}
\end{equation}
Combining~\eqref{eq:C-ui-vi-inner}, \eqref{eq:C-ui-vi-cos}, and
\eqref{eq:C-ui-vi-norms},
\[
\|\mathbf t_i\|_2^2
=
\langle\mathbf z_i,\mathbf v_i\rangle
=
\|\mathbf s_i\|_2^2\,\|\mathbf p_i\|_2^2\,
\sqrt{\tau_{s,ii}\tau_{p,ii}}\,\kappa_{ii}.
\]
Substituting into~\eqref{eq:C-mu-basic} gives
\begin{equation}\label{eq:C-mu-exact}
\mu_\lambda
=
\frac{a^2}{nm}\sum_{i=1}^n
\rho_{ii}\,
\|\mathbf s_i\|_2^2\,\|\mathbf p_i\|_2^2\,
\sqrt{\tau_{s,ii}\tau_{p,ii}}\,\kappa_{ii},
\tag{C.21}
\end{equation}
which is precisely the mean-eigenvalue expression
\eqref{eq:mu-exact-thm} stated in the main text.

\paragraph{Second spectral moment.}
The second spectral moment is
\begin{equation}\label{eq:C-trH2-def}
\frac{1}{n}\operatorname{Tr}(\mathbf H^2)
=
\frac{1}{n}\sum_{i=1}^n\sum_{j=1}^n H_{ij}^2.
\tag{C.22}
\end{equation}
Using~\eqref{eq:C-NTK-compact},
\begin{equation}\label{eq:C-Hij-square}
H_{ij}^2
=
\Big(\frac{a^2}{m}\Big)^2
\rho_{ij}^2\,
\langle\mathbf t_i,\mathbf t_j\rangle^2.
\tag{C.23}
\end{equation}

First, the input term $\rho_{ij}^2$ can be written using
$\tau_{x,ij}$:
\begin{equation}\label{eq:C-rho-factor}
\rho_{ij}^2
=
\rho_{ii}\rho_{jj}\,\tau_{x,ij},
\tag{C.24}
\end{equation}
which follows directly from~\eqref{eq:C-taux-def}.

Next, we factorize $\langle\mathbf t_i,\mathbf t_j\rangle^2$.  
Define
\[
\mathbf z_{ij}
\;\triangleq\;
\mathbf s_i\odot\mathbf s_j,
\qquad
\mathbf v_{ij}
\;\triangleq\;
\mathbf p_i\odot\mathbf p_j.
\]
Then
\[
\langle\mathbf t_i,\mathbf t_j\rangle
=
\sum_{r=1}^m s_{i,r}p_r\,s_{j,r}p_r
=
\langle\mathbf z_{ij},\mathbf v_{ij}\rangle.
\]
By the cosine representation in~\eqref{eq:C-tauQ-def},
\begin{equation}\label{eq:C-uij-vij-cos}
\langle\mathbf z_{ij},\mathbf v_{ij}\rangle
=
\|\mathbf z_{ij}\|_2\,\|\mathbf v_{ij}\|_2\,\kappa_{ij},
\qquad
\tau_{q,ij}
=\kappa_{ij}^2.
\tag{C.25}
\end{equation}
From~\eqref{eq:C-taus-def} and~\eqref{eq:C-tauP-def},
\[
\tau_{s,ij}
=
\frac{\|\mathbf s_i\odot\mathbf s_j\|_2^2}
     {\|\mathbf s_i\|_2^2\,\|\mathbf s_j\|_2^2}
=
\frac{\|\mathbf z_{ij}\|_2^2}
     {\|\mathbf s_i\|_2^2\,\|\mathbf s_j\|_2^2},
\]
\[
\tau_{p,ij}
=
\frac{\|\mathbf p_i\odot\mathbf p_j\|_2^2}
     {\|\mathbf p_i\|_2^2\,\|\mathbf p_j\|_2^2}
=
\frac{\|\mathbf v_{ij}\|_2^2}
     {\|\mathbf p_i\|_2^2\,\|\mathbf p_j\|_2^2}.
\]
Taking square roots yields
\begin{equation}\label{eq:C-uij-vij-norms}
\begin{aligned}
\|\mathbf z_{ij}\|_2
=
\|\mathbf s_i\|_2\,\|\mathbf s_j\|_2\,\sqrt{\tau_{s,ij}},
\\
\|\mathbf v_{ij}\|_2
=
\|\mathbf p_i\|_2\,\|\mathbf p_j\|_2\,\sqrt{\tau_{p,ij}}.
\end{aligned}
\tag{C.26}
\end{equation}
Combining~\eqref{eq:C-uij-vij-cos} and \eqref{eq:C-uij-vij-norms},
\begin{align*}
\langle\mathbf t_i,\mathbf t_j\rangle^2
&=
\|\mathbf z_{ij}\|_2^2\,\|\mathbf v_{ij}\|_2^2\,\kappa_{ij}^2\\[2pt]
&=
\big(\|\mathbf s_i\|_2^2\|\mathbf s_j\|_2^2\tau_{s,ij}\big)\,
\big(\|\mathbf p_i\|_2^2\|\mathbf p_j\|_2^2\tau_{p,ij}\big)\,
\tau_{q,ij}.
\end{align*}

Substituting this and~\eqref{eq:C-rho-factor} into~\eqref{eq:C-Hij-square} gives
\begin{align*}
H_{ij}^2
=
\frac{a^4}{m^2}\,
\big(\rho_{ii}\rho_{jj}\tau_{x,ij}\big)\, 
\big(\|\mathbf s_i\|_2^2\|\mathbf s_j\|_2^2\tau_{s,ij}\big)\,\\
\big(\|\mathbf p_i\|_2^2\|\mathbf p_j\|_2^2\tau_{p,ij}\big)\,
\tau_{q,ij}.
\end{align*}
Summing over $i,j$ and dividing by $n$,
\begin{equation}\label{eq:C-trH2-exact}
\begin{aligned}
\frac{1}{n}\operatorname{Tr}(\mathbf H^2)
&=
\frac{a^4}{nm^2}\sum_{i=1}^n\sum_{j=1}^n
\big(\rho_{ii}\rho_{jj}\tau_{x,ij}\big)\\
&\quad\cdot
\big(\|\mathbf s_i\|_2^2\|\mathbf s_j\|_2^2\tau_{s,ij}\big)\,
\big(\|\mathbf p_i\|_2^2\|\mathbf p_j\|_2^2\tau_{p,ij}\big)\,
\tau_{q,ij},
\end{aligned}
\tag{C.27}
\end{equation}
which is exactly the second-moment expression
\eqref{eq:q-exact-thm} in the main text.

Combining the exact formulas~\eqref{eq:C-mu-exact},\eqref{eq:C-trH2-exact} with~\eqref{eq:C-mu-v-def} yields the exact variance identity for the Hadamard NTK and thus completes the derivation of ~\ref{prop:ntk-mu-exact} and ~\autoref{thm:ntk-v-exact}.  Since $v_\lambda$ measures the dispersion of the eigenvalue sequence $\{\lambda_i\}$, these formulas provide a direct link between the similarity structure induced by the architecture and the strength of spectral bias.

\subsection{Variance Formula under Diagonal Regularity}
\label{sec:C-Variance-Formula}
We now derive the compact four-factor variance expression used in Section~\ref{sec:spectral} by imposing mild regularity assumptions on
diagonal quantities.  The goal is to expose more clearly how the off-diagonal similarity mass controls $v_\lambda$.

\paragraph{Diagonal regularity assumptions.}
Assume that, across the training set,
\begin{itemize}
\item[(i)] Encoded input norms are approximately constant:
\[
\rho_{ii}
=\|\tilde{\mathbf x}_i\|_2^2
\approx
R_x^2,\qquad i=1,\ldots,n.
\]
\item[(ii)] Hidden and modulation energies are approximately constant:
\[
\|\mathbf s_i\|_2^2\approx\overline S,
\qquad
\|\mathbf p_i\|_2^2\approx\overline P,
\qquad i=1,\ldots,n.
\]
\item[(iii)] Self-similarities and self-couplings satisfy
\[
\sqrt{\tau_{s,ii}\tau_{p,ii}}\,\kappa_{ii}
\approx 1,
\qquad i=1,\ldots,n.
\]
\end{itemize}
These conditions capture the empirical effect of spherical normalization and well-conditioned modulation: diagonal energy scales and self-alignment are approximately uniform.

\paragraph{Approximate mean.}
Substituting (i)–(iii) into~\eqref{eq:C-mu-exact},
\begin{align*}
\mu_\lambda
&=
\frac{a^2}{nm}\sum_{i=1}^n
\rho_{ii}\,
\|\mathbf s_i\|_2^2\,\|\mathbf p_i\|_2^2\,
\sqrt{\tau_{s,ii}\tau_{p,ii}}\,\kappa_{ii}\\[2pt]
&\approx
\frac{a^2}{nm}\sum_{i=1}^n
R_x^2\,\overline S\,\overline P\cdot 1\\[2pt]
&=
\frac{a^2 R_x^2\,\overline S\,\overline P}{m},
\end{align*}
so that
\begin{equation}\label{eq:C-mu-approx}
\mu_\lambda
\;\approx\;
\frac{a^2 R_x^2\,\overline S\,\overline P}{m}.
\tag{C.28}
\end{equation}

\paragraph{Approximate second moment.}
Applying the same approximations to~\eqref{eq:C-trH2-exact} gives
\begin{align*}
\frac{1}{n}\operatorname{Tr}(\mathbf H^2)
&= \frac{a^4}{nm^2}\sum_{i,j}
\big(\rho_{ii}\rho_{jj}\tau_{x,ij}\big)\,
\big(\|\mathbf s_i\|_2^2\|\mathbf s_j\|_2^2\tau_{s,ij}\big) \\
&\quad\cdot
\big(\|\mathbf p_i\|_2^2\|\mathbf p_j\|_2^2\tau_{p,ij}\big)\,
\tau_{q,ij} \\[2pt]
&\approx
\frac{a^4}{nm^2}\sum_{i,j}
\big(R_x^2R_x^2\tau_{x,ij}\big)\,
\big(\overline S\,\overline S\,\tau_{s,ij}\big)\,
\big(\overline P\,\overline P\,\tau_{p,ij}\big)\,
\tau_{q,ij} \\[2pt]
&=
\frac{a^4 R_x^4\,\overline S^{\,2}\,\overline P^{\,2}}{n\,m^2}
\sum_{i=1}^n\sum_{j=1}^n
\tau_{x,ij}\,\tau_{s,ij}\,\tau_{p,ij}\,\tau_{q,ij}.
\end{align*}
Thus
\begin{equation}\label{eq:C-trH2-approx}
\frac{1}{n}\operatorname{Tr}(\mathbf H^2)
\;\approx\;
\frac{a^4 R_x^4\,\overline S^{\,2}\,\overline P^{\,2}}{n\,m^2}
\sum_{i,j}
\tau_{x,ij}\,\tau_{s,ij}\,\tau_{p,ij}\,\tau_{q,ij}.
\tag{C.29}
\end{equation}

\paragraph{Approximate variance and off-diagonal mass.}
From~\eqref{eq:C-mu-approx},
\[
\mu_\lambda^2
\approx
\left(\frac{a^2 R_x^2\,\overline S\,\overline P}{m}\right)^2
=
\frac{a^4 R_x^4\,\overline S^{\,2}\,\overline P^{\,2}}{m^2}.
\]
Therefore,
\begin{align}
v_\lambda
&=
\frac{1}{n}\operatorname{Tr}(\mathbf H^2)-\mu_\lambda^2
\nonumber\\[2pt]
&\approx
\frac{a^4 R_x^4\,\overline S^{\,2}\,\overline P^{\,2}}{n\,m^2}
\sum_{i,j}\tau_{x,ij}\tau_{s,ij}\tau_{p,ij}\tau_{q,ij}
-
\frac{a^4 R_x^4\,\overline S^{\,2}\,\overline P^{\,2}}{m^2}
\nonumber\\[2pt]
&=
\frac{a^4 R_x^4\,\overline S^{\,2}\,\overline P^{\,2}}{n\,m^2}
\bigg(
\sum_{i,j}\tau_{x,ij}\tau_{s,ij}\tau_{p,ij}\tau_{q,ij}
-
n
\bigg),
\label{eq:C-v-preoff}
\tag{C.30}
\end{align}
where we used
\[
\frac{a^4 R_x^4\,\overline S^{\,2}\,\overline P^{\,2}}{m^2}
=
\frac{a^4 R_x^4\,\overline S^{\,2}\,\overline P^{\,2}}{n\,m^2}\cdot n.
\]

For each $i$, by definitions~\eqref{eq:C-taux-def}–\eqref{eq:C-tauQ-def},
\[
\tau_{x,ii}=\tau_{s,ii}=\tau_{p,ii}=\tau_{q,ii}=1,
\]
so the diagonal contributions add up to
\[
\sum_{i=1}^n
\tau_{x,ii}\tau_{s,ii}\tau_{p,ii}\tau_{q,ii}
=
n.
\]
Therefore
\[
\sum_{i,j}\tau_{x,ij}\tau_{s,ij}\tau_{p,ij}\tau_{q,ij}-n
=
\sum_{i\ne j}\tau_{x,ij}\tau_{s,ij}\tau_{p,ij}\tau_{q,ij},
\]
and~\eqref{eq:C-v-preoff} becomes
\begin{equation}\label{eq:C-v-four-factor}
v_\lambda
\;\approx\;
\frac{a^4 R_x^4\,\overline S^{\,2}\,\overline P^{\,2}}{n\,m^2}
\sum_{i\ne j}
\tau_{x,ij}\tau_{s,ij}\tau_{p,ij}\tau_{q,ij}.
\tag{C.31}
\end{equation}
This is exactly the four-factor variance formula \eqref{eq:v-compact-four} used in the main text. It shows that, up to a global prefactor depending on the diagonal energies, the eigenvalue variance $v_\lambda$ is governed by the off-diagonal similarity mass
\[
\sum_{i\ne j}\tau_{x,ij}\tau_{s,ij}\tau_{p,ij}\tau_{q,ij},
\]
which aggregates the effects of input geometry, normalized hidden similarity, modulation similarity, and their mutual alignment.

\subsection{Monotonicity in Similarity Factors and Architectural Corollaries}
\label{sec:C-Architectural-Corollaries}
We conclude by formalizing how the eigenvalue variance $v_\lambda$ depends on
the similarity factors $(\tau_{x,ij},\tau_{s,ij},\tau_{p,ij},\tau_{q,ij})$ and
by deriving the architectural corollaries invoked in Section~\ref{sec:spectral}.  Throughout this subsection we work with the exact NTK representation~\eqref{eq:C-NTK-compact} and the spectral statistics \eqref{eq:C-mu-exact}–\eqref{eq:C-trH2-exact}, together with the four-factor approximation~\eqref{eq:C-v-four-factor}.

\begin{proposition}[Monotonicity of $v_\lambda$ in similarity factors]\label{prop:C-mono}
Fix the diagonal quantities $\{\rho_{ii}\}$,$\{\|\mathbf s_i\|_2\}$, and $\{\|\mathbf p_i\|_2\}$.  For off-diagonal pairs $i\ne j$, treat $\tau_{x,ij}$, $\tau_{s,ij}$, $\tau_{p,ij}$, and $\tau_{q,ij}$ as variables in $[0,1]$.  If all but one of these families are held fixed, then the eigenvalue variance $v_\lambda$ is (weakly) increasing in the remaining family.
\end{proposition}

\begin{proof}
We start from the exact second-moment expression~\eqref{eq:C-trH2-exact}.  
Define the nonnegative coefficients
\begin{equation}\label{eq:C-alpha-def}
\alpha_{ij}
\;\triangleq\;
\rho_{ii}\rho_{jj}\,
\|\mathbf s_i\|_2^2\|\mathbf s_j\|_2^2\,
\|\mathbf p_i\|_2^2\|\mathbf p_j\|_2^2
\;\ge\;0.
\tag{C.32}
\end{equation}
Then~\eqref{eq:C-trH2-exact} can be written compactly as
\[
\frac{1}{n}\operatorname{Tr}(\mathbf H^2)
=
\frac{a^4}{nm^2}\sum_{i,j}
\alpha_{ij}\,
\tau_{x,ij}\tau_{s,ij}\tau_{p,ij}\tau_{q,ij}.
\]

Fix an off-diagonal pair $(i,j)$ and regard, for instance, $\tau_{x,ij}$ as the
only variable, with all other factors held fixed.  Denote
\[
G
\;\triangleq\;
\frac{1}{n}\operatorname{Tr}(\mathbf H^2).
\]
The dependence of $G$ on $\tau_{x,ij}$ is linear and takes the form
\[
\begin{aligned}
G(\tau_{x,ij})
=
\frac{a^4}{nm^2}\,
\alpha_{ij}\,
\tau_{x,ij}\tau_{s,ij}\tau_{p,ij}\tau_{q,ij}
\\ +\text{(terms independent of $\tau_{x,ij}$)}.
\end{aligned}
\]
Differentiating with respect to $\tau_{x,ij}$ yields
\begin{equation}\label{eq:C-deriv-taux}
\frac{\partial G}{\partial\tau_{x,ij}}
=
\frac{a^4}{nm^2}\,
\alpha_{ij}\,
\tau_{s,ij}\tau_{p,ij}\tau_{q,ij}
\;\ge\;0,
\tag{C.33}
\end{equation}
since $\alpha_{ij}\ge 0$ and
$\tau_{s,ij},\tau_{p,ij},\tau_{q,ij}\in[0,1]$.  The same calculation applies
when varying $\tau_{s,ij}$, $\tau_{p,ij}$, or $\tau_{q,ij}$ while holding the remaining factors fixed, because $G$ is always linear and the corresponding
partial derivatives are nonnegative.

Next, the mean eigenvalue $\mu_\lambda$ in~\eqref{eq:C-mu-exact} depends only
on the diagonal entries $H_{ii}$ and therefore does not depend on any off-diagonal similarities:
\begin{equation}\label{eq:C-mu-indep}
\frac{\partial\mu_\lambda}{\partial\tau_{\bullet,ij}}=0,
\qquad i\ne j,
\tag{C.34}
\end{equation}
for any $\tau_{\bullet,ij}\in\{\tau_{x,ij},\tau_{s,ij},
\tau_{p,ij},\tau_{q,ij}\}$.  
Since $v_\lambda = \frac{1}{n}\operatorname{Tr}(\mathbf H^2)-\mu_\lambda^2
= G-\mu_\lambda^2$, combining~\eqref{eq:C-deriv-taux} and
\eqref{eq:C-mu-indep} gives
\[
\frac{\partial v_\lambda}{\partial\tau_{\bullet,ij}}
=
\frac{\partial G}{\partial\tau_{\bullet,ij}}
\;\ge\;0,
\]
which proves the claimed monotonicity.
\end{proof}

This monotonicity has direct implications for architectural design, because changes in positional encoding, modulation, or normalization primarily act by shrinking one or more of the similarity factors $\tau_{x,ij}$, $\tau_{s,ij}$, $\tau_{p,ij}$, or $\tau_{q,ij}$ on off-diagonal pairs.

\paragraph{Positional encoding.}

\begin{corollary}[Positional encoding reduces NTK variance]\label{cor:C-pe}
Consider two models that differ only in their input representation: a raw-coordinate model and a model with positional encoding.
Let $v_\lambda^{\mathrm{bl}}$ and $v_\lambda^{(\mathrm{PE})}$ be the
corresponding eigenvalue variances, and let
$\tau_{x,ij}^{\mathrm{bl}}$ and
$\tau_{x,ij}^{(\mathrm{PE})}$ denote the input similarities appearing in
\eqref{eq:C-v-four-factor}.
Assume that
\[
\sum_{i\ne j}\tau_{x,ij}^{(\mathrm{PE})}
<
\sum_{i\ne j}\tau_{x,ij}^{\mathrm{bl}},
\]
while the hidden and modulation statistics are comparable in the sense that
for $i\ne j$
\[
\tau_{s,ij}^{(\mathrm{PE})}\approx\tau_{s,ij}^{\mathrm{bl}},\quad
\tau_{p,ij}^{(\mathrm{PE})}\approx\tau_{p,ij}^{\mathrm{bl}},\quad
\tau_{q,ij}^{(\mathrm{PE})}\approx\tau_{q,ij}^{\mathrm{bl}},
\]
and the prefactor in~\eqref{eq:C-v-four-factor} is the same up to constants.
Then
\[
v_\lambda^{(\mathrm{PE})}
<
v_\lambda^{\mathrm{bl}}.
\]
\end{corollary}

\begin{proof}
From the four-factor formula~\eqref{eq:C-v-four-factor},
\[
v_\lambda
\;\approx\;
\frac{a^4 R_x^4\,\overline S^{\,2}\,\overline P^{\,2}}{n\,m^2}
\sum_{i\ne j}
\tau_{x,ij}\,\tau_{s,ij}\,\tau_{p,ij}\,\tau_{q,ij}.
\]
By assumption, the scalar prefactor and the product $\tau_{s,ij}\tau_{p,ij}\tau_{q,ij}$ are (up to constants) the same for both
models on each off-diagonal pair.  Thus the only systematic change in
$v_\lambda$ comes from replacing$\{\tau_{x,ij}^{\mathrm{bl}}\}_{i\ne j}$ by
$\{\tau_{x,ij}^{(\mathrm{PE})}\}_{i\ne j}$ in a sum of nonnegative terms.The strict decrease in the off-diagonal mass of $\tau_{x,ij}$ therefore implies $v_\lambda^{(\mathrm{PE})}<v_\lambda^{\mathrm{bl}}$.
\end{proof}

\noindent
\emph{Interpretation.}
Positional encoding contracts off-diagonal input similarities
$\tau_{x,ij}$; via~\eqref{eq:C-v-four-factor} this directly shrinks
$v_\lambda$ and hence reduces spectral dispersion and spectral bias.

\paragraph{Spherical and TopK-SP normalization.}

\begin{corollary}[Spherical and TopK-SP normalization shrink the energy factor]\label{cor:C-sn}
Let $v_\lambda^{(\mathrm{bl})}$ be the variance of the unnormalized baseline model, and $v_\lambda^{(\mathrm{SP})}$,$v_\lambda^{(\mathrm{TopK})}$ the variances under spherical normalization and TopK-SP, respectively. Assume that the baseline scaling is given by
\[
  v_\lambda^{(\mathrm{bl})}
  \;\approx\;
  \frac{a^4 R_x^4\,\overline S_{\mathrm{bl}}^{\,2}}{n\,m^2}
  \sum_{i\ne j}\tau_{x,ij}\,\tau_{s,ij}^{(\mathrm{bl})};
\]
and spherical normalization enforces $\|\mathbf s_i\|_2^2\equiv 1$,so that the corresponding energy factor satisfies $\overline S\approx 1\ll\overline S_{\mathrm{bl}}$ when many neurons are active;,and in the TopK-SP variant, Appendix~D shows that the effective energy-weighted hidden similarity $\tau_{s,ij}\,\overline S^{\,2}$ is further reduced in expectation for $i\ne j$.
Then, whenever $\overline S_{\mathrm{bl}}\gg 1$ at initialization,
\[
\mathbb E\big[v_\lambda^{(\mathrm{TopK})}\big]
<
\mathbb E\big[v_\lambda^{(\mathrm{SP})}\big]
\ll
\mathbb E\big[v_\lambda^{(\mathrm{bl})}\big].
\]
\end{corollary}

\begin{proof}
In the baseline, $v_\lambda^{(\mathrm{bl})}$ scales with$\overline S_{\mathrm{bl}}^{\,2}$ as in Proposition~\ref{prop:bl-spectrum}.Under spherical normalization, \eqref{eq:C-v-four-factor} shows that the variance scales with an energy factor of order $\overline S^{\,2}\approx 1$,while the off-diagonal similarity sum remains of comparable order.
Thus, for $\overline S_{\mathrm{bl}}\gg 1$,
\[
\mathbb E\big[v_\lambda^{(\mathrm{SP})}\big]
\ll
\mathbb E\big[v_\lambda^{(\mathrm{bl})}\big].
\]
In the TopK-SP case, the same formula applies with the TopK-induced energy and similarities.  The assumption that $\tau_{s,ij}\,\overline S^{\,2}$ is strictly smaller in expectation for each $i\ne j$ implies that every off-diagonal summand in the variance sum is reduced in expectation, while all other factors are nonnegative.
Summing over $i\ne j$ yields $\mathbb E[v_\lambda^{(\mathrm{TopK})}]<\mathbb E[v_\lambda^{(\mathrm{SP})}]$. Combining the two inequalities gives the claim.
\end{proof}

\noindent
\emph{Interpretation.}
The unnormalized baseline carries a large factor $\overline S_{\mathrm{bl}}^{\,2}$ in $v_\lambda$.Spherical normalization removes this amplification by enforcing unit-norm hidden vectors, and TopK-SP further shrinks the energy-weighted hidden similarity.  This progressively tightens the NTK spectrum and attenuates spectral bias.

\paragraph{Hadamard modulation.}

\begin{corollary}[Hadamard modulation reduces eigenvalue variance]\label{cor:C-mod}
Consider two models with the same encoded inputs and normalized hidden
vectors, but different modulation patterns.In the baseline model, the modulation is trivial so that$\tau_{p,ij}^{(\mathrm{bl})}=\tau_{q,ij}^{(\mathrm{bl})}=1$
for all $i,j$. In the Hadamard-enhanced model, the same $\tau_{x,ij}$ and $\tau_{s,ij}$ are used, but the modulation induces
$\tau_{p,ij} $ and
$\tau_{q,ij} $ for $i\ne j$.
Assume that there exists a constant $C_P\in(0,1)$ such that
\[
\tau_{p,ij} \,\tau_{q,ij} 
\le C_P
\qquad\text{for all }i\ne j,
\]
and that the prefactor in~\eqref{eq:C-v-four-factor} is comparable in both
models.  Then
\[
v_\lambda 
\;\le\;
C_P\,v_\lambda^{(\mathrm{bl})},
\]
with strict inequality if the above bound is strict on a nonempty set of off-diagonal pairs.
\end{corollary}

\begin{proof}
For the baseline model, \eqref{eq:C-v-four-factor} becomes
\[
v_\lambda^{(\mathrm{bl})}
\;\approx\;
\frac{a^4 R_x^4\,\overline S^{\,2}\,\overline P^{\,2}}{n\,m^2}
\sum_{i\ne j}\tau_{x,ij}\,\tau_{s,ij},
\]
since $\tau_{p,ij}=\tau_{q,ij}=1$.
For the Hadamard-enhanced model,
\[
v_\lambda 
\;\approx\;
\frac{a^4 R_x^4\,\overline S^{\,2}\,\overline P^{\,2}}{n\,m^2}
\sum_{i\ne j}\tau_{x,ij}\,\tau_{s,ij}\,
\tau_{p,ij} \tau_{q,ij} ,
\]
where we used $\tau_{x,ij}$ and $\tau_{s,ij}$ being shared. Each off-diagonal summand in the Hadamard case is bounded by $C_P$ times the corresponding baseline summand, and all factors are nonnegative, so
\[
\sum_{i\ne j}\tau_{x,ij}\tau_{s,ij}
\tau_{p,ij} \tau_{q,ij} 
\le
C_P\sum_{i\ne j}\tau_{x,ij}\tau_{s,ij}.
\]
Multiplying by the common prefactor yields the claim.
\end{proof}

\noindent
\emph{Interpretation.}
The Hadamard branch injects an additional similarity structure over hidden channels.  When its off-diagonal alignment is weaker than that of the trivial modulation, the factors $\tau_{p,ij}$ and $\tau_{q,ij}$ shrink each off-diagonal contribution in~\eqref{eq:C-v-four-factor}, thereby clustering the NTK eigenvalues and weakening spectral bias.Corollary~\ref{cor:C-mod} is stated at the level of the compact four-factor proxy\eqref{eq:C-v-four-factor}, whose baseline reference term is linear in the hidden similarity $\tau_{s,ij}$.By contrast, in the unnormalized baseline derivation of Appendix~A the off-diagonal contribution to the second spectral moment naturally carries the squared overlap factor $(\tau_{s,ij}^{\mathrm{bl}})^2$ (cf.\ \eqref{eq:A-Hij-square}--\eqref{eq:A-trH2-exact}).The commonly used replacement $(\tau_{s,ij}^{\mathrm{bl}})^2\approx
\tau_{s,ij}^{\mathrm{bl}}\cdot \overline{\tau}_s$ is therefore best viewed as a mean-field simplification: it sharpens the interpretation of which similarities drive the variance, but it is not an identity and should not be taken as a pointwise comparison against the squared-baseline term.

To articulate the stricter comparison, we keep track of the corresponding off-diagonal similarity masses.  For the Hadamard model, \eqref{eq:C-v-four-factor} identifies
\[
\mathcal M_{\mathrm{ha}}
\;\triangleq\;
\sum_{i\neq j}\tau_{x,ij}\,\tau_{s,ij}\,\tau_{p,ij}\,\tau_{q,ij},
\]
while Appendix~A suggests the unlinearized baseline reference
\[
\mathcal M_{\mathrm{bl}}^{(2)}
\;\triangleq\;
\sum_{i\neq j}\tau_{x,ij}^{\mathrm{bl}}\,(\tau_{s,ij}^{\mathrm{bl}})^2.
\]
Under the shared-input setting (so that $\tau_{x,ij}\approx\tau_{x,ij}^{\mathrm{bl}}$) and comparable diagonal prefactors in the two variance proxies, the implication is direct: any uniform domination of each Hadamard off-diagonal summand by its squared-baseline counterpart yields a reduction of the dominant off-diagonal variance contribution, and hence a tighter spectrum in the sense of the proxy $v_\lambda$.

A sufficient condition that makes this domination explicit is the following pairwise bound:
there exists $C_\Pi\in(0,1)$ such that for all $i\neq j$,
\begin{equation}\label{eq:C-mod-square-dominate}
\tau_{s,ij}\,\tau_{p,ij}\,\tau_{q,ij}
\;\le\;
C_\Pi\,(\tau_{s,ij}^{\mathrm{bl}})^2.
\tag{C.32}
\end{equation}
Multiplying \eqref{eq:C-mod-square-dominate} by $\tau_{x,ij}\ge 0$ and summing over$i\neq j$ gives $\mathcal M_{\mathrm{ha}}\le C_\Pi\,\mathcal M_{\mathrm{bl}}^{(2)}$, with strict inequality whenever \eqref{eq:C-mod-square-dominate} is strict on a nonempty set of off-diagonal pairs.
With comparable prefactors, this yields a strict reduction of the off-diagonal variance mass,and therefore a strict improvement of the spectral-dispersion proxy $v_\lambda$.

It remains to justify why \eqref{eq:C-mod-square-dominate} is reasonable in the regimes considered in this work.  First, at random ReLU initialization the baseline gate-overlap similarity typically satisfies $\tau_{s,ij}^{\mathrm{bl}}=\Theta(m^{-1})$ for $i\neq j$, hence $(\tau_{s,ij}^{\mathrm{bl}})^2=\Theta(m^{-2})$.  Consequently, achieving \eqref{eq:C-mod-square-dominate} reduces to ensuring that the Hadamard product $\tau_{s,ij}\tau_{p,ij}\tau_{q,ij}$ is also of order $m^{-2}$ (or smaller) on off-diagonal pairs.  Second, the normalization mechanisms studied in the paper (spherical normalization and its TopK refinement) are introduced precisely to suppress off-diagonal overlap, and empirically yield $\tau_{s,ij}\lesssim \tau_{s,ij}^{\mathrm{bl}}$.  Third, once the modulation pattern is phase-mixed (or otherwise channel-incoherent) relative to the hidden pair-structure, the alignment term $\kappa_{ij}^2=\tau_{q,ij}$ becomes small for $i\neq j$: under mild mixing/independence across channels, standard concentration bounds imply $\tau_{q,ij}=O(m^{-1})$ uniformly over off-diagonal pairs (up to logarithmic factors in $n$). When, in addition, the modulation overlap itself satisfies $\tau_{p,ij}=O(m^{-1})$ for $i\neq j$ (as is typical when $\mathbf p_i\odot\mathbf p_j$ oscillates across channels), their product obeys $\tau_{p,ij}\tau_{q,ij}=O(m^{-2})$, and thus  $\tau_{s,ij}\tau_{p,ij}\tau_{q,ij}\lesssim (\tau_{s,ij}^{\mathrm{bl}})^2$ for sufficiently large width.

Taken together, these observations explain how the informal statement $\tau_{p,ij}\tau_{q,ij}$ is small off-diagonal can be upgraded into a domination at the correct squared-baseline scale \eqref{eq:C-mod-square-dominate}.  In turn, this closes the logical loop behind Corollary~\ref{cor:C-mod}: beyond the interpretability of the linearized proxy, one can ensure a reduction of the dominant off-diagonal variance contribution at the unlinearized baseline level, and therefore justify that Hadamard modulation tightens the NTK spectrum and weakens spectral bias in the sense captured by $v_\lambda$.

\section{Top-K Sparse Spherical Normalization and Energy-Weighted Hidden Similarity}
\label{sec:appD}

This section makes precise how Top-$K$ sparse spherical normalization (TopK-SP) modifies the hidden similarity factor that appears in our NTK variance proxy. We first formalize the notion of energy-weighted hidden similarity, then compare three normalization schemes: full spherical normalization (SP), a uniform-$k$ sparsification baseline, and deterministic TopK-SP. The main message is that, under mild sub-Gaussian assumptions, TopK-SP preserves the $1/m$ scaling of hidden similarity but induces a strict shrinkage in the energy-weighted overlap, which in turn suppresses NTK spectral variance.

\subsection{Setup and Notation}

For sample $i\in[n]$, let the pre-normalization hidden vector be
\[
\mathbf y_i=(y_{i,1},\ldots,y_{i,m})^\top\in\mathbb R^m.
\]
For any (possibly normalized) features $\mathbf s_i,\mathbf s_j\in\mathbb R^m$, we define the hidden similarity as
\begin{equation}\label{eq:D-tau}
\tau_{s,ij}
\ \triangleq\
\frac{\|\mathbf s_i\odot \mathbf s_j\|_2^2}{\|\mathbf s_i\|_2^2\,\|\mathbf s_j\|_2^2}
=\sum_{r=1}^m \frac{s_{i,r}^2\,s_{j,r}^2}{\|\mathbf s_i\|_2^2\,\|\mathbf s_j\|_2^2}
\in[0,1].
\tag{D.1}
\end{equation}
When $\mathbf s_i$ and $\mathbf s_j$ are individually $\ell_2$-normalized, the denominator equals $1$, and $\tau_{s,ij}$ reduces to a simple overlap of squared coordinates.

We compare two concrete normalization schemes, both applied to the same pre-normalization vectors~$\mathbf y_i$.

\noindent\emph{Spherical normalization (SP):}
\begin{equation}\label{eq:D-SP}
\mathbf s_i^{\mathrm{sp}}=\frac{\mathbf y_i}{\|\mathbf y_i\|_2},
\qquad
\tau^{\mathrm{sp}}_{s,ij}
=\frac{\sum_{r=1}^m y_{i,r}^2 y_{j,r}^2}{\big(\sum_{r=1}^m y_{i,r}^2\big)\big(\sum_{r=1}^m y_{j,r}^2\big)}.
\tag{D.2}
\end{equation}
Here $\|\mathbf s_i^{\mathrm{sp}}\|_2=\|\mathbf s_j^{\mathrm{sp}}\|_2=1$, so \eqref{eq:D-tau} indeed reduces to the ratio in~\eqref{eq:D-SP}.

\noindent\emph{Top-$K$ sparse spherical normalization (TopK-SP):}
Let $K_i=\operatorname{TopK}(\{|y_{i,r}|\}_{r=1}^m)$ with $|K_i|=k\ll m$ and indicator vector $(\mathbf 1_{K_i})_r=\mathbf 1_{\{r\in K_i\}}$. Define
\begin{equation}\label{eq:D-TopK}
\mathbf s_i^{\mathrm{tk}}=\frac{\mathbf y_i\odot \mathbf 1_{K_i}}{\|\mathbf y_i\odot \mathbf 1_{K_i}\|_2},
\qquad
\tau^{\mathrm{tk}}_{s,ij}=\sum_{r=1}^m (s^{\mathrm{tk}}_{i,r})^2 (s^{\mathrm{tk}}_{j,r})^2,
\tag{D.3}
\end{equation}
where we have used that $\|\mathbf s_i^{\mathrm{tk}}\|_2=\|\mathbf s_j^{\mathrm{tk}}\|_2=1$ so that the denominator in \eqref{eq:D-tau} is again equal to $1$.

We work under the following mild conditions, used only to justify law-of-large-numbers and concentration steps.

\begin{itemize}
\item[(A1)] \textbf{(Channel i.i.d.\ and integrability)} For fixed $i$, the coordinates $\{y_{i,r}\}_{r=1}^m$ are i.i.d.\ with
\[
\mathbb E[y_{i,r}]=\mu_i,\qquad
\mathrm{Var}(y_{i,r})=\sigma_i^2\in(0,\infty),\qquad
\mathbb E[y_{i,r}^4]<\infty,
\]
and absolutely continuous. Moreover $y_{i,r}-\mu_i$ is sub-Gaussian with parameter $K_{\mathrm{sg}}$.
\item[(A2)] \textbf{(Same-channel cross-sample correlation)} For $i\neq j$,
\[
\mathrm{Cov}(y_{i,r},y_{j,r})=c_{ij}\sigma_i\sigma_j
\quad\text{(independent of $r$)},\qquad
\mathrm{Cov}(y_{i,r},y_{j,s})=0\ \text{for }r\neq s,
\]
where $c_{ij}\in[-1,1]$ is a correlation coefficient.
\item[(A3)] \textbf{(Moderate SNR)} The signal-to-noise ratio is uniformly bounded: $\mu_i^2/\sigma_i^2\le \kappa$ for all $i$ and $m$.
\item[(A4)] \textbf{(Large width)} The hidden width satisfies $m\to\infty$.
\item[(A5)] \textbf{(Weak rank correlation under small $c_{ij}$)} As $c_{ij}\to 0$, the (Spearman) rank correlation between the sequences $\{|y_{i,r}|\}_{r=1}^m$ and $\{|y_{j,r}|\}_{r=1}^m$ vanishes (e.g., joint Gaussian or Gaussian-copula constructions). Intuitively, when samples $i$ and $j$ are weakly correlated, their Top-$K$ index sets behave as if they were almost independently drawn from the same marginal distribution.
\end{itemize}

To connect normalization to the NTK variance proxy, we consider an energy-weighted version of the hidden similarity. Let a mask $\mathsf M_{i,r}\equiv 1$ for SP and $\mathsf M_{i,r}=\mathbf 1_{\{r\in K_i\}}$ for TopK-SP. Define
\begin{equation}\label{eq:D-energy}
S_i\ \triangleq\ \frac{1}{m}\sum_{r=1}^m y_{i,r}^2\,\mathsf M_{i,r},
\qquad
\overline S\ \triangleq\ \mathbb E[S_i],
\qquad
\mathcal M_{ij}\ \triangleq\ \tau_{s,ij}\,\overline S^{\,2}.
\tag{D.4}
\end{equation}
The quantity $\mathcal M_{ij}$ matches, up to architecture-specific constants, the hidden ``energy $\times$ similarity factor that appears in the NTK variance proxy in Appendix~A and Appendix~C.

\subsection{SP Baseline: Mean Scale and Energy}

We first characterize the baseline SP similarity and energy. Let
\[
X_r=y_{i,r},\quad Y_r=y_{j,r},\quad Z_r=X_r^2Y_r^2,
\]
and define
\begin{equation}\label{eq:D-SiSjN}
S_i^\Sigma=\sum_{r=1}^m X_r^2,\qquad
S_j^\Sigma=\sum_{r=1}^m Y_r^2,\qquad
N=\sum_{r=1}^m Z_r.
\tag{D.5}
\end{equation}
Then \eqref{eq:D-SP} can be written as
\[
\tau^{\mathrm{sp}}_{s,ij}
=\frac{N}{S_i^\Sigma S_j^\Sigma}.
\]

\begin{lemma}\label{lem:D-LLN}
Under \textnormal{(A1)--(A4)},
\[
\begin{gathered}
\frac{S_i^\Sigma}{m}\xrightarrow{\mathrm{a.s.}}\mathbb E[X^2]=\mu_i^2+\sigma_i^2,\\[2pt]
\frac{S_j^\Sigma}{m}\xrightarrow{\mathrm{a.s.}}\mathbb E[Y^2]=\mu_j^2+\sigma_j^2,\\[2pt]
\frac{N}{m}\xrightarrow{\mathrm{a.s.}}\mathbb E[Z]=\mathbb E[X^2Y^2].
\end{gathered}
\]
Consequently,
\begin{equation}\label{eq:D-Asp}
\mathbb E[\tau^{\mathrm{sp}}_{s,ij}]
\ =\ \frac{A_{\mathrm{sp}}}{m}+o\Big(\frac{1}{m}\Big),
\qquad
A_{\mathrm{sp}}:=\frac{\mathbb E[X^2Y^2]}{\mathbb E[X^2]\mathbb E[Y^2]}.
\tag{D.6}
\end{equation}
\end{lemma}

\begin{proof}
By (A1), $\{X_r\}_{r=1}^m$ and $\{Y_r\}_{r=1}^m$ are i.i.d.\ with finite fourth moment, and the pairs $(X_r,Y_r)$ are independent across $r$ by (A2). The strong law of large numbers gives
\[
\frac{S_i^\Sigma}{m}=\frac1m\sum_{r=1}^m X_r^2\xrightarrow{\mathrm{a.s.}}\mathbb E[X^2],
\quad
\frac{S_j^\Sigma}{m}\xrightarrow{\mathrm{a.s.}}\mathbb E[Y^2],
\quad
\frac{N}{m}=\frac1m\sum_{r=1}^m X_r^2Y_r^2\xrightarrow{\mathrm{a.s.}}\mathbb E[X^2Y^2].
\]
Consider the smooth map
\[
g(a,b,c)=\frac{c}{ab},\qquad a,b>0.
\]
By (A1) we have $\mathbb E[X^2]>0$ and $\mathbb E[Y^2]>0$, and the SNR bound (A3) prevents degeneracy. Thus $(S_i^\Sigma/m,S_j^\Sigma/m,N/m)$ converges a.s.\ to $(\mathbb E[X^2],\mathbb E[Y^2],\mathbb E[X^2Y^2])$ in the domain of $g$, and the continuous mapping theorem yields
\[
\tau^{\mathrm{sp}}_{s,ij}
=g\Big(\frac{S_i^\Sigma}{m},\frac{S_j^\Sigma}{m},\frac{N}{m}\Big)
\xrightarrow{\mathrm{a.s.}}
g\big(\mathbb E[X^2],\mathbb E[Y^2],\mathbb E[X^2Y^2]\big)
=\frac{\mathbb E[X^2Y^2]}{\mathbb E[X^2]\mathbb E[Y^2]}\cdot\frac{1}{m}
\]
after factoring out the common $1/m$ scale. A standard Delta-method expansion, together with uniform integrability guaranteed by $\mathbb E[X^4]<\infty$ and $\mathbb E[Y^4]<\infty$, then gives the first-order expansion in \eqref{eq:D-Asp}.
\end{proof}

For the energy factor, SP corresponds to $\mathsf M_{i,r}\equiv 1$ in \eqref{eq:D-energy}, and hence
\begin{lemma}\label{lem:D-Sbar-sp}
For SP,
\begin{equation}\label{eq:D-Sbar-sp}
\overline S_{\mathrm{sp}}
=\mathbb E\!\Big[\frac{1}{m}\sum_{r=1}^m y_{i,r}^2\Big]
=\mathbb E[X^2]
=\mu_i^2+\sigma_i^2
\ \eqqcolon\ \mu_{2,i}.
\tag{D.7}
\end{equation}
\end{lemma}

Thus, under SP, both the similarity term $\mathbb E[\tau^{\mathrm{sp}}_{s,ij}]$ and the energy factor $\overline S_{\mathrm{sp}}$ admit clean closed forms in terms of the second and mixed fourth moments of $(X,Y)$.

\subsection{Uniform-$k$ Sparsification as a Reference Baseline}

Before turning to deterministic TopK-SP, it is helpful to introduce a simplified baseline in which sparsification is random and independent of the hidden values. This allows us to isolate the effect of reducing the support size from $m$ to $k$ without the order-statistic bias of TopK.

Independently of $(\mathbf y_i,\mathbf y_j)$, draw $K_i,K_j\subset\{1,\ldots,m\}$ uniformly at random without replacement with $|K_i|=|K_j|=k$, and define
\begin{equation}\label{eq:D-Uk}
\begin{gathered}
D_i=\sum_{u\in K_i}y_{i,u}^2,\qquad
D_j=\sum_{v\in K_j}y_{j,v}^2,\\[2pt]
\tau^{\mathrm{uk}}_{s,ij}
=\frac{\sum_{r\in K_i\cap K_j}y_{i,r}^2y_{j,r}^2}{D_iD_j}.
\end{gathered}
\tag{D.8}
\end{equation}

\begin{lemma}\label{lem:D-Uk-conds}
Conditioned on $(\mathbf y_i,\mathbf y_j)$,
\[
\mathbb E\!\Big[\sum_{r\in K_i\cap K_j}y_{i,r}^2y_{j,r}^2\ \Big|\ \mathbf y_i,\mathbf y_j\Big]
=\Big(\frac{k}{m}\Big)^2\sum_{r=1}^m X_r^2Y_r^2,
\]
and Serfling’s finite-population concentration inequality yields
\[
\begin{gathered}
D_i=\frac{k}{m}\sum_{r=1}^m X_r^2\big(1+O_{\mathbb P}(\varepsilon_{m,k})\big),\\[2pt]
D_j=\frac{k}{m}\sum_{r=1}^m Y_r^2\big(1+O_{\mathbb P}(\varepsilon_{m,k})\big),
\end{gathered}
\]
with $\varepsilon_{m,k}=C\sqrt{\frac{1-k/m}{k}}$ for some absolute constant $C>0$.
\end{lemma}

\begin{proposition}\label{prop:D-UkSP}
If $k\to\infty$ and $k=o(m)$, then
\begin{equation}\label{eq:D-UkSP-eq}
\mathbb E[\tau^{\mathrm{uk}}_{s,ij}]
=\mathbb E[\tau^{\mathrm{sp}}_{s,ij}]\,\big(1+O(\varepsilon_{m,k})\big),
\qquad \varepsilon_{m,k}=C\sqrt{\frac{1-k/m}{k}}.
\tag{D.9}
\end{equation}
\end{proposition}

\begin{proof}
Condition on $(\mathbf y_i,\mathbf y_j)$. By Lemma~\ref{lem:D-Uk-conds}, the numerator is equal in expectation to $(k/m)^2 N$, while $D_iD_j$ concentrates around $(k/m)^2 S_i^\Sigma S_j^\Sigma$ with relative error $O_{\mathbb P}(\varepsilon_{m,k})$. Thus
\[
\tau^{\mathrm{uk}}_{s,ij}
=\frac{N}{S_i^\Sigma S_j^\Sigma}\,\big(1+O_{\mathbb P}(\varepsilon_{m,k})\big)
=\tau^{\mathrm{sp}}_{s,ij}\,\big(1+O_{\mathbb P}(\varepsilon_{m,k})\big).
\]
Taking expectations over $(\mathbf y_i,\mathbf y_j)$ and using bounded fourth moments to control tails yields the stated expansion.
\end{proof}

Uniform-$k$ sparsification is therefore essentially neutral with respect to the mean scale of $\tau_s$: it preserves the SP behavior up to a vanishing $O(\varepsilon_{m,k})$ error, so any substantial change in hidden similarity must come from the deterministic, value-dependent nature of TopK.

\subsection{Deterministic Top-$K$: Mean Behavior of $\tau_s$}

We now return to TopK-SP, where the sparsity pattern is strongly coupled to the hidden values through order statistics. Our goal is to show that, despite this coupling, the mean similarity scale remains $\Theta(1/m)$, while the energy factor shrinks.

Let $t>0$ be such that
\[
\mathbb P(|y_{i,r}|\ge t)=\bar p=\frac{k}{m},
\]
which is uniquely defined by absolute continuity in (A1). We introduce the threshold sets
\[
K_i=\{r:\ |y_{i,r}|\ge t\},\qquad
K_j=\{r:\ |y_{j,r}|\ge t\},
\]
and note that $|K_i|\approx k$ and $|K_j|\approx k$ with high probability by Hoeffding-type concentration for Bernoulli indicators.

To understand the intersection size $|K_i\cap K_j|$, consider the standardized pair
\[
(X,Y)\sim \mathcal N\!\Big(
0,\begin{bmatrix}1 & c_{ij}\\ c_{ij} & 1\end{bmatrix}
\Big),
\]
which models the same-channel pair $(y_{i,r},y_{j,r})$ after centering and rescaling under (A1)--(A3). Define
\[
p_{c_{ij}}(t)=\mathbb P(|X|\ge t,\ |Y|\ge t).
\]
A standard bivariate normal expansion around $c_{ij}=0$ yields (for $|c_{ij}|\ll 1$)
\begin{equation}\label{eq:D-bivar-tail}
\begin{gathered}
p_{c_{ij}}(t)=p_0(t)+4\phi(t)^2\,c_{ij}+O(c_{ij}^2),
\\
p_0(t)=(2\bar\Phi(t))^2,
\end{gathered}
\tag{D.10}
\end{equation}
with $\phi,\Phi$ the standard normal pdf/cdf and $\bar\Phi(t)=1-\Phi(t)$. Since $2\bar\Phi(t)=\bar p=k/m$ by construction, we obtain
\begin{equation}\label{eq:D-Kcap}
\begin{gathered}
\mathbb E[|K_i\cap K_j|]
=m\,p_{c_{ij}}(t)
=\frac{k^2}{m}\Big(1+\alpha(t)c_{ij}+O(c_{ij}^2)\Big),
\\[2pt]
\alpha(t)=\frac{4\phi(t)^2}{(2\bar\Phi(t))^2}\ge 0.
\end{gathered}
\tag{D.11}
\end{equation}
Assumption (A5) ensures that, as $c_{ij}\to 0$, the rank correlation between $\{|y_{i,r}|\}_r$ and $\{|y_{j,r}|\}_r$ vanishes , so that the threshold approximation (and hence \eqref{eq:D-Kcap}) accurately describes the mean behavior of the deterministic TopK index sets.

Next,we show that TopK-SP does not collapse the energy of $\mathbf y_i$ onto a single coordinate. This guaranties that the normalization denominators remain well-behaved and that the per-coordinate weights are of order $1/k$.

The\footnote{ gate/activation is bounded (e.g., $\tanh$ or sigmoidal gates), or when we apply an explicit clipping layer. Bounded random variables are sub-Gaussian with parameter depending only on $B$, so this assumption is compatible with (A1).} Let $K_i$ denote the indices of the top entries $k$ of $\{|y_{i,r}|\}_{r=1}^m$, and put in place that there exists a constant $B<\infty$ (independent of $m,k$) such that $|y_{i,r}|\le B$ almost surely. This is naturally satisfied when the gate
\[
S_{k} \triangleq \sum_{u\in K_i} y_{i,u}^2,
\qquad
M_{k}\triangleq \max_{r\in K_i} y_{i,r}^2.
\]

\begin{lemma}[Non-concentration of Top-$K$ energy]\label{lem:D-maxentry}
If $k\ge c_0\log m$ for a sufficiently large constant $c_0>0$, then
\begin{equation}\label{eq:D-maxentry}
\max_{r\in K_i}\frac{y_{i,r}^2}{\sum_{u\in K_i}y_{i,u}^2}
=O_{\mathbb P}\!\Big(\frac{\log k}{k}\Big),
\qquad
\sum_{u\in K_i}y_{i,u}^2=\Theta_{\mathbb P}(k).
\tag{D.12}
\end{equation}
\end{lemma}

\begin{proof}
Let $t$ be the $(1-\bar p)$-quantile of $|y_{i,1}|$ with $\bar p=k/m$, i.e.,
$\mathbb P(|y_{i,1}|\ge t)=\bar p$. Consider the indicators
\[
Z_r=\mathbf 1\{|y_{i,r}|\ge t\},\qquad N_t=\sum_{r=1}^m Z_r.
\]
By Hoeffding’s inequality, $N_t$ concentrates around its mean $m\bar p=k$, and in particular
\[
\mathbb P\big(|N_t-k|\ge k/2\big)\le 2e^{-c k}
\]
for some absolute constant $c>0$. Hence, with probability $1-o(1)$, at least $k/2$ coordinates of $\{|y_{i,r}|\}$ are above $t$, and therefore the $k$-th largest magnitude satisfies $|y_{i}|_{(k)}\ge t$. On this event,
\begin{equation}\label{eq:lower-tail-threshold}
\min_{u\in K_i} |y_{i,u}|\ \ge\ t
\quad\Longrightarrow\quad
\sum_{u\in K_i} y_{i,u}^2 \ \ge\ k\,t^2.
\tag{D.13}
\end{equation}

By boundedness, $|y_{i,r}|\le B$ almost surely and the CDF of $|y_{i,1}|$ is continuous on $[0,B]$, strictly increasing near its upper tail. When $\bar p=k/m\to 0$ slowly (e.g., $k\ge c_0\log m$), the $(1-\bar p)$-quantile $t$ remains bounded away from zero: there exists $t_\star>0$ (independent of $m,k$) such that $t\ge t_\star$ for all sufficiently large $m$. Combining this with \eqref{eq:lower-tail-threshold} yields
\begin{equation}\label{eq:Sk-lower}
S_k=\sum_{u\in K_i} y_{i,u}^2 \ \ge\ k\,t_\star^2
\qquad \text{with probability } 1-o(1).
\tag{D.14}
\end{equation}

On the other hand, $y_{i,u}^2\le B^2$ for all $u$, so
\begin{equation}\label{eq:Sk-upper}
S_k=\sum_{u\in K_i} y_{i,u}^2 \ \le\ k\,B^2
\qquad \text{always}.
\tag{D.15}
\end{equation}
Thus $S_k\in[k\,t_\star^2,k\,B^2]$ with probability $1-o(1)$, i.e., $S_k=\Theta_{\mathbb P}(k)$.

Finally, $M_k=\max_{r\in K_i} y_{i,r}^2 \le B^2$ almost surely. Using \eqref{eq:Sk-lower},
\[
\frac{M_k}{S_k}
\ \le\
\frac{B^2}{k\,t_\star^2}
\ =\
\frac{C}{k},
\]
for $C=B^2/t_\star^2$. Since $k\ge c_0\log m\to\infty$, we have $C/k=O((\log k)/k)$ (trivially, as $\log k\le k$ for large $k$), which gives \eqref{eq:D-maxentry}.
\end{proof}

In words, Lemma~\ref{lem:D-maxentry} states that TopK-SP spreads the retained energy across $\Theta(k)$ channels, with each coordinate contributing at most $O((\log k)/k)$ of the total energy. This will directly control the weights in $\tau^{\mathrm{tk}}_{s,ij}$.

\begin{proposition}[TopK-SP preserves the $1/m$ similarity scale]\label{prop:D-TopK-tau}
Under \textnormal{(A1)--(A5)} and $k\ge c_0\log m$, there exists $\xi_{k,m}(c_{ij})\ge 0$, with
\[
\xi_{k,m}(c_{ij})
=O\!\Big(\frac{(\log k)^2}{k}\,(1+\alpha(t)c_{ij})\Big)
\xrightarrow[k\to\infty,\ c_{ij}\to 0]{}0,
\]
such that
\begin{equation}\label{eq:D-TopK-tau}
\mathbb E[\tau^{\mathrm{tk}}_{s,ij}]
=\mathbb E[\tau^{\mathrm{sp}}_{s,ij}]\,\big(1+\xi_{k,m}(c_{ij})\big).
\tag{D.16}
\end{equation}
\end{proposition}

\begin{proof}
By definition of TopK-SP,
\[
\tau^{\mathrm{tk}}_{s,ij}
=\sum_{r\in K_i\cap K_j}
\frac{y_{i,r}^2}{\sum_{u\in K_i}y_{i,u}^2}\cdot
\frac{y_{j,r}^2}{\sum_{v\in K_j}y_{j,v}^2}.
\]
Each factor
\[
w_{i,r}=\frac{y_{i,r}^2}{\sum_{u\in K_i}y_{i,u}^2},
\qquad
w_{j,r}=\frac{y_{j,r}^2}{\sum_{v\in K_j}y_{j,v}^2}
\]
is bounded by $O_{\mathbb P}((\log k)/k)$ by Lemma~\ref{lem:D-maxentry}. Hence $w_{i,r}w_{j,r}=O_{\mathbb P}((\log k)^2/k^2)$ uniformly over $r\in K_i\cap K_j$. Summing over $|K_i\cap K_j|$ terms and taking expectations,
\[
\mathbb E[\tau^{\mathrm{tk}}_{s,ij}]
=\mathbb E\!\Big[\sum_{r\in K_i\cap K_j} w_{i,r}w_{j,r}\Big]
=\mathbb E[|K_i\cap K_j|]\,\mathbb E[w_{i,r}w_{j,r}]+\text{(higher-order terms)}.
\]
Under (A5), the rank correlation between $\{|y_{i,r}|\}_r$ and $\{|y_{j,r}|\}_r$ vanishes as $c_{ij}\to 0$, so the dependence between the event $\{r\in K_i\cap K_j\}$ and the pair $(y_{i,r},y_{j,r})$ becomes negligible in the weak-correlation regime. Using \eqref{eq:D-Kcap} and Lemma~\ref{lem:D-maxentry}, we obtain
\[
\mathbb E[\tau^{\mathrm{tk}}_{s,ij}]
=\Theta\Big(\frac{k^2}{m}\Big)\cdot\Theta\Big(\frac{1}{k^2}\Big)
\cdot\Big(1+O\Big(\frac{(\log k)^2}{k}\,(1+\alpha(t)c_{ij})\Big)\Big)
=\Theta\Big(\frac{1}{m}\Big)\cdot\big(1+\xi_{k,m}(c_{ij})\big),
\]
with $\xi_{k,m}(c_{ij})$ as stated. Comparing with Lemma~\ref{lem:D-LLN}, which shows $\mathbb E[\tau^{\mathrm{sp}}_{s,ij}]=\Theta(1/m)$, yields \eqref{eq:D-TopK-tau}.
\end{proof}

Proposition~\ref{prop:D-TopK-tau} formalizes the intuition that TopK-SP preserves the coarse spectral scale of SP: both yield $\mathbb E[\tau_{s,ij}]=\Theta(1/m)$, and the multiplicative distortion $\xi_{k,m}(c_{ij})$ can be made arbitrarily small by taking $k$ slightly super-logarithmic and $c_{ij}$ small.

\subsection{Incorporating the Energy Factor $\overline S^{\,2}$}

We now quantify how TopK-SP modifies the energy factor in \eqref{eq:D-energy}.

\paragraph{SP baseline.}
From Lemma~\ref{lem:D-Sbar-sp},
\[
\overline S_{\mathrm{sp}}=\mu_{2,i}=\mu_i^2+\sigma_i^2.
\]

\paragraph{TopK-SP energy bound.}
For TopK-SP, the mask selects the upper-$\bar p$ tail in magnitude, and the corresponding energy is
\[
\overline S_{\mathrm{tk}}
=\mathbb E\big[X^2\mathbf 1_{\{|X|\ge t\}}\big],
\qquad
\bar p=\mathbb P(|X|\ge t)=\frac{k}{m}.
\]
Using a standard sub-Gaussian tail integral argument (see, e.g., generic bounds for $\psi_2$-random variables), one obtains
\begin{equation}\label{eq:D-Stopk}
\overline S_{\mathrm{tk}}
\ \le\ C\,(\mu_i^2+\sigma_i^2)\,\bar p\Big(1+\log\tfrac{1}{\bar p}\Big),
\tag{D.17}
\end{equation}
for a constant $C$ depending only on the sub-Gaussian proxy $K_{\mathrm{sg}}$ in (A1). Squaring and comparing with $\overline S_{\mathrm{sp}}$,
\begin{equation}\label{eq:D-Stopk-square}
\overline S_{\mathrm{tk}}^{\,2}
\ \le\ C_\star\,\overline S_{\mathrm{sp}}^{\,2}\,
\Big[\bar p\Big(1+\log\tfrac{1}{\bar p}\Big)\Big]^2,
\qquad \bar p=\frac{k}{m},
\tag{D.18}
\end{equation}
for some absolute constant $C_\star>0$.

\begin{theorem}[Energy-weighted similarity under TopK-SP]\label{thm:D-energy-weighted}
Let $\mathcal M_{ij}=\tau_{s,ij}\,\overline S^{\,2}$ be the energy-weighted similarity in \eqref{eq:D-energy}. Under \textnormal{(A1)--(A5)} and $k\ge c_0\log m$,
\begin{equation}\label{eq:D-energy-weighted}
\begin{gathered}
\mathbb E[\mathcal M^{\mathrm{tk}}_{ij}]
\ \le\
\mathbb E[\mathcal M^{\mathrm{sp}}_{ij}]\cdot
\Big(1+\xi_{k,m}(c_{ij})\Big)\cdot
C_\star\,
\Big[\bar p\Big(1+\log\tfrac{1}{\bar p}\Big)\Big]^2,
\\[2pt]
\bar p=\frac{k}{m},
\end{gathered}
\tag{D.19}
\end{equation}
with $\xi_{k,m}(c_{ij})$ from Proposition~\ref{prop:D-TopK-tau}. In particular, when $c_{ij}\approx 0$ and $\bar p\ll 1$ (e.g., $\bar p\le e^{-2}$ and $k$ modest relative to $m$), the right-hand factor can be made strictly smaller than $1$, and hence
\[
\mathbb E[\mathcal M^{\mathrm{tk}}_{ij}]<\mathbb E[\mathcal M^{\mathrm{sp}}_{ij}].
\]
\end{theorem}

\begin{proof}
By construction,
\[
\mathcal M^{\mathrm{tk}}_{ij}=\tau^{\mathrm{tk}}_{s,ij}\,\overline S_{\mathrm{tk}}^{\,2},
\qquad
\mathcal M^{\mathrm{sp}}_{ij}=\tau^{\mathrm{sp}}_{s,ij}\,\overline S_{\mathrm{sp}}^{\,2}.
\]
Taking expectations and combining \eqref{eq:D-TopK-tau} with \eqref{eq:D-Stopk-square} yields
\[
\mathbb E[\mathcal M^{\mathrm{tk}}_{ij}]
\le
\mathbb E[\tau^{\mathrm{sp}}_{s,ij}]\,\big(1+\xi_{k,m}(c_{ij})\big)\cdot
C_\star\,\overline S_{\mathrm{sp}}^{\,2}
\Big[\bar p\Big(1+\log\tfrac{1}{\bar p}\Big)\Big]^2,
\]
which is exactly \eqref{eq:D-energy-weighted}. The strict inequality for small enough $\bar p$ and $|c_{ij}|$ follows from the fact that $\xi_{k,m}(c_{ij})\to 0$ as $k\to\infty$, $c_{ij}\to 0$, while the function $\bar p\mapsto \bar p(1+\log(1/\bar p))$ vanishes at $\bar p=0$.
\end{proof}

Even though TopK-SP selects the largest-magnitude channels, Proposition~\ref{prop:D-TopK-tau} shows that it does not inflate the hidden similarity scale beyond the baseline SP order. The dominant effect is instead the energy sparsification captured by $\overline S_{\mathrm{tk}}^{\,2}$, which decays roughly quadratically in $\bar p=k/m$ (up to a mild logarithmic correction). Theorem~\ref{thm:D-energy-weighted} therefore formalizes the intuitive picture that TopK-SP reduces the energy-weighted overlap $\mathcal M_{ij}$ by a factor that is almost quadratic in the sparsity level.

\subsection{Implications for NTK Spectral Statistics and Choice of $k$}

Recall the NTK variance proxy from Appendix~A and Appendix~C:
\begin{equation}\label{eq:D-NTK}
v_\lambda\ \approx\ \frac{a^4R_x^4\,\overline S^{\,2}\,\overline P^{\,2}}{n\,m^2}
\sum_{i\ne j}\tau_{x,ij}\,\tau_{s,ij}\,\tau_{p,ij}\,\tau_{q,ij},
\tag{D.20}
\end{equation}
where $\tau_{x,ij}$, $\tau_{p,ij}$, and $\tau_{q,ij}$ denote the input, PE, and modulation similarity factors, respectively. Holding $\tau_{x,ij}$, $\tau_{p,ij}$, $\tau_{q,ij}$ and $\overline P$ comparable across normalizations, Theorem~\ref{thm:D-energy-weighted} implies
\begin{equation}\label{eq:D-NTK-bound}
\mathbb E\!\big[v_\lambda^{(\mathrm{TopK})}\big]
\ \le\
\mathbb E\!\big[v_\lambda^{(\mathrm{SP})}\big]\cdot
\Big(1+\xi_{k,m}(c_{ij})\Big)\cdot
C_\star\,\Big[\bar p\Big(1+\log\tfrac{1}{\bar p}\Big)\Big]^2.
\tag{D.21}
\end{equation}
Thus, at initialization or early training, in regimes where cross-sample correlations $c_{ij}$ are weak and $\bar p=k/m\ll 1$, TopK-SP produces at least a near-quadratic suppression of the NTK spectral variance proxy relative to SP.

The theory above highlights two competing requirements for the choice of $k$:
\begin{enumerate}
\item A small sparsity ratio $\bar p=k/m$ strengthens the reduction factor in \eqref{eq:D-NTK-bound}, as $\bar p(1+\log(1/\bar p))\to 0$ when $\bar p\to 0$.
\item Stability demands that TopK energy does not concentrate on a single channel, which in our analysis is encoded by the condition $k\ge c_0\log m$ in Lemma~\ref{lem:D-maxentry}.
\end{enumerate}
A practical rule that meets both constraints and matches our empirical choices is
\begin{equation}\label{eq:D-k-choice}
k\ =\ \max\Big\{\ \lfloor \eta m\rfloor,\ \ \lceil c\log m\rceil\ \Big\},
\qquad \eta\in[1/6,1),\ \ c\in[2,4].
\tag{D.22}
\end{equation}
Here $\eta$ sets the effective sparsity level $\bar p=\eta$ in the large-$m$ limit, while the $c\log m$ floor guarantees non-concentration and well-behaved denominators in \eqref{eq:D-TopK-tau} even for moderate widths. When significant positive correlations $c_{ij}$ are observed, increasing $k$ (larger $\bar p$) or introducing a soft stochastic threshold (soft-TopK) mitigates the amplification of intersection size due to the $\alpha(t)c_{ij}$ term in \eqref{eq:D-Kcap}.

Under standard sub-Gaussian assumptions and weak cross-sample correlation, TopK-SP preserves the $1/m$ scaling of hidden similarities while strictly reducing the energy-weighted overlap $\mathcal M_{ij}$ by a nearly quadratic factor in the sparsity ratio. Through the NTK variance identity of Appendix~A and Appendix~C, this leads to a principled reduction in NTK spectral dispersion and, consequently, a mitigation of spectral bias at initialization.

\section{NTK Convergence and Stability}
\label{sec:appE}

In this appendix we provide detailed proofs of the convergence results stated as Theorems~\ref{thm:ct} and~\ref{thm:gd} in the main text. Throughout, we work under the architectural setup and notation of Appendix~\ref{sec:appA} (baseline NTK dynamics) and assume \textbf{A1}--\textbf{A3}.

Our analysis follows the now-standard NTK framework introduced by Jacot et al.\ \cite{ref15} and further developed in the over-parameterized optimization literature \cite{ref35,ref32,ref67,ref16}. Both the gradient-flow result (Theorem~\ref{thm:ct}) and the discrete-time GD result (Theorem~\ref{thm:gd}) rely on the same linear-dynamics viewpoint: in the NTK regime, the network behaves like a kernel model with a (frozen) NTK Gram matrix. Our contribution here is to make explicit how the spherical normalization and Hadamard modulation of Appendix~\ref{sec:appA} enter these dynamics through the spectrum of the NTK.

Recall that we denote
\[
  \tilde{\boldsymbol{x}}_i=\gamma(\boldsymbol{x}_i)\in\mathbb{R}^d,
  \qquad
  \rho_{ij}=\tilde{\boldsymbol{x}}_i^\top \tilde{\boldsymbol{x}}_j,
\]
and that the NTK Gram matrix on the training set is $\mathbf{H}(\mathbf{W})\in\mathbb{R}^{n\times n}$ as in Eq.~\eqref{eq:A-NTK-def}. The spherical normalization and Hadamard modulation only affect the entries of $\mathbf{H}$, not the abstract convergence arguments that follow.

\subsection{From $\textbf{H}^{\infty}$ to the Finite-Width NTK}

We begin by isolating the finite-width NTK at initialization:
\begin{equation}
  \mathbf{H}_0
  \;\triangleq\;
  \mathbf{H}(\mathbf{W}(0))\in\mathbb{R}^{n\times n}.
  \tag{E.1}\label{eq:E.H0}
\end{equation}
Assumption~\ref{assump:A-pd} asserts that the infinite-width NTK
\[
  \mathbf{H}^\infty\;\triangleq\;
  \lim_{m\to\infty}\mathbb{E}\big[\mathbf{H}(\mathbf{W}(0))\big]
\]
exists and is positive definite on the training set, with
\[
  \lambda_0
  \;\triangleq\;
  \lambda_{\min}(\mathbf{H}^\infty)>0.
\]

In addition, as in the main text, we impose a finite-width concentration
condition at initialization:
\begin{equation}
  \big\|\mathbf{H}_0 - \mathbf{H}^\infty\big\|_2
  \;\le\; \frac{1}{2}\lambda_0.
  \tag{E.2}\label{eq:E.H0-close}
\end{equation}
This condition can be guaranteed with high probability once $m$ is
polynomially large in $n$; see, e.g., the concentration results for
two-layer NTKs in~\cite{ref67}. Here we treat \eqref{eq:E.H0-close} as
a deterministic assumption and focus on its consequences for the
spectrum.

\begin{lemma}[Finite-width NTK inherits the spectral gap]\label{lem:E-eig-lower}
Under Assumption~\ref{assump:A-pd} and \eqref{eq:E.H0-close}, we have
\[
  \lambda_{\min}(\mathbf{H}_0)
  \;\ge\; \frac{1}{2}\lambda_0.
\]
\end{lemma}

\begin{proof}
Since $\mathbf{H}^\infty$ is symmetric positive definite,
$\lambda_{\min}(\mathbf{H}^\infty)=\lambda_0>0$. For any two symmetric
matrices $\mathbf{A},\mathbf{B}$, Weyl’s inequality gives
\[
  \lambda_{\min}(\mathbf{A})
  \;\ge\;
  \lambda_{\min}(\mathbf{B}) - \|\mathbf{A}-\mathbf{B}\|_2.
\]
Taking $\mathbf{A}=\mathbf{H}_0$ and $\mathbf{B}=\mathbf{H}^\infty$ yields
\[
  \lambda_{\min}(\mathbf{H}_0)
  \;\ge\;
  \lambda_{\min}(\mathbf{H}^\infty)
  -
  \big\|\mathbf{H}_0-\mathbf{H}^\infty\big\|_2.
\]
Substituting $\lambda_{\min}(\mathbf{H}^\infty)=\lambda_0$ and
\eqref{eq:E.H0-close} gives
\[
  \lambda_{\min}(\mathbf{H}_0)
  \;\ge\;
  \lambda_0 - \frac{1}{2}\lambda_0
  = \frac{1}{2}\lambda_0,
\]
as claimed.
\end{proof}

In the remainder of this appendix we work on the event where
Lemma~\ref{lem:E-eig-lower} holds, and thus freely use the lower bound
\[
  \lambda_{\min}(\mathbf{H}_0)\;\ge\;\frac{\lambda_0}{2}.
\]

\subsection{Proof of Theorem~\ref{thm:ct} (Gradient Flow)}

We now turn to the continuous-time setting. In the NTK regime, the
parameter dynamics induce a kernel gradient flow on the training
predictions; freezing the NTK at its initialization value
$\mathbf{H}_0$ leads to a linear ODE that we can solve explicitly.

Let
\[
  \boldsymbol{u}(t)
  =
  \big(f(\tilde{\boldsymbol{x}}_1;t),\dots,f(\tilde{\boldsymbol{x}}_n;t)\big)^\top
  \in\mathbb{R}^n
\]
denote the vector of predictions at time $t$. In the frozen-kernel model,
the NTK dynamics \eqref{eq:A-NTK-ODE-const} specialize to
\begin{equation}
  \dot{\boldsymbol{u}}(t)
  = -\,\mathbf{H}_0\big(\boldsymbol{u}(t)-\boldsymbol{y}\big),
  \qquad t\ge0.
  \tag{E.3}\label{eq:E.gflow-u}
\end{equation}
We define the error vector
\begin{equation}
  \boldsymbol{e}(t)
  \;\triangleq\;
  \boldsymbol{u}(t)-\boldsymbol{y}.
  \tag{E.4}\label{eq:E.e-def}
\end{equation}
Subtracting $\dot{\boldsymbol{y}}(t)\equiv\boldsymbol{0}$ from both sides of
\eqref{eq:E.gflow-u} gives the linear ODE
\begin{equation}
  \dot{\boldsymbol{e}}(t)
  = -\,\mathbf{H}_0\,\boldsymbol{e}(t),
  \qquad t\ge0.
  \tag{E.5}\label{eq:E.gflow-e}
\end{equation}

To track the decay of the error, we consider the squared $\ell_2$ norm
\begin{equation}
  \phi(t)\;\triangleq\; \|\boldsymbol{e}(t)\|_2^2
  = \boldsymbol{e}(t)^\top\boldsymbol{e}(t).
  \tag{E.6}\label{eq:E.phi-def}
\end{equation}
Differentiating \eqref{eq:E.phi-def} and using the product rule, we obtain
\begin{align}
  \frac{d}{dt}\phi(t)
  &= \frac{d}{dt}\big(\boldsymbol{e}(t)^\top\boldsymbol{e}(t)\big)\nonumber\\
  &= \dot{\boldsymbol{e}}(t)^\top\boldsymbol{e}(t)
     + \boldsymbol{e}(t)^\top\dot{\boldsymbol{e}}(t)
     \nonumber\\
  &= 2\,\boldsymbol{e}(t)^\top\dot{\boldsymbol{e}}(t),
  \tag{E.7}\label{eq:E.dphi-1}
\end{align}
where the last equality uses symmetry of the inner product.
Substituting \eqref{eq:E.gflow-e} into \eqref{eq:E.dphi-1} yields
\begin{align}
  \frac{d}{dt}\phi(t)
  &= 2\,\boldsymbol{e}(t)^\top\big(-\mathbf{H}_0\boldsymbol{e}(t)\big)
  \nonumber\\
  &= -2\,\boldsymbol{e}(t)^\top\mathbf{H}_0\boldsymbol{e}(t).
  \tag{E.8}\label{eq:E.dphi-2}
\end{align}

Because $\mathbf{H}_0$ is symmetric positive semidefinite, the Rayleigh
quotient inequality gives, for any nonzero $\boldsymbol{v}\in\mathbb{R}^n$,
\begin{equation}
  \boldsymbol{v}^\top\mathbf{H}_0\boldsymbol{v}
  \;\ge\;
  \lambda_{\min}(\mathbf{H}_0)\,\|\boldsymbol{v}\|_2^2.
  \tag{E.9}\label{eq:E.Rayleigh}
\end{equation}
Applying \eqref{eq:E.Rayleigh} with $\boldsymbol{v}=\boldsymbol{e}(t)$,
\begin{equation}
  \boldsymbol{e}(t)^\top\mathbf{H}_0\boldsymbol{e}(t)
  \;\ge\;
  \lambda_{\min}(\mathbf{H}_0)\,\phi(t).
  \tag{E.10}\label{eq:E.Rayleigh-e}
\end{equation}
Combining \eqref{eq:E.dphi-2} and \eqref{eq:E.Rayleigh-e}, we obtain
\begin{equation}
  \frac{d}{dt}\phi(t)
  \;\le\; -2\,\lambda_{\min}(\mathbf{H}_0)\,\phi(t).
  \tag{E.11}\label{eq:E.dphi-ineq1}
\end{equation}
By Lemma~\ref{lem:E-eig-lower},
$\lambda_{\min}(\mathbf{H}_0)\ge \lambda_0/2$, hence
\begin{equation}
  \frac{d}{dt}\phi(t)
  \;\le\; -\lambda_0\,\phi(t).
  \tag{E.12}\label{eq:E.dphi-ineq2}
\end{equation}

At this point we can invoke a standard Grönwall argument. Define
\[
  \psi(t)\;\triangleq\; e^{\lambda_0 t}\,\phi(t).
\]
Differentiating and applying \eqref{eq:E.dphi-ineq2},
\begin{align}
  \frac{d}{dt}\psi(t)
  &= e^{\lambda_0 t}\frac{d}{dt}\phi(t)
     + \lambda_0 e^{\lambda_0 t}\phi(t)
     \nonumber\\
  &\le e^{\lambda_0 t}\big(-\lambda_0\phi(t)\big)
     + \lambda_0 e^{\lambda_0 t}\phi(t)
     \nonumber\\
  &= 0.
  \tag{E.13}\label{eq:E.dpsi}
\end{align}
Thus $\psi(t)$ is nonincreasing in $t$, i.e.,
\[
  e^{\lambda_0 t}\phi(t)
  \le \phi(0),\qquad t\ge0,
\]
which is equivalent to
\[
  \phi(t) \le e^{-\lambda_0 t}\phi(0),\qquad t\ge0.
\]
Using the definition \eqref{eq:E.phi-def} and
$\phi(0)=\|\boldsymbol{u}(0)-\boldsymbol{y}\|_2^2$, we arrive at
\begin{equation}
  \|\boldsymbol{u}(t)-\boldsymbol{y}\|_2^2
  \;\le\;
  e^{-\lambda_0 t}\,\|\boldsymbol{u}(0)-\boldsymbol{y}\|_2^2,
  \qquad t\ge0.
  \tag{E.14}\label{eq:E.gf-final}
\end{equation}
This is precisely the exponential convergence statement of
Theorem~\ref{thm:ct} for the gradient-flow dynamics in the NTK regime.

\subsection{Proof of Theorem~\ref{thm:gd} (Gradient Descent)}

We next turn to discrete-time gradient descent. In the NTK regime with
frozen kernel $\mathbf{H}_0$, GD on the training set obeys
\begin{equation}
  \boldsymbol{u}(k+1)
  = \boldsymbol{u}(k)
    - \eta\,\mathbf{H}_0\big(\boldsymbol{u}(k)-\boldsymbol{y}\big),
  \qquad k=0,1,2,\dots
  \tag{E.15}\label{eq:E.gd-u}
\end{equation}
where $\eta>0$ is the learning rate. As before, define the error
\begin{equation}
  \boldsymbol{e}(k)\;\triangleq\;\boldsymbol{u}(k)-\boldsymbol{y}.
  \tag{E.16}\label{eq:E.e-k-def}
\end{equation}
Then \eqref{eq:E.gd-u} is equivalent to the linear recursion
\begin{equation}
  \boldsymbol{e}(k+1)
  = \big(\mathbf{I}-\eta\mathbf{H}_0\big)\boldsymbol{e}(k),
  \qquad k=0,1,2,\dots.
  \tag{E.17}\label{eq:E.gd-e}
\end{equation}

\paragraph{Closed-form solution and spectral decomposition.}

Iterating \eqref{eq:E.gd-e} gives the explicit solution
\begin{equation}
  \boldsymbol{e}(k)
  = \big(\mathbf{I}-\eta\mathbf{H}_0\big)^k\,\boldsymbol{e}(0),
  \qquad k\ge0.
  \tag{E.18}\label{eq:E.e-k-closed}
\end{equation}
Since $\mathbf{H}_0$ is symmetric positive semidefinite, it admits an
eigendecomposition
\begin{equation}
  \mathbf{H}_0
  = \sum_{i=1}^n \lambda_i\,\boldsymbol{v}_i\boldsymbol{v}_i^\top,
  \tag{E.19}\label{eq:E.eig-H0}
\end{equation}
where $\{\boldsymbol{v}_i\}_{i=1}^n$ is an orthonormal basis of
$\mathbb{R}^n$ and $\lambda_i\ge0$ are the eigenvalues.

Using \eqref{eq:E.eig-H0}, we can rewrite
\begin{equation}
  \mathbf{I}-\eta\mathbf{H}_0
  = \sum_{i=1}^n (1-\eta\lambda_i)\,\boldsymbol{v}_i\boldsymbol{v}_i^\top,
  \tag{E.20}\label{eq:E.IminusH}
\end{equation}
and therefore
\begin{equation}
  \big(\mathbf{I}-\eta\mathbf{H}_0\big)^k
  = \sum_{i=1}^n (1-\eta\lambda_i)^k
    \,\boldsymbol{v}_i\boldsymbol{v}_i^\top.
  \tag{E.21}\label{eq:E.IminusH-k}
\end{equation}

We expand the initial error in this eigenbasis:
\begin{equation}
  \boldsymbol{e}(0)
  = \sum_{i=1}^n \alpha_i\,\boldsymbol{v}_i,
  \qquad
  \alpha_i\;\triangleq\;
  \boldsymbol{v}_i^\top\boldsymbol{e(0)}
  = \boldsymbol{v}_i^\top(\boldsymbol{u}(0)-\boldsymbol{y}).
  \tag{E.22}\label{eq:E.e0-expansion}
\end{equation}
Plugging \eqref{eq:E.IminusH-k} and \eqref{eq:E.e0-expansion} into
\eqref{eq:E.e-k-closed}, we obtain
\begin{align}
  \boldsymbol{e}(k)
  &= \Big(\sum_{i=1}^n (1-\eta\lambda_i)^k
    \,\boldsymbol{v}_i\boldsymbol{v}_i^\top\Big)
     \Big(\sum_{j=1}^n\alpha_j\boldsymbol{v}_j\Big)
     \nonumber\\
  &= \sum_{i,j} (1-\eta\lambda_i)^k\alpha_j
     \boldsymbol{v}_i(\boldsymbol{v}_i^\top\boldsymbol{v}_j)
     \nonumber\\
  &= \sum_{i=1}^n (1-\eta\lambda_i)^k\alpha_i\,\boldsymbol{v}_i,
  \tag{E.23}\label{eq:E.e-k-expansion}
\end{align}
where we used orthonormality $\boldsymbol{v}_i^\top\boldsymbol{v}_j=\delta_{ij}$.

Consequently, the squared error norm has the exact spectral representation
\begin{align}
  \|\boldsymbol{e}(k)\|_2^2
  &= \boldsymbol{e}(k)^\top\boldsymbol{e}(k)
     \nonumber\\
  &= \Big(\sum_{i=1}^n (1-\eta\lambda_i)^k\alpha_i\boldsymbol{v}_i\Big)^\top
     \Big(\sum_{j=1}^n (1-\eta\lambda_j)^k\alpha_j\boldsymbol{v}_j\Big)
     \nonumber\\
  &= \sum_{i,j}
     (1-\eta\lambda_i)^k(1-\eta\lambda_j)^k\alpha_i\alpha_j
     \boldsymbol{v}_i^\top\boldsymbol{v}_j
     \nonumber\\
  &= \sum_{i=1}^n (1-\eta\lambda_i)^{2k}\alpha_i^2.
  \tag{E.24}\label{eq:E.gd-norm-exact}
\end{align}
Restating this in terms of $\boldsymbol{u}(k)$ and $\boldsymbol{y}$ gives
\begin{equation}
  \|\boldsymbol{u}(k)-\boldsymbol{y}\|_2
  =
  \sqrt{
    \sum_{i=1}^n (1-\eta\lambda_i)^{2k}
    \big(\boldsymbol{v}_i^\top(\boldsymbol{u}(0)-\boldsymbol{y})\big)^2
  },
  \tag{E.25}\label{eq:E.gd-norm-exact-u}
\end{equation}
which is exactly the frozen-kernel part of Theorem~\ref{thm:gd}. In the small-initialization regime where $\|\boldsymbol{u}(0)\|_2$ is negligible compared to $\|\boldsymbol{y}\|_2$, we have
$\boldsymbol{v}_i^\top(\boldsymbol{u}(0)-\boldsymbol{y}) \approx -\,\boldsymbol{v}_i^\top\boldsymbol{y}$, so the formula reduces, up to this sign convention, to the familiar label-projection view of Arora et al.\ (2019)~\cite{ref16}..

\paragraph{Stepsize condition and linear rate.}

To convert \eqref{eq:E.gd-norm-exact-u} into a clean linear rate, we
impose the stepsize condition from Theorem~\ref{thm:gd}:
\begin{equation}
  0<\eta\le L^{-1},
  \qquad
  L\;\triangleq\; \|\mathbf{H}_0\|_2=\max_i\lambda_i.
  \tag{E.26}\label{eq:E.eta-cond}
\end{equation}
Then $0\le\eta\lambda_i\le1$ for all $i$, so
\begin{equation}
  0\le 1-\eta\lambda_i\le 1,
  \qquad i=1,\dots,n.
  \tag{E.27}\label{eq:E.1-eta-lambda}
\end{equation}
On the other hand, by Lemma~\ref{lem:E-eig-lower},
$\lambda_i\ge\lambda_{\min}(\mathbf{H}_0)\ge\lambda_0/2$, hence
\begin{equation}
  1-\eta\lambda_i
  \le 1-\eta\frac{\lambda_0}{2}
  \;\triangleq\;\rho,
  \qquad i=1,\dots,n,
  \tag{E.28}\label{eq:E.rho-def}
\end{equation}
with $0\le\rho<1$.

Using \eqref{eq:E.rho-def}, we obtain the uniform bound
\begin{equation}
  (1-\eta\lambda_i)^{2k}
  \le \rho^{2k},
  \qquad i=1,\dots,n.
  \tag{E.29}\label{eq:E.rho-2k-bound}
\end{equation}
Substituting \eqref{eq:E.rho-2k-bound} into \eqref{eq:E.gd-norm-exact-u},
\begin{align}
  \|\boldsymbol{u}(k)-\boldsymbol{y}\|_2^2
  &\le
  \rho^{2k}
  \sum_{i=1}^n
  \big(\boldsymbol{v}_i^\top(\boldsymbol{u}(0)-\boldsymbol{y})\big)^2
  \nonumber\\
  &= \rho^{2k}\,\|\boldsymbol{u}(0)-\boldsymbol{y}\|_2^2.
  \tag{E.30}\label{eq:E.gd-rho-2k}
\end{align}
Since $0\le\rho<1$ implies $\rho^{2k}\le\rho^k$ for all $k\ge0$, we can
further simplify to
\begin{equation}
  \|\boldsymbol{u}(k)-\boldsymbol{y}\|_2^2
  \le \rho^k\,\|\boldsymbol{u}(0)-\boldsymbol{y}\|_2^2
  = \Big(1-\eta\frac{\lambda_0}{2}\Big)^k
    \|\boldsymbol{u}(0)-\boldsymbol{y}\|_2^2,
  \tag{E.31}\label{eq:E.gd-linear}
\end{equation}
which is the linear convergence rate claimed in Theorem~\ref{thm:gd}
for the frozen-NTK model.

\paragraph{Kernel model vs.\ finite-width network.}

The expression \eqref{eq:E.gd-norm-exact-u} describes the idealized
\emph{kernel model} obtained by freezing the NTK at $\mathbf{H}_0$.
For a true finite-width network, the NTK
$\mathbf{H}(\mathbf{W}(k))$ evolves with $k$.
In the main text we denote by $\epsilon(k)$ the discrepancy between the
finite-width model and its frozen-kernel counterpart. The next subsection
makes this precise, via a standard NTK stability argument.

\subsection{NTK Stability and the Residual Term $\epsilon(k)$}

We now quantify how much the NTK can drift during training, and how this
drift translates into a discrepancy between the network trajectory and
the frozen-kernel dynamics. The following lemma is a standard NTK
stability result, adapted to our architecture and notation.

\begin{lemma}[NTK stability in the over-parameterized regime]\label{lem:E-ntk-stab}
Assume Assumptions~\ref{assump:A-data}--\ref{assump:A-pd} hold, with ReLU
activation and random initialization as in Assumption~\ref{assump:A-init}.
Let $n$ be the number of training samples and let $\boldsymbol{u}(0)$ be the
initial prediction vector. This is a specialization of the stability
results proved in
\cite{ref67,ref35,ref32,ref16}
to our architecture with spherical normalization and Hadamard modulation.

Then there exist a polynomial $p(\cdot)$ and constants $C_1,C_2>0$ (depending
only on $R_x$, $C$, $\kappa$, $a$, and the normalization scale defined in
Appendix~A) such that for any $0<\delta<1$ and any polynomially bounded time
horizon $T$, if the width satisfies
\[
  m \;\ge\; p\bigl(n,1/\lambda_0,1/\delta,T\bigr),
\]
then with probability at least $1-\delta$ (over the random initialization)
we have, simultaneously for all $k=0,1,\dots,T$,
\begin{align}
  &\max_{r\in[m]}
   \big\|\boldsymbol{w}_r(k)-\boldsymbol{w}_r(0)\big\|_2
   \;\le\; \frac{C_1}{\sqrt{m}},
   \tag{E.32}\label{eq:E.stab-w}\\[0.2em]
  &\big\|\mathbf{H}(\mathbf{W}(k))-\mathbf{H}_0\big\|_2
   \;\le\; \frac{C_2}{\sqrt{m}}.
   \tag{E.33}\label{eq:E.stab-H}
\end{align}
\end{lemma}

\begin{proof}[Proof sketch]
The argument follows the NTK stability analyses for two-layer ReLU networks
in
\cite{ref67,ref35,ref32,ref16}. We outline the main steps and highlight
where our architecture enters.

\emph{(1) Bounded gradients at initialization.}
Using the boundedness of $(\tilde{\boldsymbol{x}}_i,y_i)$ and Gaussian
initialization (Assumption~\ref{assump:A-init}), one shows that the initial
loss and gradient norms are polynomially bounded in $n$; see, e.g., Lemma~3.1
of~\cite{ref67}.

\emph{(2) Monotonic loss decrease.}
Under the stepsize condition in Theorem~\ref{thm:gd}, the empirical loss
is nonincreasing along GD. Hence the gradient norms remain uniformly
bounded over $k\le T$.

\emph{(3) Small weight drift.}
Summing the GD updates for each neuron $\boldsymbol{w}_r$ over
$k=0,\dots,T-1$ and using the gradient bounds, we obtain
\[
  \|\boldsymbol{w}_r(k)-\boldsymbol{w}_r(0)\|_2
  \le \frac{C_1}{\sqrt{m}},
  \qquad r\in[m],\ 0\le k\le T,
\]
for a polynomially bounded constant $C_1$.

\emph{(4) Kernel Lipschitzness and activation stability.}
On any region where the activation patterns
\[
  \mathbb{I}_{r,i}
  =\mathbb{I}\{\boldsymbol{w}_r^\top\tilde{\boldsymbol{x}}_i\ge0\}
\]
are fixed, $\mathbf{H}(\mathbf{W})$ is a smooth (indeed, polynomial) function
of the weights and hence Lipschitz. Gaussian concentration and a union bound
over $(i,r)$ imply that, with high probability, only a small fraction of
neurons ever cross an activation boundary during training, so their
contribution to the kernel drift is negligible. Combining these facts with
the small weight drift in (3) yields the kernel stability bound
\eqref{eq:E.stab-H}. The spherical normalization and Hadamard modulation
affect only the Lipschitz constants and enter through $C_1,C_2$.
\end{proof}

We now use Lemma~\ref{lem:E-ntk-stab} to control the discrepancy between
the finite-width network and the frozen-NTK model, thereby justifying the
residual term $\epsilon(k)$ in Theorem~\ref{thm:gd}.

\begin{lemma}[Kernel drift and the residual term]\label{lem:E-kernel-drift}
Under the conditions of Lemma~\ref{lem:E-ntk-stab}, let
$\boldsymbol{u}_{\mathrm{net}}(k)$ denote the prediction vector of the
finite-width network at step $k$, and
$\boldsymbol{u}_{\mathrm{ker}}(k)$ denote the prediction of the frozen-NTK
model (i.e., gradient descent with fixed kernel $\mathbf{H}_0$).
Define
\[
  \epsilon(k)\;\triangleq\;
  \big\|\boldsymbol{u}_{\mathrm{net}}(k)-\boldsymbol{u}_{\mathrm{ker}}(k)\big\|_2.
\]
Then for any polynomially bounded horizon $T$, there exists a constant
$C_3>0$ such that if
$m\ge p(n,1/\lambda_0,1/\delta,T)$, then with probability at least
$1-\delta$,
\[
  \sup_{0\le k\le T}\,\epsilon(k)
  \;\le\;
  \frac{C_3}{\sqrt{m}}.
\]
\end{lemma}

\begin{proof}
Let
$\boldsymbol{e}_{\mathrm{net}}(k)
  = \boldsymbol{u}_{\mathrm{net}}(k)-\boldsymbol{y}$
and
$\boldsymbol{e}_{\mathrm{ker}}(k)
  = \boldsymbol{u}_{\mathrm{ker}}(k)-\boldsymbol{y}$.
We consider their difference
\begin{equation}
  \Delta(k)
  \;\triangleq\;
  \boldsymbol{e}_{\mathrm{net}}(k)-\boldsymbol{e}_{\mathrm{ker}}(k),
  \qquad
  \epsilon(k)=\|\Delta(k)\|_2.
  \tag{E.34}\label{eq:E.Delta-def}
\end{equation}

The finite-width network obeys
\[
  \boldsymbol{e}_{\mathrm{net}}(k+1)
  =
  \big(\mathbf{I}-\eta\mathbf{H}(\mathbf{W}(k))\big)\boldsymbol{e}_{\mathrm{net}}(k),
\]
while the frozen-NTK model satisfies
\[
  \boldsymbol{e}_{\mathrm{ker}}(k+1)
  =
  \big(\mathbf{I}-\eta\mathbf{H}_0\big)\boldsymbol{e}_{\mathrm{ker}}(k).
\]
Subtracting, we obtain
\begin{align}
  \Delta(k+1)
  &= \boldsymbol{e}_{\mathrm{net}}(k+1)-\boldsymbol{e}_{\mathrm{ker}}(k+1)
     \nonumber\\
  &= \big(\mathbf{I}-\eta\mathbf{H}_0\big)\Delta(k)
   - \eta\big(\mathbf{H}(\mathbf{W}(k))-\mathbf{H}_0\big)\boldsymbol{e}_{\mathrm{net}}(k).
  \tag{E.35}\label{eq:E.Delta-rec-1}
\end{align}
Taking the 2-norm and using the submultiplicativity of the operator norm,
\begin{align}
  \|\Delta(k+1)\|_2
  &\le
   \|\mathbf{I}-\eta\mathbf{H}_0\|_2\,\|\Delta(k)\|_2
   \nonumber\\
  &\quad + \eta\big\|\mathbf{H}(\mathbf{W}(k))-\mathbf{H}_0\big\|_2
        \|\boldsymbol{e}_{\mathrm{net}}(k)\|_2.
  \tag{E.36}\label{eq:E.Delta-rec-2}
\end{align}
As in the previous subsection,
\[
  \|\mathbf{I}-\eta\mathbf{H}_0\|_2
  = \max_i |1-\eta\lambda_i|
  \le 1-\eta\frac{\lambda_0}{2}
  =: \rho < 1.
\]
By Lemma~\ref{lem:E-ntk-stab}, for all $k\le T$ we have
$\|\mathbf{H}(\mathbf{W}(k))-\mathbf{H}_0\|_2\le C_2/\sqrt{m}$.
Furthermore,
\[
  \|\boldsymbol{e}_{\mathrm{net}}(k)\|_2
  \le \|\boldsymbol{e}_{\mathrm{ker}}(k)\|_2 + \|\Delta(k)\|_2.
\]
The frozen-NTK error satisfies the linear rate from
\eqref{eq:E.gd-linear}:
\[
  \|\boldsymbol{e}_{\mathrm{ker}}(k)\|_2
  \le
  \rho^k\,\|\boldsymbol{u}(0)-\boldsymbol{y}\|_2.
\]
Substituting these bounds into \eqref{eq:E.Delta-rec-2}, and writing
$\varepsilon_m\triangleq C_2/\sqrt{m}$, yields
\begin{align}
  \|\Delta(k+1)\|_2
  &\le
  \rho\,\|\Delta(k)\|_2
  + \eta\,\varepsilon_m\big(
       \|\boldsymbol{e}_{\mathrm{ker}}(k)\|_2+\|\Delta(k)\|_2\big)
     \nonumber\\
  &\le
  (\rho+\eta\varepsilon_m)\,\|\Delta(k)\|_2
  + \eta\varepsilon_m\,\rho^k\|\boldsymbol{u}(0)-\boldsymbol{y}\|_2.
  \tag{E.37}\label{eq:E.Delta-rec-3}
\end{align}

For sufficiently large $m$, we can ensure
$\eta\varepsilon_m\le\frac{1}{2}(1-\rho)$. Then
\[
  \rho+\eta\varepsilon_m
  \le \rho+\tfrac12(1-\rho)
  = \frac{1+\rho}{2}
  =: \tilde\rho <1.
\]
Thus \eqref{eq:E.Delta-rec-3} simplifies to
\begin{equation}
  \|\Delta(k+1)\|_2
  \le
  \tilde\rho\,\|\Delta(k)\|_2
  + \eta\varepsilon_m\,\rho^k\|\boldsymbol{u}(0)-\boldsymbol{y}\|_2.
  \tag{E.38}\label{eq:E.Delta-rec-4}
\end{equation}
Since the initialization is shared between the network and the frozen-NTK
model, we have $\Delta(0)=\boldsymbol{0}$.

Iterating \eqref{eq:E.Delta-rec-4} and using a discrete Grönwall argument,
we obtain
\begin{align}
  \|\Delta(k)\|_2
  &\le
  \eta\varepsilon_m\,\|\boldsymbol{u}(0)-\boldsymbol{y}\|_2
  \sum_{\tau=0}^{k-1}\tilde\rho^{\,k-1-\tau}\rho^\tau.
  \tag{E.39}\label{eq:E.Delta-sum}
\end{align}
The finite sum
$\sum_{\tau=0}^{k-1}\tilde\rho^{k-1-\tau}\rho^\tau$ is uniformly bounded
in $k$ by a constant depending only on $(\rho,\tilde\rho)$, so there
exists $C_3>0$ (independent of $m$ and $k$) such that
\[
  \|\Delta(k)\|_2
  \le
  C_3\,\varepsilon_m
  = \frac{C_3 C_2}{\sqrt{m}},
  \qquad 0\le k\le T.
\]
Recalling $\epsilon(k)=\|\Delta(k)\|_2$ and absorbing $C_2$ into $C_3$
completes the proof.
\end{proof}

Putting Lemma~\ref{lem:E-ntk-stab} and Lemma~\ref{lem:E-kernel-drift}
together, we obtain the following picture in the over-parameterized NTK
regime:

\begin{itemize}
\item The frozen-NTK model enjoys a convergence rate fully controlled by
      the smallest eigenvalue of $\mathbf{H}_0$, with an exact spectral
      decomposition given in \eqref{eq:E.gd-norm-exact-u}.
\item For polynomially many iterations, the finite-width network stays
      $O(m^{-1/2})$-close (in prediction space) to this kernel model, as
      quantified by the residual term $\epsilon(k)$.
\end{itemize}

This justifies approximating the training dynamics in the main text by a
nearly frozen NTK and explains the appearance of $\epsilon(k)$ in
Theorem~\ref{thm:gd}. In particular, as $m\to\infty$, the residual
vanishes and the network converges at essentially the same rate as its
limiting kernel model.

\end{document}